\documentclass{article}

 \usepackage[main, final]{neurips_2026}

\usepackage[utf8]{inputenc} 
\usepackage[T1]{fontenc}    
\usepackage{hyperref}       
\usepackage{url}            
\usepackage{booktabs}       
\usepackage{amsfonts}       
\usepackage{nicefrac}       
\usepackage{microtype}      

\usepackage{tikz}
\usetikzlibrary{arrows, positioning}
\usetikzlibrary{matrix, calc}

\usepackage{fancyhdr}
\usepackage[dvipsnames, svgnames]{xcolor}
\definecolor{referencecolor}{HTML}{65730E}
\definecolor{localreferencecolor}{HTML}{3B3730}
\hypersetup{
  colorlinks=true,
  citecolor=referencecolor,
  linkcolor=localreferencecolor,
  urlcolor=referencecolor
}
\usepackage{amsthm, amsmath, amssymb}
\usepackage{mathtools}
\usepackage{graphicx}
\usepackage[normalem]{ulem}
\usepackage{enumerate}
\usepackage{enumitem}
\usepackage[capitalise]{cleveref}
\usepackage{xspace}
\usepackage{caption}
\usepackage{tikz-cd}
\usepackage{soul}
\usepackage{wrapfig}
\usepackage{wrapstuff}
\usepackage{mathabx}
\usepackage{tabularx}
\usepackage{wrapfig}

\usepackage[
    linesnumbered,
    ruled,
    vlined,
    ]{algorithm2e}

\usepackage{glossaries}
\makeglossaries
\setglossarystyle{long}
\glsdisablehyper

\makeatletter
\renewcommand*{\glossaryentrynumbers}[1]{(pg.~#1)\gls@save@numberlist{#1}}
\makeatother

\makeatletter
\newtheorem*{rep@theorem}{\rep@title}
\newcommand{\newreptheorem}[2]{%
\newenvironment{rep#1}[1]{%
 \def\rep@title{#2 \ref{##1}}%
 \begin{rep@theorem}}%
 {\end{rep@theorem}}}
\makeatother

\newreptheorem{theorem}{Theorem}
\newreptheorem{lemma}{Lemma}

\usepackage{stmaryrd}

\definecolor{EditDarkGreen}{HTML}{003C00} 
\newcommand{\add}[1]{\textcolor{black}{#1}}

\title{
Caterpillar GNN:
Replacing Message Passing\\%
with Graph-Level Aggregation}

\author{%
  Marek \v{C}ern\'{y}\\
  University of Antwerp, Belgium\\
  \texttt{marek.cerny@uantwerp.be} \\
}

\usepackage{amsmath,amsfonts,bm}

\def\eqref#1{equation~\ref{#1}}

\def\1{\bm{1}}

\def\vone{{\bm{1}}}

\def\va{{\bm{a}}}
\def\vb{{\bm{b}}}

\def\vh{{\bm{h}}}

\def\vs{{\bm{s}}}

\def\vu{{\bm{u}}}
\def\vv{{\bm{v}}}
\def\vw{{\bm{w}}}
\def\vx{{\bm{x}}}

\def\vz{{\bm{z}}}

\def\mA{{\bm{A}}}
\def\mB{{\bm{B}}}
\def\mC{{\bm{C}}}
\def\mD{{\bm{D}}}

\def\mF{{\bm{F}}}

\def\mI{{\bm{I}}}

\def\mM{{\bm{M}}}

\def\mP{{\bm{P}}}
\def\mQ{{\bm{Q}}}
\def\mR{{\bm{R}}}

\def\mW{{\bm{W}}}
\def\mX{{\bm{X}}}
\def\mY{{\bm{Y}}}

\DeclareMathAlphabet{\mathsfit}{\encodingdefault}{\sfdefault}{m}{sl}
\SetMathAlphabet{\mathsfit}{bold}{\encodingdefault}{\sfdefault}{bx}{n}

\DeclareMathOperator*{\argmin}{arg\,min}

\newtheorem{theorem}{Theorem}[section]

\newtheorem{lemma}[theorem]{Lemma}

\newtheorem{corollary}[theorem]{Corollary}

\newtheorem{proposition}[theorem]{Proposition}

\newtheorem{observation}[theorem]{Observation}

\theoremstyle{remark}

\crefname{theorem}{Theorem}{Theorems}
\Crefname{theorem}{Theorem}{Theorems}
\crefname{lemma}{Lemma}{Lemmas}
\Crefname{lemma}{Lemma}{Lemmas}
\crefname{proposition}{Proposition}{Propositions}
\Crefname{proposition}{Proposition}{Propositions}
\crefname{corollary}{Corollary}{Corollaries}
\Crefname{corollary}{Corollary}{Corollaries}
\crefname{definition}{Definition}{Definitions}
\Crefname{definition}{Definition}{Definitions}
\crefname{observation}{Observation}{Observations}
\Crefname{observation}{Observation}{Observations}
\crefname{example}{Example}{Examples}
\Crefname{example}{Example}{Examples}
\crefname{remark}{Remark}{Remarks}
\Crefname{remark}{Remark}{Remarks}
\crefname{replemma}{Lemma}{Lemmas}
\Crefname{replemma}{Lemma}{Lemmas}
\crefname{reptheorem}{Theorem}{Theorems}
\Crefname{reptheorem}{Theorem}{Theorems}

\newcommand{\colsub}[2]{{{#1}\ifthenelse{\equal{#2}{}}{}{, {#2}}}}
\newcommand{\col}[3]{p^{\left(#1\right)}_\colsub{#2}{#3}(\mathbf{x})}

\newcommand{\tup}[1]{{\left(#1\right)}}

\newcommand{\One}{\mathbf{1}}

\newcommand{\Elt}{\sqsubsetneqq}
\newcommand{\Egt}{\sqsupsetneqq}
\newcommand{\Ele}{\sqsubseteq}
\newcommand{\Ege}{\sqsupseteq}
\newcommand{\Eeq}{\equiv}
\newcommand{\Eneq}{\not\equiv}

\newcommand{\Nat}{\mathbb{N}}
\newcommand{\Real}{\mathbb{R}}
\newcommand{\Cat}{\mathcal{C}}
\newcommand{\CaT}[2]{\mathcal{C}_{#1, #2}}
\newcommand{\Catu}{\mathcal{C}^{\sqcup}}
\newcommand{\CaTu}[2]{\mathcal{C}^{\sqcup}_{#1, #2}}

\newcommand{\Walk}[1]{\ensuremath{\mathsf{wr}^{(#1)}}}
\newcommand{\Walkk}{\ensuremath{\mathsf{wr}}}

\newcommand{\Tree}[1]{\ensuremath{\mathsf{cr}^{(#1)}}}
\newcommand{\Treee}{\ensuremath{\mathsf{cr}}}

\newcommand{\stdml}{\{\!\!\{\,}
\newcommand{\stdmr}{\,\}\!\!\}}
\newcommand{\bigml}{\big\{\kern-0.4em\big\{}

\newcommand{\bigmr}{\big\}\kern-0.4em\big\}}

\newcommand{\chitriv}{\chi_{\operatorname{triv}}}
\newcommand{\chiid}{\chi_{\operatorname{id}}}
\newcommand{\chideg}{\chi_{\operatorname{deg}}}
\newcommand{\chitree}[1]{{\chi_{\mathsf{cr}}^\tup{#1}}}

\renewcommand{\hom}{\ensuremath{\mathsf{hom}}}
\newcommand{\rank}{\ensuremath{\operatorname{rank}}}
\newcommand{\nnz}{\ensuremath{\operatorname{nnz}}}

\newcommand{\reduced}{{reduced} }

\newcommand{\In}{ {\textnormal{ in }} }

\newcommand{\emptyword}{\lambda}
\newcommand{\Graphs}{\mathcal{G}}

\newcommand{\sem}[1]{\llbracket {#1} \rrbracket}
\newcommand{\Ac}{\mathcal{A}}

\newcommand{\restate}[2]{%
    \medskip\noindent\textbf{\cref{#1}.}~\textit{\nameref{#1}}%
    \textit{{#2}}
}

\newcommand{\Sigg}{\overline \Sigma}
\newcommand{\iSigg}{\In\Sigg}
\newcommand{\ovecw}{\overline{\vw}}

\newcommand{\vmin}{\!\rotatebox[origin=c]{90}{\ensuremath{-}}\!}

\newcommand{\RealVV}{\Real^{V\times V}}
\newcommand{\RealV}{\Real^{V}}
\newcommand{\iRealVV}{\In\Real^{V\times V}}

\newcommand{\OneT}{\One^\top}

\newcommand{\pin}{{+}}
\newcommand{\tp}{{t+1}}

\newcommand{\Fg}{\mF^{G}}
\newcommand{\Fh}{\mF^{H}}

\newcommand{\lab}{\mathsf{lab}}
\newcommand{\Fbl}{F^\bullet}

\newcommand{\Tbl}{T^\bullet}

\newcommand{\Fcbl}{{\mathcal{F}^{\bullet}}}
\newcommand{\Gbl}{G^\bullet}

\newcommand{\Lbl}{L^\bullet}

\newcommand{\clsF}{{\ensuremath{\mathcal{F}}}}

\newcommand{\RealSSp}{\Real^{S_t \times S_\tp}}

\newcommand{\mult}{\operatorname{mult}}
\newcommand{\UPDATE}{\operatorname{\small UPDATE}}
\newcommand{\READOUT}{\operatorname{\small READOUT}}
\newcommand{\AGG}{\operatorname{\small AGGREGATE}}
\newcommand{\REDUCE}{\operatorname{\small REDUCE}}

\newcommand{\invEA}{{\mathcal{I}_{\text{R}}}}

\definecolor{takeawaycolor}{HTML}{65730E}

\usepackage[most]{tcolorbox}
\newtcolorbox{takeawaybox}[1][]{
colback=white,
colframe=takeawaycolor!50,
boxrule=1pt,
arc=2pt,
left=5pt,
right=5pt,
top=4pt,
bottom=4pt,
fonttitle=\bfseries,
}
\begin{document}

\maketitle

\begin{abstract}
Message-passing graph neural networks (MPGNNs) dominate modern graph learning. 
Typical efforts enhance MPGNN's expressive power by enriching the adjacency-based aggregation. 
In contrast, we introduce a \emph{graph-level aggregation} over walk incidence-based matrices that are constructed to deliberately trade off some expressivity for stronger and more structured inductive bias.
This approach allows for gradual scaling between classical message-passing and simpler methods based on walks.
Drawing on \emph{weighted automata}, we ensure our gradual scaling is polynomial-time and avoids the combinatorial explosion of colored walks.
Our second result characterizes the expressive power at each scale using homomorphism counts over a hierarchy of \emph{generalized caterpillar graphs}. 
Based on these foundations, we propose Caterpillar GNNs that substantially reduce the number of nodes in the hidden layers of the computational graphs on real-world datasets.
This gradual reduction does not obstruct learning, enabling a systematic study of lower-order expressivity.
We further describe benchmark settings where task-aligned expressivity supports learning.
\end{abstract}

\addtocontents{toc}{\protect\setcounter{tocdepth}{0}}

\section{Introduction}\label{sec:introduction}

Graphs are a powerful structure, capable of representing relational information across various domains such as biology, chemistry, databases, or social sciences. 
Unlike sequences or grids, graphs do not come with a fixed topology of computation: the topology is part of the input.
The established incorporation of this variability relies on the inductive bias of (equivariant) message-passing (MP) in graph neural networks (MPGNNs).
Prior work has shown the limits of MP in capturing structural biases \citep{Xu+2018, Mor+2019}, leading to many extensions of MP.
Learning is also affected by phenomena such as nodal over-smoothing \citep{Oon+2020} and over-squashing \citep{Alo+2021}, motivating direct interventions in neighborhood topology \citep{Top+2022, Gio+2023}.



For graph-level learning, the usefulness of a structural bias depends on the task.
Our motivation is expressivity alignment: retaining structure suited to the patterns relevant for prediction, while keeping the training procedure undisturbed by unaligned biases.
In this work, we develop foundations for tools that can diagnose potentially harmful effects of task-unaligned expressivity.
Our approach is to control the structure available to the neural architecture and study the resulting inductive biases.
We achieve this through an exact characterization of lower-order expressivity and a polynomial-time  construction of a computational graph that is incorporated into a learnable architecture.
This shows a complementary way to organize the computation around expressivity alignment by controlling the structural information embodied in the aggregation matrices.
The proposed set of results provides a controlled inductive bias that can be empirically studied in standard GNN architectures.

Our starting point is colored-walk statistics.
A colored walk is a sequence, or a word, where each letter represents a vertex color visited along a walk.
Its weight is the number of its occurrences in the input graph.
With a constant vertex coloring, this weight reduces to the ordinary walk count.
Coloring vertices by their degree additionally records degree patterns.
We iteratively refine the coloring using the color refinement \citep{Bab+1979}.
Each refinement round can attach richer local structure to each letter, up to the full limit of MP's expressivity.

The number of distinct colored walks grows exponentially with their length in the worst case.
As a result, existing architectures process sequential patterns using random-walk sampling \citep{Ton+2023, Zen+2023, Che+2024}.
Since random sampling sacrifices equivariance, we instead construct a canonical representation of colored-walk counts.
For each walk length, a walk-incidence matrix records occurrences of colored walks terminating at each vertex.
Since each column has one entry per vertex, its column space has dimension at most the number of input vertices. 
The subsequent canonical reduction of this column space draws on the techniques of \emph{weighted (or multiplicity) automata} \citep{Sch+1961, Tze+1996}, recognizing a weighted counterpart of regular languages.
The layer-specific reduced matrices define our \emph{graph-level aggregation} (GLA), constructed in polynomial time while retaining the expressive power of the colored-walk statistics.
 We instantiate GLA in Caterpillar GNNs that enable a systematic study of learning on graphs under controlled lower expressivity over a fixed set of hyperparameters.

 \begin{figure}[t]
  \centering
  \definecolor{homcolor}{RGB}{220, 35, 0}
  \definecolor{walkcolor}{RGB}{0, 35, 220}
  \includegraphics[width=0.95\linewidth]{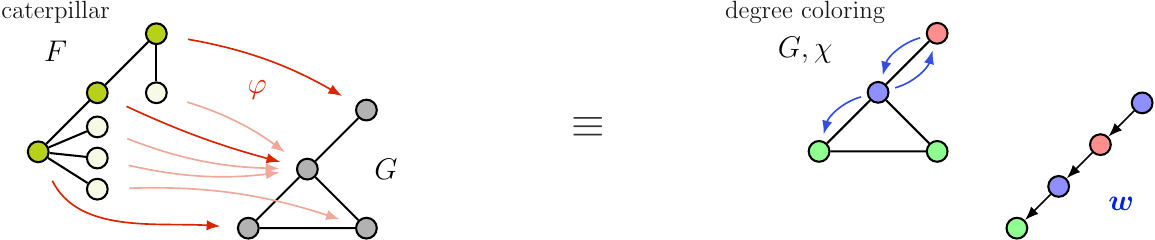}
  \caption{
    Counting over either of the two depicted notions is exactly as expressive on graphs.
    Left: \emph{A~graph homomorphism} $\textcolor{homcolor}{\varphi} \colon F\textcolor{homcolor}{\to}\, G$,
    where $F$ is a caterpillar graph.
    Right: \emph{An~occurrence of a~colored walk} $\textcolor{walkcolor}{\vw} = \mathsf{brbg}$ in the same graph $G$ colored by $\chideg$ according to vertex degrees.
    }
    \label{fig:correspondence}
    \vspace{-.8em}
  \end{figure}

To understand and characterize the graph structure accessible through GLA, or equivalently through colored walks, we use an algebraic definition of expressive power. 
A graph homomorphism maps the vertices of a source graph to the input graph while preserving adjacency. Choosing a class of source graphs, denoted by $\mathcal{F}$, gives an invariant by counting these maps from every graph in $\mathcal{F}$.
The expressivity of some MPGNN extensions can be characterized by graph classes $\mathcal{F}$ that contain all trees, see~\autoref{tab:related_work}.
In the limit, extending $\mathcal{F}$ from trees to all finite graphs yields graph isomorphism, as shown by \citet{Lov+1967}.
\emph{Our work answers the complementary question: which inductive biases arise when $\clsF$ is restricted downwards---to strict subclasses of trees?}

\begin{wraptable}[15]{r}{0.35\textwidth}
\vspace{-10pt}
\small
\centering
\begin{tabularx}{\linewidth}{ll}
\toprule
Higher order$^{b}$ & treewidth$^{g}$ \\
$\mathcal{P}$-enabled$^{c}$ & $\mathcal{P}$-trees$^{c}$ \\
Subgraph Aggr.$^{d,e}$ & apex trees$^{h}$ \\
Spectral Invariant$^{f}$ & parallel trees$^{i}$ \\
\midrule
Vanilla MPGNN$^{a}$ & trees$^{g}$ \\
\midrule 
Caterpillar (ours) & caterpillars\\
\bottomrule
\end{tabularx}
\caption{
MPGNNs and corresp. expressivity bounds using class~$\clsF$.
\footnotesize
$^{a}$\cite{Gil+2017},  
$^{b}$\cite{Mor+2019},  
$^{c}$\cite{Pab+2021},  
$^{d}$\cite{Qia+2022}
$^{e}$\cite{Fra+2022},  
$^{f}$\cite{Zha+2024}, 
$^{g}$\cite{Dvo+2010},  
$^{h}$\cite{Rat+2021},  
$^{i}$\cite{Gai+2025}.
}
\label{tab:related_work}
\end{wraptable}

To that end, we propose a suitable hierarchy of \emph{generalized caterpillar graphs}---trees that yield a path after a bounded number of leaf-removal rounds. 
We call this bound \emph{height}, in particular, height one corresponds to one round of removing all leaves, and defines folklore caterpillars. 
We completely characterize expressivity of this height-controlled hierarchy by counting walks colored by a correspondingly bounded number of color refinement rounds.
In particular, for an input graph~$G$, counting homomorphisms to $G$ over folklore caterpillars is exactly as expressive as coloring $G$ according to vertex degrees and then counting colored walks.
See \autoref{fig:correspondence}.

GLA thus provides a homomorphism canonization over generalized caterpillars of bounded height.
Within the graph class determined by this correspondence, its invariant is sufficient to identify graphs up to isomorphism, as follows from the factorization argument of \citet{Lov+1967}.
This motivates applications where the task depends on caterpillar-like structures such as proteins or polymers.
Finally, we construct a dataset, NStepAddition, benchmarking this identification ability by construction: 
we encode addition of two binary numbers as a small unlabeled caterpillar-like graph.
Solving it only requires expressive power of homomorphism counts over bounded-height caterpillars,
and helps to empirically uncover settings where MP tends to be task-unaligned and suffering from topological bottlenecks.

Our main contributions are as follows.
\begin{enumerate}
    \item We introduce GLA (\cref{sec:efficientaggregation}). We prove its tractability (\cref{thm:algorithm}) and desired expressivity (\cref{thm:walkInvariantEquivalence}).
    The challenge of its complete derivation and proofs we address by developing techniques in automata theory (\autoref{appendix:automata}).
    \item We characterize expressivity of GLA using a hierarchy of \emph{generalized caterpillar graphs} and its graph homomorphism counts (\autoref{sec:expressivity}, \autoref{thm:homht}).
    For this, we develop novel combinatorial arguments in graph theory (\autoref{appendix:homcounts}).
    \item We incorporate GLA into Caterpillar GNNs (\cref{subsec:efficient_matrices}, \cref{eq:scheme_reduce}), and investigate its gradual scaling (\cref{subsec:scaling}). 
    Using walk incidence-based topology (\cref{subsec:topology}), we illustrate how the optimal expressivity can be important when its inductive bias aligns with the task.
\end{enumerate}

Empirically, on TUDataset, we illustrate dataset-specific tradeoffs between accuracy and nodal efficiency of gradual lower-order expressivity of Caterpillar GNN.
 Under practical conditions, GLA substantially reduces the number of nodes in the hidden layers of the computational graphs on datasets from Open Graph Benchmark \citep{Hu+2021}, LRGB \citep{Dwi+2022}, 
 ZINC \citep{Irw+2012},
 MoleculeNet \citep{Wu+2018}, and MalNet-Tiny \citep{Fre+2021}.

\section{Preliminaries}\label{sec:prelim}

\newglossaryentry{adjacency}{
  name={$\mA \In\Real^{V\times V}$},
  description={adjacency matrix of $G$}
  sort=A
}

\newglossaryentry{graph}{
  name={$G=(V,E)$},
  description={undirected graph with finite vertex set $V$ and edge set $E\subseteq V^2$},
  sort=G
}

\newglossaryentry{vertices}{
  name={$V$},
  description={vertex set of the graph $G$},
  sort=V
}

\newglossaryentry{edges}{
  name={$E$},
  description={edge set of the graph $G$, with $E\subseteq V^2$},
  sort=E
}

\newglossaryentry{degree}{
  name={$\deg(u)$},
  description={degree of a vertex $u$, i.e.\ $\deg(u)=|\{v \mid uv \In E\}|$},
  sort=degree
}

\newglossaryentry{n}{
  name={$n$},
  description={number of vertices, $n = |V|$},
  sort=n
}

\newglossaryentry{paths}{
  name={$\mathcal{P}\subset \mathcal{G}$},
  text={$\mathcal{P}$},
  description={class of all finite undirected paths},
  sort=paths
}

\newglossaryentry{paths-bounded}{
  name={$\mathcal{P}_t\subset\mathcal{P}$},
  description={class of paths of length at most $t$, where length means $|E|$},
  sort=paths_t
}

\newglossaryentry{trees}{
  name={$\mathcal{T}\subset \mathcal{G}$},
  text={$\mathcal{T}$},
  description={class of all trees},
  sort=trees
}

\newglossaryentry{rooted-trees}{
  name={$\mathcal{T}^\bullet$},
  description={class of rooted trees},
  sort=trees-rooted
}

\newglossaryentry{rooted-trees-bounded}{
  name={$\mathcal{T}^\bullet_h\subset \mathcal{T}^\bullet$},
  description={class of rooted trees where every vertex is at distance at most $h$ from the root},
  sort=trees-rooted-h
}

\newglossaryentry{multiset-brackets}{
  name={$\stdml - \stdmr$},
  text={$\stdml, \text{ and }\stdmr$},
  description={multiset},
  sort=multisetbrackets
}

\newglossaryentry{indexset-k}{
  name={$[k]$},
  description={(index) set for $k\In\Nat$},
  sort=kindexset
}

\newglossaryentry{identity}{
  name={$\mI \In \Real^{V\times V}$},
  description={self-loop (identity) matrix of $G$},
  sort=I
}

\newglossaryentry{ones}{
  name={$\vone \In \Real^{V\times 1}$},
  description={all-ones column vector indexed by $V$},
  sort=1
}

Let \gls{graph} be an undirected graph with a finite vertex set $V$ and an edge set $E\subseteq V^2$.
Loops are not assumed, and an edge between $u$ and $v$ is denoted by $uv$.
The degree of a vertex $u$ is $\deg(u)=|\{v \mid uv \In E\}|$, and $n=|V|$.
A path is a connected acyclic graph with vertices of degree at most two.
We denote the class of all paths by \gls{paths}, and by \gls{paths-bounded} the subclass of paths of length at most $t$ where length means $|E|$.
A tree is a connected acyclic graph.
We denote by \gls{trees} the class of all trees.
By \gls{rooted-trees}, we mean the class of rooted trees, and by \gls{rooted-trees-bounded}, the class where every vertex is at distance at most $h$ from the root, that is, at most $h$ edges from a root (e.g. \cite[page 8]{Die+2025}).

Multisets are represented using symbols \gls{multiset-brackets}.
Let $X, Y$ be sets, and $x\In X$, $y\In Y$.
The family of all multisets of elements from $X$ is denoted by $\Nat^X$.
For a multiset $m\In\Nat^X$, we access multiplicity of $x$ by $m[x]$.
For a vector $\vv\In\Real^X$, we access its $x$-th component by $\vv [x]$.
For a matrix $\mM\In \Real^{X\times Y}$ (of shape $X\times Y$), we access its entries by $\mM[x, y]$, rows by $\mM[x]$ and columns by $\mM[-, y]$.
Finally, the notation \gls{indexset-k} stands for the set $\{1, 2, \ldots, k\}$ for $k\in\Nat$.
For the graph $G$, we denote its \emph{adjacency matrix} by \gls{adjacency}, that is, $\mA[u, v] = 1$ if $uv \In E$ and $0$ otherwise.
Its \emph{identity} or \emph{self-loop matrix} is denoted by \gls{identity}, and the all-ones vector by \gls{ones}.

\newglossaryentry{features-mp}{
  name={$\vh_\text{MP}^\tup \ell \In \Real^{V\times d}$},
  text={\ensuremath{\vh_\text{MP}^\tup {\ell+1}}},
  description={vertex features after $\ell$ layers of an MPGNN with $d$ channels},
  sort=h-mp
}

\newglossaryentry{feature-mp}{
  name={$\vh_\text{MP} \In \Real^{1\times d}$},
  text={\ensuremath{\vh_\text{MP}}},
  description={vector computed by an MPGNN with $d$ channels},
  sort=h-mpf
}

\paragraph{MPGNNs.}
Throughout, we often represent graphs by matrices, for which we adopt a specific notation.
For a matrix $\mM\In \Real^{X\times Y}$, features of $d$ channels $\vh \In \Real^{Y\times d}$ and $x\In X$, 
we define \[\mult(x, \mM, \vh) \coloneqq \bigml \mM[x, y]\cdot \vh[y] \mid y \In Y\bigmr.\]
Let $\vh_\text{MP}^\tup {0}$ be features in $\Real^{V\times d}$ if given and $\vone$ otherwise.
Then MPGNN of $L$ layers is given for each $u\In V$ and integer $\ell$ such that $0 \le \ell < L$ as follows
\begin{align*}
    {\gls{features-mp}}[u] &= 
    \UPDATE\big( 
        \mult(u, \mI, \vh_\text{MP}^\tup \ell),\,\,
        \AGG\big(
            \mult(u, \mA, \vh_\text{MP}^\tup \ell)  
        \big)
    \big),\\
    {\gls{feature-mp}} &= 
    \READOUT\big(
        \mult(0, {\textstyle\frac{1}{n}}\vone\!^\top, \vh_\text{MP}^\tup{L})
    \big),
\end{align*}
where we use $\mult$ to self-loop with $\mI$, to range over adjacent nodes with $\mA$ of shape $V \times V$
and to collect all nodes with $\frac{1}{n}\vone^\top$ of shape $\{0\} \times V$.
The functions $\AGG$, $\UPDATE$, and $\READOUT$ are specific to each layer.
We omit their indexing and learnable parameters.

\newglossaryentry{Graphs}{
  name={$\Graphs$},
  description={class of all finite undirected graphs},
  sort=Graphs
}

\paragraph{Expressivity and homomorphism counts.}
Let \gls{Graphs} denote the class of all graphs and 
let $\mathsf{f}$ and $\mathsf{g}$ be two functions on $\Graphs$. 
Then the function 
\emph{$\mathsf{f}$~is at least as expressive as $\mathsf{g}$}, denoted by $\mathsf{f} \Ege \mathsf{g}$, if for every two graphs $G$ and $H$, it holds that
$\mathsf{f}(G) = \mathsf{f}(H)\text{ implies }\mathsf{g}(G) = \mathsf{g}(H).$
Next, \emph{$\mathsf{f}$ is (exactly) as expressive as $\mathsf{g}$}, denoted by $\mathsf{f} \Eeq \mathsf{g}$, if $\mathsf{f} \Ege \mathsf{g}$ and $\mathsf{g} \Ege \mathsf{f}$. Furthermore,
\emph{$\mathsf{f}$ is (strictly) more expressive than $\mathsf{g}$}, denoted by $\mathsf{f} \Egt \mathsf{g}$, if $\mathsf{f} \Ege \mathsf{g}$ and $\mathsf{f} \Eneq \mathsf{g}$.
The relation of expressivity is a partial ordering on the family of functions on $\mathcal G$.

\newglossaryentry{homcnt}{
  name={$\hom(\clsF, -)$},
  text={\ensuremath{\hom(\clsF, G)}},
  description={function of graphs assigning vector in $\Nat^{\clsF}$ for a class of source graphs $\clsF$},
  sort=homF
}

For a source graph $F=(V_s,E_s)$, a function $\varphi\colon V_s \to V$ is a \emph{graph homomorphism} $F\to G$ if every edge $uv\In E_s$ implies edge $\varphi(u)\varphi(v)\In E$.
See \autoref{fig:correspondence}.
For a class of source graphs $\clsF\subseteq \Graphs$, 
we define a (possibly infinite) vector of \emph{homomorphism counts over} $\clsF$, denoted by ${\gls{homcnt}}\In \Nat^{\clsF}$, as
$\hom(\clsF, G)[F] = |\{ \varphi \mid \varphi\colon F\to G \}|$ for all $F\In \clsF$.
Note that every class $\clsF\subseteq \Graphs$ induces the function $\hom(\clsF, -)\colon \Graphs\to \Nat^{\clsF}$,
assigning $\hom(\clsF, G)$ to the target graph $G$. 
It always holds that $\clsF\supseteq \clsF'$ implies $\hom(\clsF, -) \Ege \hom(\clsF', -)$.

\newglossaryentry{coloring}{
    name={$\chi$},
    description={coloring function $\chi(G, -)\colon V \to \Sigma'$ for $G$},
    sort=chi
}

\newglossaryentry{chitriv}{
    name={$\chitriv$},
    description={trivial coloring},
    sort=chitriv
}

\newglossaryentry{chiid}{
    name={$\chiid$},
    description={identity coloring},
    sort=chiid
}

\newglossaryentry{chideg}{
    name={$\chideg$},
    description={degree coloring},
    sort=chideg
}

\paragraph{Graph colorings and color refinement.}
A coloring $\chi$ is a map that assigns specific colors to the vertices.
Formally, for each graph $G=(V,E)$, we have a function $\chi(G, -)\colon V \to \Sigma'$ where $\Sigma'$ denotes a color set.
We say coloring \gls{coloring} is a \emph{$\Sigma$-coloring on $G$} if $\Sigma' \subseteq \Sigma$.
A coloring example is the trivial coloring \gls{chitriv}, which assigns $0$ to every vertex $u$ of $G$, $\chitriv(G, u) = 0$.
Therefore, $\chitriv$ is a ${\{0\}}$-coloring on every graph.
Another ``extreme'' is the \emph{identity coloring} \gls{chiid} assigning identities on vertices, $\chiid(G, u) = u$ for vertex $u$, and thus $\chiid$ is a $V$-coloring on $G$.
The \emph{degree coloring} \gls{chideg}, assigns degree to every vertex in $G$ which can be written as $\chideg(G, -) = \deg(-)$.

\newglossaryentry{chitree}{
    name={$\chitree{h}$},
    text={\ensuremath{\chitree{h+1}}},
    description={color refinement coloring at height $h$},
    sort=chicr
}

\newglossaryentry{Treee}{
    name={$\Treee(-)$},
    text={\ensuremath{\Treee}},
    description={function on graphs of color refinement},
    sort=cr
}

A \emph{color refinement} constructs a sequence of 
graph colorings: $\chitree{0}(G, u) = 1$,
and for all $h\ge 0$ and each $u\In V$ as
${\gls{chitree}}(G, u) = 
(\chitree{h}(G, u),
\bigml \chitree{h}(G, v) \mid  
uv \In E\bigmr)$.
Secondly, it defines a sequence of functions on graphs
$\Tree{h}(G) = \bigml \chitree{h}(G, u) \mid u\In V\bigmr$;
and, finally, the function on graphs: ${\gls{Treee}}(G) = \{\Tree{h}(G) \mid  h \In \Nat\}$.
For instance, the first coloring is as expressive as the trivial: $\chitree{0}(G, -) \Eeq \chitriv(G, -)$,
and the second exactly as the degree coloring: $\chitree{1}(G, -) \Eeq \chideg(G, -)$.

Let $\chi$ be a $\Sigma$-coloring on $G$.
A \emph{walk} in $G$ is a sequence of vertices $v_1, v_2, \ldots, v_t$ such that $v_i v_{i+1} \In E$ for $i\In [t-1]$. 
A special case is a \emph{path} in $G$, which is a walk with all vertices distinct.
A \emph{colored walk} is a word $\va = a_1 a_2 \dots a_t$ such that $a_i = \chi(G, v_i)\In \Sigma$ for $i\In [t]$.
At the same time, the sequence $v_1, v_2, \dots, v_t$ is an \emph{occurrence} of $\va$ in $G$.
See~\autoref{fig:correspondence}.
We say that vertex $u$ is \emph{incident} to colored walk $\va$ if $u = v_t$,
and \emph{adjacent} if $uv_t \in E$.
We denote by $\Sigma^t$, resp. $\Sigma^{\le t}$ the set of all words over $\Sigma$ of length exactly $t$, resp. at most $t$; and by $\Sigma^*$ the set of all words. 
Note that $\Sigma^0 = \{\emptyword\}$ where $\emptyword$ is the \emph{empty word}.

\section{Constructing Graph-Level Aggregation}\label{sec:efficientaggregation}

This section introduces graph-level aggregation (GLA) and its instantiation in Caterpillar GNNs.
Although we start from colored walks, the main idea is to avoid one-by-one sequential processing.
For each length~$t$, we collect colored-walk statistics in a walk-incidence matrix $\mW_t$.
Since the number of possible colored walks grows exponentially, we select a canonical linearly independent subset of columns, yielding a reduced space of dimension at most linear in the number of vertices of the graph.
We then project and restrict the adjacency and identity matrices onto consecutive layer-specific reduced spaces.
The resulting matrices form the layer-specific computational graph of GLA.

\begin{wrapfigure}[18]{r}{0.5\textwidth}
    \vspace{-1.15em}
    \hfill
    \includegraphics[width=0.99\linewidth]{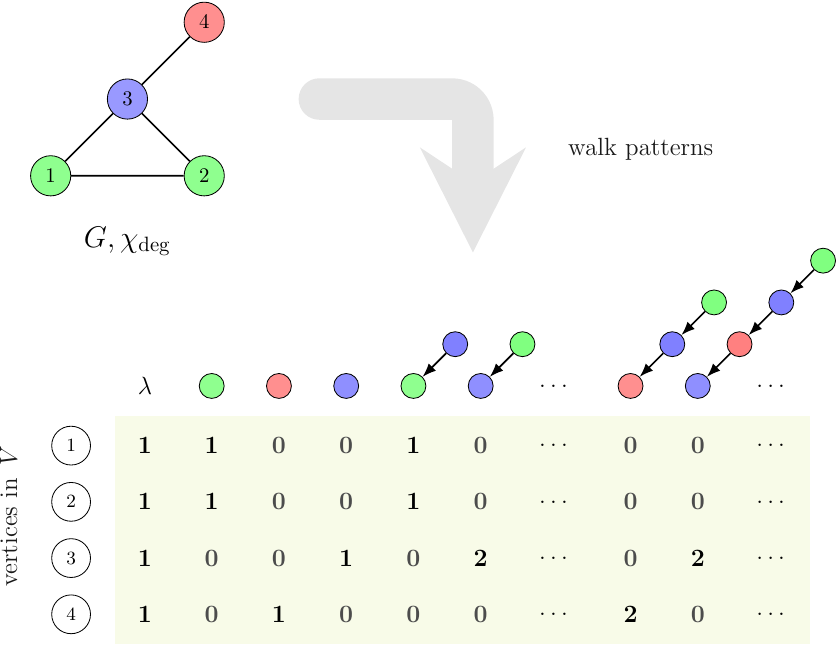}
    \caption{
        Top: Graph $G$ with vertices colored by $\chideg$.
        Colors red, green and blue depict degrees 1, 2 and 3, respectively. 
        Bottom: Walk incidence matrix $\mW$ of shape $V\times \Sigma^*$. 
        The entry $\mW[3, \mathsf{gb}]$ equals $2$, since 
        vertex $3$ terminates two occurrences of the colored walk $\mathsf{gb}$ in $G$.
        }
        \label{fig:neighborhood}
\end{wrapfigure}

Part I formalizes the canonical reduction of colored-walk statistics,
Part II exploits this reduction to construct matrices defining GLA and its instantiation in Caterpillar GNNs.
Part III explains how a single height parameter~$h$ can control the tradeoff between reduction and expressivity.
Omitted proofs are given in \autoref{appendix:automata}.

\subsection{Part I: Sequential patterns}\label{subsec:sequential}

For tractable incorporation of lower-order inductive biases, we improve a way of processing sequential patterns further motivated by \cref{thm:homht} as analyzed in \cref{sec:expressivity}.
In the language of colored walks, many successful machine learning approaches first \emph{sample} a tractable number of random walks and then process the visited colors as sequences with either kernels  \citep{Bor+2005, Kri+2022} or neural networks \citep{Ton+2023, Zen+2023, Che+2024}. 
Our approach is a fundamental \emph{reversal} of these steps:
given a prescribed colored walk, we count its occurrences.

\newglossaryentry{walk-incidence-matrix}{
  name={$\mW_t \in \Real^{V\times \Sigma^t}$},
  text={$\mW_t$},
  description={walk-incidence matrix of length $t$ colored walks in $G$ for a $\Sigma$-coloring $\chi$},
  sort=W
}

Crucially, we consider a tractable and canonical subset of colored walks. 
As shown later (\autoref{thm:walkInvariantEquivalence}), this subset suffices to determine all other colored walks.
In contrast to prior sampling-heavy methods, we preserve determinism, equivariance and intended expressivity.
To formalize our reversal, we relate vertices and colored walks using incidence matrix.

\paragraph{Walk incidences.}
For a given graph $G$ with $\chi$ a $\Sigma$-coloring, and a given length $t\ge 0$, we define the \emph{walk-incidence matrix} \gls{walk-incidence-matrix} of shape $V\times \Sigma^t$ 
for each $u\In V$ and $\va \In \Sigma^t$ by the entry
\begin{align}
    \mW_t[u, \va]
    \label{eq:walkincidence}
\end{align}
that equals the number of occurrences of colored walk $\va$ terminating in vertex $u$.
Each column $\mW_t[-, \va] \In \Nat^V \subseteq \Real^V$ corresponds to a \emph{multiset of vertices} incident to \add{colored walk}~$\va$.
For instance, the column $\mW_1[-, c]$ coincides with vertices $u$ of color $c=\chi(G, u)$ in $\Sigma$.
By \emph{convention}, the empty walk is incident to every vertex, $\mW_0[u, \emptyword]=1$ for $u\In V$.
See the (infinite) matrix $\mW = \left[\mW_0 | \mW_1 | \cdots \right]$ of shape $V\times\Sigma^*$ in \autoref{fig:neighborhood}.

\newglossaryentry{canonical-walk-subset}{
    name={$S_t \subseteq \Sigma^t$},
    text={$S_t$},
    description={canonical subset of colored walks of length $t$ in $G$ for a $\Sigma$-coloring $\chi$},
    sort=S
}

\paragraph{Walk selection.} 
The row dimension $V$ of incidence matrices $\mW_t$ remains fixed, while the column dimension $\Sigma^t$ grows exponentially in $t$.
We avoid this growth by selecting subsets of $\Sigma^t$, 
such that the induced columns of $\mW_t$ form a basis of the column space of $\mW_t$.
The definition proceeds inductively: $S_0 = \{\emptyword\} = \Sigma^0$,
and for known \gls{canonical-walk-subset}, the set $S_{t+1}\subseteq \Sigma^{t+1}$ satisfies the following conditions:
\begin{enumerate}[label=(\roman*)]
    \item for every $\va c\In S_{t+1}$ there is $\va \In S_t$ (\emph{prefix-closedness}),\label{cond:i}
    \item the columns of $\mW_{t+1}$ induced by $S_{t+1}$ are \emph{linearly independent}, and \label{cond:ii}
    \item the set $S_{t+1}$ is \emph{lexicographically minimal} among other sets satisfying \ref{cond:i} and \ref{cond:ii}. \label{cond:iii}
\end{enumerate}
The last Condition~\ref{cond:iii} together with $S_0 = \{\emptyword\}$ ensures uniqueness, making this selection canonical.
Condition~\ref{cond:ii} implies the upper bound $|S_t| \le  \rank(\mW_t) \le |V|$,
\add{reducing the representation of potentially $|V|^t$ columns to at most $|V|$.}
Finally, Condition~\ref{cond:i} allows for a \add{time-efficient algorithm} reminiscent of breadth-first search with linear independence checking \add{using matrix decomposition}.
\begin{theorem}\label{thm:algorithm}
    Let $\chi$ be a $\Sigma$-coloring on a graph $G$ with $n$ vertices, and let $T\In \Nat$ be a limit. 
    Then the canonical subsets $(S_t)_{t=0}^{T}$ defined above are computable in time $\mathcal{O}(Tn^{\omega + \epsilon})$ for any $\epsilon > 0$ where $\omega$ is the exponent of matrix multiplication.\footnote{As of 2026, $\omega\le 2.371339$ \citep{Alm+2024}. For instance, we can assume that $\omega + \epsilon \le 2.372\mathbf{.}$ } 
\end{theorem}

\begin{takeawaybox}
Exponential colored-walk statistics can be compressed to at most linear statistics.
Three selection rules ensure the compression agrees on graphs with the same statistics,
and it can be computed with the asymptotics of graph spectra estimation via $T$ QR decompositions.
\end{takeawaybox}

\subsection{Part II: Reduced Matrices}\label{subsec:efficient_matrices}
Up to this point, we have considered walks as sequences processed one by one.
However, such a representation is inefficient, in particular because it ignores shared prefix structure, as is well-known from string-searching algorithms (e.g., suffix trees \citep{Wei+1973}).
To overcome this inefficiency, we organize walk-incidence statistics into matrices, where selected colored walks correspond to columns and also rows.
This matrix formulation enables hierarchical aggregation of walk patterns, analogous to how message passing (MP) aggregates over neighborhoods, but dropping the assumption of repeating fixed neighborhood structure at each layer.

MP on graph $G=(V, E)$ with $n$ vertices consists of two steps:
aggregation via adjacency operator $\mA$ and update via self-looping operator $\mI$ both $\RealV \to \RealV$.
We aim to deliberately restrict this repeating mechanism: informally, we lock the corresponding vector space $\RealV$ by projecting those operators into $\Real^{S_{t+1}}\to \Real^{S_{t}}$ 
of possibly lower dimension as implied by Condition~\ref{cond:ii}.

\newglossaryentry{walk-indicence-base}{
    name={$\mB_t \in \Real^{V\times S_t}$},
    text={$\mB_t$},
    description={column submatrix of walk-incidence matrix $\mW_t$ induced by $S_t$},
    sort=Bt
}

For a $\Sigma$-coloring on $G$, and integer $t\ge 0$, we denote by \gls{walk-indicence-base} of shape $V \times S_t$ the submatrix of $\mW_t$ (of shape $V \times \Sigma^t$) that keeps only the columns indexed by $S_t$ (see \autoref{fig:lock} for an illustration).
Condition~\ref{cond:ii} guarantees that every matrix $\mB_t$ is tractable and has full rank.

\begin{figure}[t]
    \centering
    \includegraphics[width=0.9\linewidth]{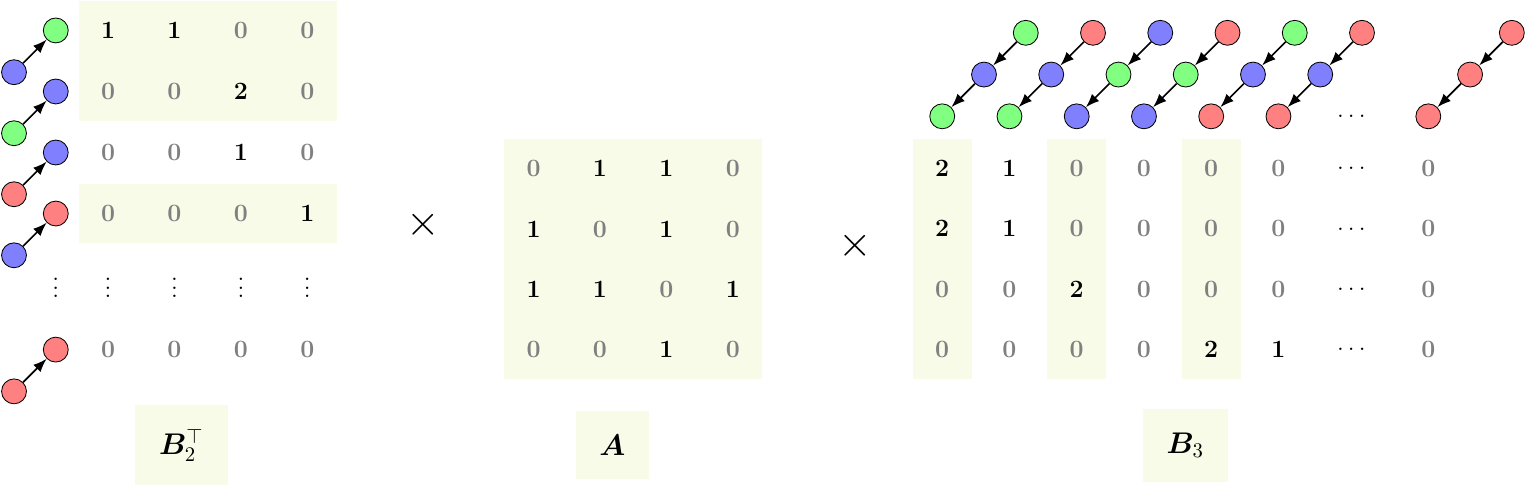}
    \caption{
        Locking the adjacency matrix $\mA$ of the graph $G$ as in \autoref{fig:neighborhood}.
        Matrix $\mB_3$ of shape $V\times S_3$ is a submatrix of $\mW_3$ (of shape $V\times \Sigma^3$) induced by highlighted columns,
        similarly, $\mB_2^\top$ and $\mW_2^\top$.
    }\label{fig:lock}
\end{figure}

\newglossaryentry{efficient-matrix}{
  name={$\mC_t^\mM \in \Real^{S_t\times S_{t+1}}$},
  text={$\mC_t^\mM$},
  description={$t$-th reduced $\mM$-matrix for  $\mM \in \Real^{V\times V}$},
  sort=C-M
}
\newglossaryentry{efficient-adjacency}{
  name={$\mC_t^\mA \in \Real^{S_t\times S_{t+1}}$},
  text={\ensuremath{\mC_t^\mA}},
  description={$t$-th \reduced adjacency matrix for  $\mA \in \Real^{V\times V}$},
  sort=C-A
}
\newglossaryentry{efficient-identity}{
  name={$\mC_t^\mI \in \Real^{S_t\times S_{t+1}}$},
  text={\ensuremath{\mC_t^\mI}},
  description={$t$-th \reduced self-loop matrix for  $\mI \in \Real^{V\times V}$},
  sort=C-I
}

\paragraph{Graph-Level aggregation.}
Let $\mM$ be a matrix of shape $V\times V$, such as $\mA$ or $\mI$ above, then a~\emph{$t$-th~\reduced $\mM$-matrix} \gls{efficient-matrix} of shape $S_t\times S_{t+1}$ is defined as 
\begin{align}
    \mC^\mM_t &= (\mB^{\top}_t\!\mB_t)^{\!-1}\!\mB^{\top}_t\,\mM\,\mB_{t+1},
    \label{eq:catmat}
\end{align}
which solves the least-squares problem
$\argmin_{\mC}\left\|\,\mB_t \mC - \mM \mB_{t+1}\,\right\|_F$.
Informally, the unique operator $\mC^\mM_t\colon \Real^{S_{t+1}}\!\!\to \Real^{S_{t}}$ is
the best approximation of $\mM$ in the basis indexed by canonical $S_t$.

\newglossaryentry{invariant-ea}{
  name={$\invEA(-, \chi)$},
  text={\ensuremath{\invEA(G, \chi)}},
  description={function of graphs of graph-level aggregation invariant for a coloring $\chi$},
  sort=inv-ea
}

We call a \emph{graph-level aggregation (GLA)} the collection of the first $n$ reduced adjacency and identity matrices into the graph invariant
\begin{align}
    \gls{invariant-ea} = \left\{\left(\gls{efficient-adjacency}\!,\, \gls{efficient-identity}\right) \mid 0 \le t < n\right\}.
    \label{eq:invariant}
\end{align}
Since reduced matrices are indexed by colored walks, we compare $\invEA$ directly across graphs.
Its expressive power (\autoref{sec:expressivity}) as the function on graphs $\invEA(-, \chi)$ motivates the following model.

\newglossaryentry{features-ea}{
  name={$\vh_\text{EA}^\tup \ell \in \Real^{S_{t_\ell}\times d}$},
  text={\ensuremath{\vh_\text{GLA}^\tup {\ell+1}}},
  description={features after $\ell$ layers of a Caterpillar GNN with $d$ channels},
  sort=h-ea
}

\newglossaryentry{feature-ea}{
  name={$\vh_\text{GLA}\in\Real^{\{\lambda\}\times d}$},
  text={\ensuremath{\vh_\text{GLA}}}, 
  description={vector computed by a Caterpillar GNN with $d$ channels},
  sort=h-eaf
}

\paragraph{Caterpillar GNNs.}
We now describe how reduced matrices are used across $L$ layers. 
\emph{Caterpillar GNN} initializes with 
$\vh_\text{GLA}^\tup {0}(\va c) = \REDUCE(c, \stdml\!\!\stdmr, \vone)$ for colored walk $\va c \In S_{L}$.
Then, at each layer $\ell$ such that $0\le \ell \le L-1$ with $t_\ell = L - \ell$ and for each colored walk $\va c \In S_{t_\ell}$ we have
\begin{align}
    {\gls{features-ea}}[\va c] &= 
    \REDUCE\Big( c,\,
        \mult(\va c, \mC^\mI_{t_\ell}, \vh_\text{GLA}^\tup \ell),\,\AGG\big(
            \mult(\va c, \mC^\mA_{t_\ell}, \vh_\text{GLA}^\tup \ell)  
        \big)
    \Big),\notag\\
    {\gls{feature-ea}} &= 
    \READOUT\big(
        \mult(\lambda, \mC^\mI_{0}, \vh_\text{GLA}^\tup{L})
    \big),\label{eq:scheme_reduce}
\end{align}
where the function $\REDUCE$ replaces the usual $\UPDATE$ and targets a colored walk $\va c$ instead of a fixed vertex and requires color $c$ as an additional input.
A visual side-by-side comparison with the standard MP is given in \autoref{fig:efficient_neighborhood}.

\begin{takeawaybox}
  Reduced walk statistics provide the subspaces onto which message-passing matrices are projected and restricted.
  Using these reduced matrices, GLA defines its layer-specific computational graph.
  The simpler the initial statistics, the smaller these matrices become.
\end{takeawaybox}

\subsection{Part III: Gradual Lower-Order Scaling}\label{subsec:scaling}

The vertex coloring $\chi$ controls the coarseness of distinguished colored walks and thus governs the resulting inductive bias of GLA.
In our approach to GLA, we utilize colorings of the color refinement $\chi=\chitree{h}$,
simplifying the choice for end-users to the parameter $h\ge 0$ called \emph{height}.
To guide our exploration, we analyze two extreme cases: trivial coloring $\chitriv$ and identity coloring $\chiid$.
For illustration of a resulting reduced computational graph, consult \autoref{fig:efficient_neighborhood} 
for the case of graph in \autoref{fig:neighborhood},
and \autoref{fig:comp_proteins} for the case of a protein graph.

Under $\chitriv$, all vertices share the same color $0$. Thus, every walk of length $t$ has color $\vz_{t} = 00\cdots 0$ (constant word of length $t$).
Hence, every set $S_t$ collapses to the singleton of $\vz_{t}$, and computation over $L$ layers collapses to a linear sequence of length $L$.
The REDUCE function receives the color~$0$ together with multisets of form $\mult(\vz_{t}, \mC, \vh)$ containing a single pair $(m, \vh[\vz_{t}])$,
where $m$ is a normalized count of plain walks (cf., walk partition \citep{Chu+1997}).
Formally, we have:

\begin{observation}\label[observation]{obs:specialwalkpartition}
  Let $\chitriv$ be the $\{0\}$-coloring on a graph $G$ with at least one edge.
  Then for every $t\ge 0$: (a) it holds that $|S_t| = 1$;
  (b) the only entries,
  $\mC_t^\mI[\vz_{t}, \vz_{{t+1}}]$ and $\mC_t^\mA[\vz_{t}, \vz_{{t+1}}]$ are equal to
  $\frac{\vone^\top \mA^{2t+1} \vone}{\vone^\top \mA^{2t} \vone},$ and 
  $\frac{\vone^\top \mA^{2t+2} \vone}{\vone^\top \mA^{2t} \vone},$ respectively.
\end{observation}

\begin{wrapfigure}[15]{r}{0.5\textwidth}
  \vspace{-2.15em}
  \hfill
  \includegraphics[width=0.99\linewidth]{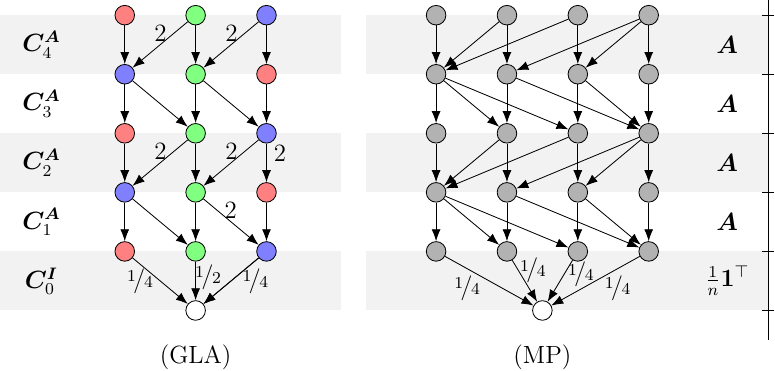}
  \caption{
    Comparison of computational graphs (without self-loops): (GLA) \emph{graph-level aggregation} (ours), and (MP) \emph{message-passing}
    for the graph $G$ and coloring $\chideg$ as given in~\autoref{fig:neighborhood} (left).
    Connections between layers are given by (GLA) $t$-th reduced graph matrices;
    (MP)~copies of the adjacency matrix and the global readout.
    For unit weights, we omit labels.
    }
    \label{fig:efficient_neighborhood}  
  \end{wrapfigure}

When each vertex is assigned a distinct color under $\chiid$, every colored walk in the graph has its unique occurrence. 
Hence, every set $S_t$ reaches the maximum size $|V|$, with one colored walk per vertex. 
In this regime, the reduced matrices coincide entry-wise with the original matrices.
Moreover, if $\REDUCE$ ignores its first parameter then GLA reaches the computational graph of standard message passing.
\begin{proposition}\label[proposition]{obs:specialmessagepassing}
    Let $\chiid$ be the $V$-coloring on a graph $G=(V, E)$ with $n$ vertices.
    By $\vv_{t, u}$, we denote the (unique) word in $S_{t}$ with the last color $u \In V$.
    Then for every $t\ge 1$: (a) it holds that $|S_0|=1$, and $|S_t| = |V|$;
    (b) for entries $\mC_0^\mI[\lambda, u] = \frac{1}{n}$, and 
    $\mC_t^\mI[\vv_{t,u}, \vv_{t+1,v}] = \mI[u, v]$
    and $\mC_t^\mA[\vv_{t,u}, \vv_{t+1,v}] = \mA[u, v]$.
\end{proposition}

\begin{takeawaybox}
  The \emph{height parameter $h$} controls the tradeoff between per-layer reduction and expressivity.
  Smaller $h$ yields smaller reduced layers, down to a single-node layer weighted by plain walk counts at $h=0$. 
  Larger $h$ preserves more of the structural information, up to message passing.
\end{takeawaybox}


\newglossaryentry{Cat}{
    name={$\Cat_h$},
    text={\ensuremath{\Cat_h}},
    description={class of all caterpillars of height at most $h$},
    sort=Cat
}
\newglossaryentry{CaTx}{
  name={$\mathcal{C}_{h, t} \subseteq \mathcal{C}_h$},
  text={\ensuremath{\CaT{h}{t}}},
  description={class of all $(h, t)$-caterpillars},
  sort=CatT
}


\begin{figure*}[h]
  \vspace{1em}
  \centering
\begin{tikzpicture}[every node/.style={anchor=center}, font=\normalsize, scale=0.75]
  \matrix (m) [matrix of math nodes, nodes in empty cells, row sep=0.45cm] {
  \hom(\mathcal{P},-) 
  &\Elt & \hom(\Cat_1,-) 
  &\Elt & \hom(\Cat_2,-) 
  &\Elt{\text{$\cdots$}}\Elt& \hom(\Cat_h,-) 
  &\Elt{\text{$\cdots$}}\Elt& \hom(\mathcal{T},-)\\
  \invEA(-,\chitriv) 
  &\Elt & \invEA(-,\chideg) 
  &\Elt & \invEA(-,\chitree{2}) 
  &\Elt{\text{$\cdots$}}\Elt& \invEA(-,\chitree{h}) 
  &\Elt{\text{$\cdots$}}\Elt& \Treee(-)\\
  };
  \foreach \col in {1,3,5,7,9}{
      \node[rotate=90] at ($(m-1-\col)!0.5!(m-2-\col)$) {$\Eeq$};
  }
  \label{diag:hierarchy}
\end{tikzpicture}
\vspace{-0.5em}
  \caption{
      \emph{Strictly separated lower-order expressivity hierarchy} of graph-level aggregations $\invEA$ along the height parameter $h$
      related to caterpillar graph classes $\mathcal{P} = \Cat_0 \subset \Cat_1 \subset \cdots \subset \Cat_h \subset \cdots \subset \mathcal{T}$.
  }\label{fig:hierarchy}
\end{figure*}

\section{Characterizing Expressivity of GLA using Caterpillar Hierarchy}\label{sec:expressivity}

\begin{wrapfigure}[18]{r}{0.39\textwidth}
  \vspace{-1.15em}
  \hfill
  \includegraphics[width=0.99\linewidth]{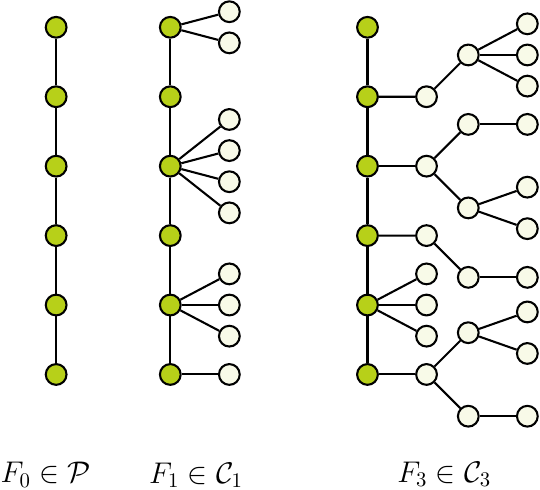}
  \caption{Caterpillar graphs with highlighted (possible) spine (green).
  Graph $F_0$ is a path of length $6$ and hence a $(0,6)$-caterpillar.
  Graph $F_1$ is a $(1,6)$-caterpillar, and $F_3$ is  a $(3,6)$-caterpillar.
  }
  \label{fig:caterpillars}
\end{wrapfigure}

This section characterizes the expressivity of GLA (\autoref{eq:invariant}) using graph homomorphism counts,
providing a \emph{scale} of expressive power of the associated lower-order inductive biases, see \autoref{fig:hierarchy}.
Our two main results, \cref{thm:homht} and \cref{thm:walkInvariantEquivalence}, established separately in the respective subsections, together give the vertical equivalences.
The last vertical equivalence involving $\mathcal T$ is due to {\citet[Theorem 7]{Dvo+2010}}.
Note that color refinement $\mathsf{cr}$ symbolizes message-passing (MP).
The strictness of horizontal bounds follows from \cref{lem:strictseparation}, adopting the recent results of \citet{Rob+2022, Sch+2025, Neu+2026}.
We proceed with formal definitions.


\paragraph{Caterpillar graphs.}
A \emph{caterpillar of height at most $h$ and length at most $t$}, or shortly \emph{$(h, t)$-caterpillar} is a graph $F$ constructable as follows:
take a sequence of rooted trees in $\mathcal{T}^\bullet_h$, i.e.,  $(L_1, {s_1}), \ldots, (L_t, {s_t})$ and connect consecutive roots with edges
so that vertices $s_1, s_2, \ldots, s_t$ form a path $S$.
We call $S$ a \emph{spine}, the rooted trees \emph{legs}, and their sequence a \emph{leg sequence}.
We denote the class of all caterpillars of height at most $h$ by \gls{Cat}, and by \gls{CaTx}
the subclass of $(h, t)$-caterpillars.

For instance, every caterpillar in $\mathcal{C}_{0}$ is a path, $\mathcal{P}=\mathcal{C}_{0}$, see examples in \autoref{fig:caterpillars}.
The folklore caterpillars here correspond to $\Cat_1$ and are often used in graph theory \citep{Har+1973, Elb+1987}.
Another generalization of caterpillars via \emph{hair-length} is due to \cite{Bur+1986}.

\subsection{Homomorphism Counts over Caterpillars are as Expressive as Colored Walks}\label{sec:homtowalk}

\newglossaryentry{Walkk}{
    name={$\Walkk(-, \chi)$},
    text={\ensuremath{\Walkk(G, \chi)}},
    description={function on graphs of colored walk refinement for a coloring $\chi$},
    sort=Walkk
}

Let $\chi$ be a $\Sigma$-coloring on a graph $G$ with $n$ vertices.
We define a sequence of multisets $\Walk{t}(G, \chi)\In \Nat^{\Sigma^t}$ for each $t\ge 0$ with $\va\In \Sigma^t$ as:
$\Walk{0}(G, \chi)[\emptyword] = n$, $\Walk{t}(G, \chi)[\va] \text{ equals the number of occurrences of $\va$ in $G$}$, and $\gls{Walkk} = \{\Walk{t}(G, \chi)\mid t\In \Nat\}$.

Note that our colored walk refinement is distinct from what is usually called walk refinement, i.e. \citep{Lic+2019}.
The following result explains the importance of colored walks, which is not ad hoc but rather structurally grounded.
\begin{theorem}\label{thm:homht}
    For every $h, t\ge 0$, it holds that 
   $ \hom({\CaT h t}, -) \Eeq \Walk{t}(-, \chitree{h}).$
\end{theorem}
\emph{(Proof in~\autoref{appendix:homcounts})}. A direct consequence of \autoref{thm:homht} is: to capture the expressive power of homomorphism counts over folklore caterpillars (for instance) of length $t$, it suffices to color the vertices by their degrees and record every occurrence of a colored walk of length $t$ by $\Walk{t}(G, \chideg)$, see \autoref{fig:correspondence} for example of such a graph homomorphism.

\subsection{Graph-Level Aggregation is as Expressive but Tractable}\label{sec:walktoea}
The previous result depicted more clearly the semantics of caterpillar homomorphisms, which is still not computationally tractable.
As we observe, the number of distinct colored walks in a graph can be large, exponential in the worst case.
Therefore, it is crucial that we introduce a more reduced but as expressive representation of $\Walkk(-, \chi)$,
\add{which we state for \emph{any} coloring of vertices.}
\begin{theorem}\label{thm:walkInvariantEquivalence}
    For every coloring $\chi$, it holds that 
    $ \Walkk(-, \chi) \Eeq \invEA(-, \chi).$
\end{theorem}
\emph{(Proof in~\autoref{appendix:automata})}. \add{\emph{Sketch of proof.}}
Central to our arguments is that all the colored walks occurring together in a graph $G$ with a $\Sigma$-coloring $\chi$
form a regular language. As a consequence, rather than (naively) enumerating all elements of the language, we construct an automaton $\mathcal{A}(G, \chi)$ recognizing such a language.
To reflect multiplicities of occurrences in the language, we use a weighted variant of finite automata introduced by \citet{Tze+1996}.

The ``$\Ege$-direction'' of the equivalence requires that $\invEA(G, \chi)$ derived from $\Ac(G, \chi)$ depends purely on the recognized language and ordering on $\Sigma$ via the canonical walk selection ($S_t$ in \cref{subsec:sequential}). 
For the other ``$\Ele$-direction'', given two automata recognizing distinct languages, subsequent matrices of $\invEA$ decompose back to automata transitions (\cref{lem:giGeA}) to detect the difference.


\section{Experiments}\label{sec:experiments}
We next turn to an empirical analysis of Caterpillar GNN (\autoref{eq:scheme_reduce}) incorporating graph-level aggregation (GLA).
Because expressivity of GLA (\autoref{sec:expressivity}) is controlled by its height, we propose experiments to empirically evaluate the behavior of the resulting inductive bias.
Two scenarios are considered: (I.) a controlled benchmark isolating topology-driven preference for stronger task-aligned inductive bias, and (II.) real-world graph-level tasks investigating the impact of height (\cref{subsec:scaling}) on the trade-off between nodal efficiency and performance.
We defer full implementation details to Appendix~\ref{appendix:gnns}, and the training setup to Appendix~\ref{appendix:implementation}.

\subsection{Scenario I: Reducing a Bottleneck}\label{subsec:topology}
Prior to any processing of a graph $(V, E)$, the neighborhood topology $\tau(E)$ on $V$ specifies which vertices are considered close, namely those in neighborhoods.
We instead consider an alternative \emph{incidence topology} $\tau(\chi)$ on $V$, induced by a coloring $\chi$:
two vertices are considered close if they are incident \emph{or} adjacent to a common colored walk of length $T$.
Since a colored walk may have multiple occurrences, this captures relationships beyond direct neighbors.
We use $\tau(\chi)$ as a model to study different inductive biases in graph learning, grounded in lower-order concepts such as colored walks as shown by \cref{thm:homht}.

We illustrate this with our $\textsc{nStepAddition}$ benchmark.
Given two integers of at most $T$ bits, take a graph with two occurrences of a colored walk $a_1 \cdots a_T$.
We associate each number with one occurrence as follows: encode the $i$-th bit of the integer in a vertex adjacent to $a_1\cdots a_i$.
This yields a graph embedding of two $T$-bit integers, and the classification task is whether their sum equals a target integer $N$.
Under the incidence topology $\tau(\chi)$ corresponding bit positions are naturally close, while standard topology $\tau(E)$ may obscure such alignments.
Therefore, we evaluated Caterpillar GNN on $\textsc{nStepAddition}$ with increasing height $h$, and compared to MPGNN.
The results are presented in \autoref{fig:expressivity_hurts} with more detailed information in Appendix~\ref{appendix:implementation}.

\begin{figure}[h]
  \centering
  \includegraphics[width=0.7\linewidth]{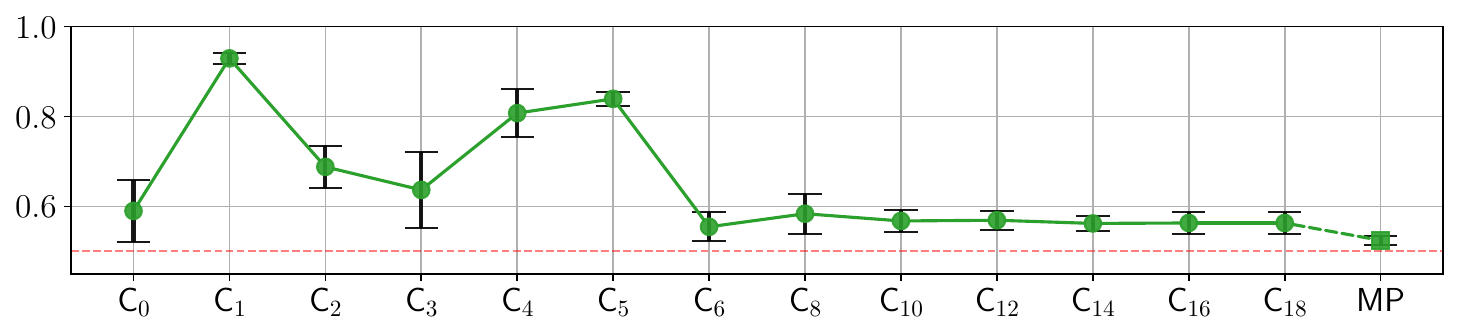}
  \caption{
    \textsc{nStepAddition}: Trading off expressivity for a suitably structured task-aligned inductive bias.
      $\textsf{C}_h$~denotes a model of Caterpillar GNN with height~$h$~(ours),
      while $\textsf{MP}$ refers to a model of MPGNN.
      The~$y$-axis shows validation accuracy.
  }
  \label{fig:expressivity_hurts}
\end{figure}

\paragraph{Importance of topology.} The extremal results for $\textsf{C}_1$ and $\textsf{MP}$, see \autoref{fig:expressivity_hurts}, highlight a clear difference between the two models.
In MPGNN, information propagates according to $\tau(E)$.
A hypothesis arises that Caterpillar GNN propagates information according to topology $\tau(\chi)$.
We validate empirically on \textsc{nStepAddition} that model $\textsf{C}_1$, as well as topology $\tau(\chideg)$, aligns bit positions for effective digit-by-digit addition,
while MP within $\tau(E)$ effectively promotes learning values separately for each input pair, resulting in almost missing generalization.
We detail the observed behavior in \cref{appendix:descent}.

\subsection{Scenario II: Nodal Efficiency}\label{subsec:nodal_efficiency}

We evaluate GNNs (\autoref{fig:height}) on TUDataset~\citep{Mor+2020} with respect to \emph{nodal efficiency}, i.e. the average number of nodes of the computational graph (\autoref{fig:efficient_neighborhood}).
We fixed the number of layers for models to ensure a relative comparison.
To this end, hyperparameters are per-dataset, so that the behavior can be attributed solely to the height parameter.
In our experiments (\autoref{fig:height}), every real-world dataset exhibits a unique behavior under increasing height.
This suggests that every type of data may contain patterns organized in varying topologies resulting in distinct preferences for inductive biases.
Effectively, the height parameter \emph{shapes the model performance and nodal efficiency}.
Choosing an appropriate height may improve both accuracy and nodal efficiency (\autoref{tab:dataset_statistics_acc}).

Alongside nodal efficiency, we analyze the density of the induced computational graphs, both relative to MP and as the fraction of nonzeros in the reduced matrices (\cref{appendix:density_computational_graphs}). Empirically, for five layers, the reduced matrices have a similar or smaller number of nonzero entries than MP, depending on the height $h$ (\autoref{tab:sparsity_computation_graphs}). 
However, lower height can increase absolute density: for example, MalNet-Tiny rises from $0.8\%$ under MP to $24.7\%$ under GLA, ogbg-molpcba from $8.8\%$ to $22.0\%$, and Peptides-func from $2.1\%$ to $15.4\%$ (\autoref{tab:dataset_statistic_group}). 
Thus, sparse input graphs often yield smaller but denser representations, which can be favorable for dense linear-algebra kernels.

\section{Related Work}
Graph homomorphisms are an active area of research in graph learning.
One line of work uses homomorphism counts directly as \emph{features} \citep{Pab+2021, Mae+2024, Jin+2024}, or as \emph{embeddings} \citep{Hoa+2020, Thi+2022}.
Several \emph{extensions} of MPGNNs formally demonstrate the expressivity of homomorphism counts over classes extending beyond trees, listed in \autoref{tab:related_work}, including \citep{Zha+2023, Pao+2024}.
Another line of work enhances expressivity without relating to homomorphisms.
This includes cycle representations \citep{Yan+2022, Bau+2025}, path representations \citep{Gas+2023, Gra+2024},
distance encodings \citep{Li+2020, Zha+2023b}, and spectral information such as \citep{Def+2017, Kre+2021},
which is in contrast to our study of lower expressivity.

In our results, we rely on theoretical study of homomorphism counts that traces back to \cite{Lov+1967, Lov+2012}, and their connection to Weisfeiler-Leman refinement \citep{Wei+1968}, which is due to \cite{Dvo+2010, Del+2018}. 
Further developments include quantum isomorphism via homomorphisms over planar graphs  \citep{Man+2020}, with further connection explored by \cite{Gro+2021b, Kar+2024}.
Connections of weighted automata and homomorphism canonization given in this work
motivated tight bounds for computational complexity of homomorphism indistinguishability \citep{Cer+2025}, originally studied by \citet{Sep+2024}, and for complexity of weighted automata equivalence, followed by recent improvements in \citep{Chi+2026}.
Our canonization generalizes to higher-order colored walks \citep{Cer+2026}.

Learning on sequential patterns such as walks has also been approached via non-equivariant random-walk kernels \citep{Bor+2005, Kri+2022}.
Other work investigates slowing down message-passing as a regularizing inductive bias \citep{Bau+2022}.
Recently, the role of the computational graph in deep learning has been explored \citep{Vit+2025}.
In addition, least squares-based operators have been applied to cross-network optimization \citep{Wan+2024}, or graph coarsening \citep{Yuj+2020, Sta+2023}. These operators target graphs at a different level of abstraction, not considering layer-specific walk incidence matrices or homomorphism counts.

\section{Conclusion}

We introduced a height-parameterized mechanism for controlling the computational graph of GNNs. 
The resulting \emph{Caterpillar GNNs} expose a controlled trade-off: lower heights yield condensed computational graphs, while higher heights recover richer structural information.
We characterize this hierarchy through colored walks and homomorphism counts over generalized caterpillars.
Empirically, the mechanism can improve validation accuracy while substantially reducing computational nodes; on unattributed IMDB-BINARY, one setting retains approximately 6\% of the MP baseline nodes (\autoref{tab:dataset_statistics_acc}).
 
 Motivating application regimes are those in which the objective of a graph-level task depends on sequential patterns representable by colored walks, while additional branching patterns are less relevant.
The benchmark NStepAddition and similar tasks within the walk incidence topology realize such a regime by construction.
Downstream tasks on proteins, polymers, and other caterpillar-like structures motivate targeted evaluation when their labels depend primarily on shared backbone patterns.

Reduced walk-based matrices are precomputed and reused across training epochs, and hence practical runtime depends on preprocessing, graph size, and matrix density, which we explored using our experimental implementation (\crefrange{appendix:density_computational_graphs}{appendix:modern_architectures}; \cref{tab:dataset_statistic_group,tab:running_times}).
These factors motivate efficient implementations that exploit the smaller layers of computational graphs.

The height of GLA governing the expressive power of this homomorphism canonization is currently selected by validation over a discrete range, a heuristical or even differentiable selection is a direction for future work.
Further directions include directed or edge-labeled graphs and richer equivariant architectures, such as investigating analogous controls to treat the expressivity gap between 1-WL and 2-WL or Subgraph GNNs.
To extend our approach from graph-level to node-level prediction, additional care is needed to preserve output vertices during the reduction process which is currently based purely on colored-walk incidences.

\addtocontents{toc}{\protect\setcounter{tocdepth}{2}} 

\begin{ack}
I thank Tim Seppelt, Petar Veli\v{c}kovi\'{c}, and Martin Ritzert for their insightful early feedback and valuable corrections to the initial version of the paper.
I also thank Matthias Lanzinger, Diaaeldin Taha, Len Feremans, Joey De Pauw, and Jens Leysen for their valuable suggestions, careful reading, and encouragement throughout this project. 
I thank the anonymous reviewers for their constructive feedback and careful reading of the paper and its appendix. 
I am especially grateful to Floris Geerts for his sustained guidance since the inception of this project.

Marek \v{C}ern\'{y} is funded by the University of Antwerp (BOF, Doctoral project 47103).
\end{ack}


\newpage
\bibliographystyle{plainnat}
\bibliography{references}
\newpage
\onecolumn
\appendix
\renewcommand{\contentsname}{\add{Table of Contents}}
\tableofcontents
\newpage

\begin{figure}[t]
    \centering
    \includegraphics[width=0.6\textwidth]{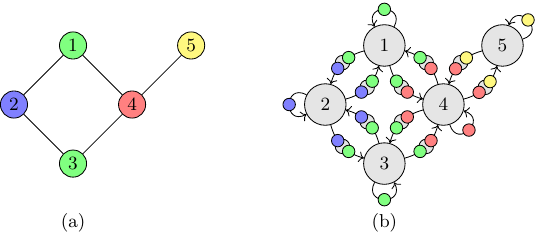}
\caption{
(a)~An example graph $G$ on vertices $\{1,2,3,4,5\}$ with a $\Sigma$-coloring $\chi$, where $\Sigma = \{\mathsf{r, b, g, y}\}$.
(b)~The weighted automaton $\mathcal{A}(G, \chi)$ that accepts a weighted language of words corresponding to colored walks in $G$ (\cref{lem:equivalencewalkrefinementinvariant}). 
The weights of the language represent the number of occurrences of each walk in $G$.
Transitions corresponding to directed edges $uv \In E(G)$ are represented by matrices $\mM(\chi(u)\vmin\chi(v)) = \mP_{\chi(u)}\mA\mP_{\chi(v)}$, while transitions corresponding to loops at vertex $u \In V(G)$ are given by matrices $\mM(\chi(u)) = \mP_{\chi(u)}\mI\mP_{\chi(u)}$.
}
\label{fig:walks2words}
\end{figure}

\section{Weighted Automata}\label{appendix:automata}
In this section, we introduce \emph{weighted automata} \citep{Sch+1961, Tze+1996} alternatively called \emph{multiplicity automata} in the theoretically oriented literature \citep{Kie2+2013}.
Then, we apply insights from automata theory, using them as a key technical tool to establish the results of \autoref{sec:efficientaggregation} and \autoref{thm:walkInvariantEquivalence} from \autoref{sec:expressivity}.

\begin{figure}[h]
\centering
\vspace{1em}
\scalebox{0.8}{
\begin{tikzpicture}[every node/.style={}, scale=0.9, font=\small]
    \usetikzlibrary{decorations.pathreplacing}
    \node[rectangle, inner sep=1.7em, fill=black!5] (A) at (0,0) {{}$\Walkk(-, \chi)$};
    \node[rectangle, inner sep=1.7em, fill=black!5] (B) at (5.5, 0) {$\sem{\Ac(-, \chi)}$};
    \node[rectangle, inner sep=1.7em, fill=black!5] (C) at (11, 0) {$\invEA(-, \chi)$};
    \draw[draw=black!30, very thick] (A)  --node[fill=none, draw=none, midway, above](lAB) {\Large$\Eeq$} (B);
    \draw[draw=black!30, very thick] ([yshift=5pt]B.east) to node[fill=none, draw=none, midway, above](lBC) {\Large$\Ele$} ([yshift=5pt]C.west);
    \draw[draw=black!30, very thick] ([yshift=-5pt]B.east) to node[fill=none, draw=none, midway, below](lBC2) {\Large$\Ege$} ([yshift=-5pt]C.west);
    \draw[draw=black!30, very thick] (A) -- ($(A.north)+(0,+1.2)$)  -- node[fill=none, draw=none, midway, above](lAC) {\Large$\Eeq$} ($(C.north)+(0,+1.2)$) --  (C);
    \node[above of=lAB, yshift=-1em] {\autoref{thm:walkAutomatonEquivalence}};
    \node[below of=lBC2, yshift=1em] {\autoref{lem:giLeA}};
    \node[above of=lBC, yshift=-1em] {\autoref{lem:giGeA}};
    \node[above of=lAC, yshift=-1em] {\autoref{thm:walkInvariantEquivalence}};
    \draw [draw=black!30, very thick, decorate,decoration={brace,amplitude=5pt,mirror,raise=4ex}] ($(A.south)+(0,-0.25)$) -- ($(B.south)+(-0.5,-0.25)$) node[midway,yshift=-3em]{\Cref{appendix:walksautomata}};
    \draw [draw=black!30, very thick, decorate,decoration={brace,amplitude=5pt,mirror,raise=4ex}] ($(B.south)+(0,-0.25)$) -- ($(C.south)+(+0.5,-0.25)$) node[midway,yshift=-3em]{\Cref{appendix:constraints}};
\end{tikzpicture}
}
\caption{
        Sections \ref{appendix:walksautomata} and \ref{appendix:constraints} establish the result of \autoref{thm:walkInvariantEquivalence} for a coloring $\chi$ by relating the following three functions on graphs:
        Colored walk refinement $\Walkk(-, \chi)$, semantics of graph walk automaton $\sem{\Ac(-, \chi)}$, and the matrices of graph-level aggregation $\invEA(-, \chi)$.
}
\end{figure}

\paragraph{Automata.} 
A \emph{weighted finite automaton}, or here simply \emph{automaton}, is a tuple 
\begin{equation}\Ac = (Q, \Sigma, \mM, \alpha, \omega),\end{equation}
where $Q$ is a finite set of states;
$\Sigma$ is a finite alphabet;
$\mM(-)\colon \Sigma \to \Real^{Q\times Q}$ is a per-symbol mapping of transition matrices,
$\alpha$ in $\Real^{1\times Q}$ is the initial state (row) vector, 
and $\omega$ in $\Real^{Q\times 1}$ is the final state (column) vector.
\paragraph{Semantics.}
Given an automaton $\Ac = (Q, \Sigma, \mM, \alpha, \omega)$,
we extend the mapping $\mM\colon \Sigma \to \Real^{Q\times Q}$ to words
as follows: for a given word $\vw=w_1w_2\ldots w_t\In \Sigma^*$, we define
$\mM(\vw) = \mM({w_1}) \mM({w_2}) \cdots \mM({w_t})\in\Real^{Q\times Q}$, and $\mM(\emptyword) = I\in\Real^{Q\times Q}$ for the empty word $\emptyword$ in $\Sigma$.
The \emph{semantics} of the automaton $\Ac$ is a function $\sem \Ac \colon \Sigma^* \to \Real$,
interpreted as a formal series, defined by
\begin{align}
    \sem \Ac(\vw) = \alpha \mM(\vw) \omega \in\Real \quad \text{ for all } \vw \In \Sigma^*.
\end{align}
Two automata $\Ac$ and $\Ac'$ are \emph{equivalent} if their semantics are equivalent,
that is, $\sem \Ac(\vw) = \sem{\Ac'}(\vw)$ for all $\vw\In\Sigma^*$.
The value $\sem{\Ac}(\vw)$ is called the \emph{weight} of $\vw$.
For a symbol $a\In\Sigma$, let $a^k$ denote 
the word formed by repeating $a$ exactly $k$-times.

\subsection{Graph Walks and Automata Semantics}\label{appendix:walksautomata}
As in the main text, we assume a graph $G$ and $\chi$ a $\Sigma$-coloring on $G$,
for which we now define a weighted finite automaton $\Ac(G,\chi) = (Q, \Sigg, \mM, \OneT, \One)$ defined as follows.
The states are corresponding to the vertices of the graph $Q\coloneqq V(G)$,
and the alphabet is induced by the colors of the vertices and edges as follows:
\begin{align}
    \Sigg \coloneqq 
    \left\{ \chi(u) \mid u \in V(G)  \right\} \cup 
    \left\{ \chi(u)\vmin\chi(v) \mid uv \in E(G) \right\},
\end{align}
where we consider both $a$ and $a\vmin b$ as a single symbol in $\Sigg$, for some original colors $a,b\In\Sigma$.
For $a$ in $\Sigg$, the \emph{partition matrix}  $\mP_a \In \Real^{V\times V}$ is the diagonal matrix defined as $\mP_a[u, u] = 1$ if $\chi(u) = a$ and $0$ otherwise. For each symbol $a$ or $a\vmin b\In\Sigg$, 
we define the transition matrices 
using the adjacency matrix of $G$ and the coloring-dependent partition matrices
as $$
\mM(a) \coloneqq \mP_a = \mP_a \mI \mP_a \iRealVV \text{ and }  \mM(a\vmin b)\coloneqq \mP_a \mA \mP_b\iRealVV.
$$
We set the initial and final vectors to the all-one vectors.
We depict an example of a graph and a coloring and the corresponding automata in~\autoref{fig:walks2words}.

We also recall the following from the main paper. For a given graph $G$ with $\chi$ a $\Sigma$-coloring, and a given length $t\ge 0$, we define the \emph{walk-incidence matrix} $\mW_t=\mW_t(G, \chi)$ of shape $V\times \Sigma^t$ 
for each $u\In V$ and $\va \In \Sigma^t$ by
\begin{align}
    \mW_t[u, \va]& \text{ is the number of occurrences of $\va$ that terminate in vertex $u$}.
    \label{eqR:walkincidence}
\end{align}

\paragraph{Colored walk refinement.}
Let $\chi$ be a $\Sigma$-coloring on a graph $G$ with $n$ vertices.
A \emph{colored walk refinement} is a sequence of multisets $\Walk{t}(G, \chi)\In \Nat^{\Sigma^t}$ defined for each $t\ge 0$ with $\va\In \Sigma^t$ as
\begin{align*}
    \Walk{0}(G, \chi)[\emptyword] &= n,\\
    \Walk{t}(G, \chi)[\va]&\quad\text{ equals the number of occurrences of $\va$ in G},\text{ and }\\
    \Walkk(G, \chi)& = \{\Walk{t}(G, \chi)\mid t\In \Nat\}.
\end{align*}

The main result of this subsection is the equivalence of graph-induced weighted finite automata and colored walk refinement.
\begin{theorem}\label{thm:walkAutomatonEquivalence}
For every coloring $\chi$ it holds that
\[        \sem{\Ac(-, \chi)} \Eeq \Walkk(-, \chi).\]
\end{theorem}

An implication of \autoref{thm:walkAutomatonEquivalence} is that instead of the sums of columns of walk-incidence matrices, we can walk directly with the automata semantics,
for which we can \emph{choose} the representation of the automaton potentially more suitable for our application.

To prove this theorem, we need a couple of intermediate results. We first show that the partition matrices $\mP_a$ are projection matrices. 

\begin{proposition}[Partition matrices]\label[proposition]{prop:partition}
Let $a$ and $b$ be colors in  $\Sigma$ then for the partition matrices the following hold:
\begin{enumerate}
    \item $\mP_a = \mP_a^2$ is a projection, and
    \item $\mP_a\mP_b$ is all-zero if $a \neq b$.
\end{enumerate}
\end{proposition}
\begin{proof}
For the first part, we have for all $u, v \In V(G)$:
\begin{align*}
    (\mP_a\mP_a)[u, v] &= 
    \sum_{w\in V(G)}\mP_a[u, w]\mP_a[w, v] = 
    \sum_{\substack{w = u}}1\cdot \mP_a[w, v] = \sum_{u=w=v}1\cdot 1,
\end{align*}
which is $1$ if $u=v$ and $0$ otherwise.
Similarly, for the second part, given the colors $a, b\In\Sigma$ we have for all $u, v \In V(G)$:
\begin{align*}
    (\mP_a \mP_b)[u,v] =& 
    \sum_{w\in V(G)}\mP_a[u, w]\mP_b[w, v] = 
    \sum_{\substack{
        u = w \\ a =\chi(w) 
    }} 1 \cdot \mP_b[w, v] = 
    \sum_{\substack{
        u = w,  w= v\\ a =\chi(w), \chi(w)  = b
    }} 1 \cdot 1,
\end{align*}
which is $1$ if $u=v$ and $a=b$, and $0$ otherwise.
Thus, $\mP_a \mP_b$ is all-zero matrix if $a\neq b$.
\end{proof}

\begin{observation}\label{obs:wordform}
Let $\vw$ in $\Sigg$ be a word with a non-zero weight, $\sem{\Ac(G, \chi)}(\vw) > 0$,
then $\vw$ is of the following form:
\begin{align*}
    (a_1^{k_1})
    (a_1\vmin a_2)
    (a^{k_2}_2)
    (a_2 \vmin a_3)
     (a^{k_3}_3)\cdots 
     (a^{k_{t-1}}_{t-1})
     (a^{k_{t-1}}_{t-1} \vmin a^{k_{t}}_{t})
     (a^{k_t}_t),
\end{align*}
where $a_i\In\Sigma$ are colors, $k_i\In\Nat$ are non-negative integers,
\end{observation}
\begin{proof}
For the sake of contradiction, consider a word $\vw$ that contains a subword $a b\vmin c \iSigg$, $c\vmin b a \iSigg$, or $ab\iSigg$ such that $a\neq b$ and weight $\sem{\Ac(G, \chi)}(\vw) > 0$.
In the first case,
by the definition of semantics, we get for any $\vw_1, \vw_2\In\Sigma^*$:
\begin{align*}
\sem{\Ac(G,\chi)}(\vw_1 a b\vmin c\, \vw_2 )
&=\OneT \mM(\vw_1)\mM(a)\mM(b\vmin c)\mM(\vw_2)\One\\
&=\OneT \mM(\vw_1)\mP_a \mP_b \mA \mP_c\mM(\vw_2)\One\\
 &=\OneT \mM(\vw_1) 0 \mM(\vw_2)\One = 0,
\end{align*}
where we used (1.) of \cref{prop:partition} for the second equality. 
Thus, we have $\sem{\Ac(G,\chi)}(\vw_1 a b\vmin c\, \vw_2 ) = 0$,
and symmetrically $\sem{\Ac(G,\chi)}(\vw_1 c\vmin b a\, \vw_2 ) = 0$,
and analogously $\sem{\Ac(G,\chi)}(\vw_1 ab\, \vw_2 ) = 0$.
\end{proof}

\begin{lemma}\label[lemma]{lem:correspondence}
For all $t\In \Nat$, $a_i\In\Sigma$ and $k_i\In\Nat$ for $i\In [t]$ the weight of the following words is equal:
\begin{enumerate}
    \item $(a_1^{k_1})
    (a_1\vmin a_2)
    (a^{k_2}_2)
    (a_2 \vmin a_3)
     (a^{k_3}_3)\cdots 
     (a^{k_{t-1}}_{t-1})
     (a^{k_{t-1}}_{t-1} \vmin a^{k_{t}}_{t})
     (a^{k_t}_t)$, and
     \item $(a_1\vmin a_2)
    (a_2\vmin a_3)
    \cdots 
     (a_{t-1} \vmin a_{t}).\label{eq:correspondence}$
\end{enumerate}
\end{lemma}
\begin{proof}
    This follows from  (1) of \cref{prop:partition}. Indeed, it
    suffices to observe that 
    \begin{align*}
        &\mM\left(a_i^{k_i} \right) =\mM\left(a_i \right), \text{ and }\\
        &\mM\left((a_i^{k_i}) (a_i\vmin a_{i+1})\, (a^{k_{i+1}}_{i+1})\right) =
        \mP^{k_i}_{a_i} \mP_{a_i} \mA \mP_{a_{i+1}} \mP^{k_{i+1}}_{a_{i+1}} =
        \mP_{a_i} \mA \mP_{a_{i+1}} = 
        \mM\left(a_i\vmin a_{i+1}\right).
    \end{align*}
   Then, from the semantics of automata, the statement on the equal weights follows.
\end{proof}

Therefore,
by \cref{lem:correspondence},
we can characterize the semantics of automata $\Ac(G, \chi)$, 
using only words of the form as in \cref{eq:correspondence} 
without loss of generality.
Consequently, by Observation~\ref{obs:wordform},
we shall use a more natural \emph{simplified notation} 
$$a_1\vmin a_2 \vmin a_3 \cdots a_{t-1} \vmin a_{t}$$ for the word $(a_1\vmin a_2)(a_2\vmin a_3)\cdots (a_{t-1} \vmin a_{t})$.

We now make the first connection between weighted automata and walk-incidence matrices.
\begin{lemma}\label[lemma]{lem:equivalencewalkrefinementinvariant}
    Let $G$ be a graph, $\chi$ a $\Sigma$-coloring on $G$, and $(w_1, w_2, \dots, w_t)\In\Sigma^t$ a colored walk.
    Then the following holds:
    \begin{align}
     \mM\left(w_t\vmin w_{t-1}\vmin \dots \vmin w_1\right) \One =
     \mW_t(G, \chi)[-, w_1w_2\cdots w_t].
    \end{align}
\end{lemma}
\begin{proof}
    By induction on the length $t$.
    For the base case $t=0$, 
    it holds $\mM(\emptyword)\One = \mI \One = \One = \mW_0(G, \chi)[-, \emptyword]$.

    For the induction step, suppose that the lemma holds for $t$.
    We take a word $w_1\vmin w_2\vmin \dots \vmin w_\tp$ and consider a vertex $u\In V(G)$. Then,
    \begin{align*}
        (\mM\left(w_\tp\vmin w_{t}\vmin \dots \vmin w_{1}\right)\One)[u]  &=
        \left(\mM(w_\tp \vmin w_{t})\mM\left(w_t\vmin w_{t-1}\vmin \dots \vmin w_{1}\right)\One\right)[u] \\
        &=\sum_{v\in V(G)} \mM(w_\tp \vmin w_{t}) \mW_t(G, \chi)[v, w_1w_2\cdots w_{t}] \\
        &=\sum_{v\in V(G)} \mP_{w_\tp}[u, u]\mA[u, v] \mP_{w_{t}}[v, v] \mW_t(G, \chi)[v, w_1w_2\cdots w_{t}] \\
        &=\sum_{\substack{
            v\in V(G) \\
            w_\tp = \chi(v) \\
            w_{t} = \chi(u)
        }}
        \mA[v, u] \mW_t(G, \chi)[v, w_1w_2\cdots w_{t}] \\
        &=\sum_{\substack{
            uv \in E(G) \\
            \chi(v) = w_{t} \\
            \chi(u) = w_\tp
        }}
         \mW_t(G, \chi)[u, w_1w_2\cdots w_{t}] \\
        &=\mW_\tp(G, \chi)[u, w_1w_2\cdots w_\tp],
    \end{align*}
    where the first and third equalities follow from the definition of $\mM$,
    the second from the induction hypothesis,
    the fourth and fifth follow from the definition of $\mP_{w_\tp}$ and $A$.
    Finally, the last equation follows from the observation 
    that we can find every occurrence of the walk of color $w_1 w_2 \cdots w_t w_\tp$
    ending in the vertex $u$, given that $\chi(u) = w_\tp$,
    as an occurrence of the walk of color $w_1 w_2 \cdots w_t $ ending at one of its neighbors $v$,
    which is of color $\chi(v) = w_{t}$.
\end{proof}

\begin{theorem}\label{thm:equivalencewalkrefinementinvariant}
    Let $G$ be a graph, $\chi$ a $\Sigma$-coloring on $G$, and $(w_1, w_2, \dots, w_t)\In\Sigma^t$ a colored walk.
    Then the following holds:
    \begin{align}
     \sem{\Ac(G, \chi)}\left(w_t\vmin w_{t-1}\vmin \dots \vmin w_1\right) =
     \sum_{u\in V(G)}\mW_t(G, \chi)[u, w_1w_2\cdots w_t].
    \end{align}
\end{theorem}
\begin{proof}
    By the definition of semantics, we have:
    \begin{align*}
        \sem{\Ac(G, \chi)}\left(w_t\vmin w_{t-1}\vmin \dots \vmin w_1\right) &= 
        \OneT \mM(w_1\vmin w_2\vmin \dots \vmin w_t)\One \\
        &=\OneT \left(\mM(w_1\vmin w_2\vmin \dots \vmin w_t)\right)^\top\One \\
        &=\OneT \mM(w_t\vmin w_{t-1}\vmin \dots \vmin w_1)\One \\
        &=\OneT \mW_t(G, \chi)[u, w_1w_2\cdots w_t] \\
        &=\sum_{u\in V(G)}\mW_t(G, \chi)[u, w_1w_2\cdots w_t],
    \end{align*}
    where the second equality follows from the symmetry of all matrices $\mM$,
    since $\mP_a$ and $\mA$ are symmetric matrices, $a \in \Sigma$,
    the fourth equality follows from \cref{lem:equivalencewalkrefinementinvariant}.
\end{proof}

\emph{Proof of \autoref{thm:walkAutomatonEquivalence}.} Follows from \autoref{thm:equivalencewalkrefinementinvariant},
by definition of colored walk refinement.
$\hfill\qedsymbol$

\newpage
\subsection{Minimization of Weighted Automata}\label{appendix:automataMinimization}

In this subsection, we recall the definition of the \emph{minimal weighted automaton} (cf. \citet{Kie2+2013}, \citet{Keifer2020notes}).
Unlike for the standard finite automata, in the weighted case, the concrete variant is unique only up to the change of the basis of the vector space $\Real^{Q}$. 
On the bright side, every such minimal weighted automaton has the unique dimension, that is the number of states $|Q|$,
and the unique (canonical) word subset $S\subseteq\Sigma^*$ of size at most $|Q|$.
Also, the minimal weighted automaton can be computed in time $\mathcal{O}(n^3|\Sigma|)$.

In our case, the graph-induced automaton $\Ac(G, \chi)$ are completely symmetric 
in the sense that $\sem{\Ac(G, \chi)}(w_1\vmin w_2\vmin \dots \vmin w_t) = \sem{\Ac(G, \chi)}( w_{t}\vmin w_{t-1}\vmin \cdots \vmin w_1)$
for $w_i\In\Sigma$, but also in the sense that the initial vector can be interchanged as $\alpha^\top = \omega = \One$,
and similarly matrices $\mM(a) = \mM(a)^\top$ and $\mM(a\vmin b) = \mM(b\vmin a)^\top$.
It follows that the forward and backward steps of the automata minimization described are spanning the same vector space.
And thus if there is a minimal automaton $A_1=(Q_1, \Sigg, \mM_1, \OneT, \One)$ for the graph induced automaton $\Ac(G, \chi)$,
it holds that there is a matrix of a full rank $\mF\In\Real^{V\times Q_1}$ such that for all $a\In\Sigma$:
\begin{align}
    \OneT \mF= \OneT, \quad
    \One =  \mF \One, \quad
    \mF\mM_1(a)  = \mM(a) \mF, \quad
    \mF\mM_1(a\vmin b)  =  \mM(a\vmin b) \mF.
    \label{eq:fmatequivalence}
\end{align}

In general, by e.g. \citet{Keifer2020notes}, if we have two minimal automata $\Ac_{i} = (Q_i, \Sigg, \mM_i, \alpha_i, \omega_i)$ for $i=1,2$,
then there is an invertible matrix $\mQ\In\Real^{Q_2\times Q_1}$ such that for all $a\In\Sigma$:
\begin{align}
    \alpha_2 \mQ =\alpha_1, \quad
    \omega_2 = \mQ \omega_1, \quad
    \mQ\mM_2(a) = \mM_1(a) \mQ, \quad
    \mQ\mM_2(a\vmin b) = \mM_1(a\vmin b) \mQ.
    \label{eq:qmatequivalence}
\end{align}

Inspired by the minimization procedure of weighted automata, also cf. \citet{Kie2+2013},
we propose a similar procedure \emph{layered canonical word search}, see Algorithm~\ref{alg:layeredcws}, to compute the canonical word subsets $S_t(G, \chi)$ for all $t\In\Nat$, as given in \cref{subsec:sequential}. The main distinction from the automata minimization algorithm is that we keep separate queue, base, and the word subset for each layer $t$,
and thus we are ensuring linear independence for each layer $t$ separately.

\begin{figure}[t]
    \centering
    \begin{minipage}{0.8\textwidth}
    \begin{algorithm}[H]\label{alg:layeredcws}
        \caption{Layered canonical word search}
        \KwIn{
        Number of layers $T$,
        ordered alphabet $(\Sigma, \le)$, 
        graph $G$, coloring $\chi$
        }
        \KwOut{$S_0,S_1, \dots S_T$ finite subsets of $\Sigma^*$}
        \For{$t \gets 0$ \textbf{to} $T$}{
    $S_t \gets \emptyset$\\
    $\mathcal{B}_t \gets \emptyset$\\
    queue $Q_t\gets []$\\
        }
    $Q_0.push(\emptyword)$\\
    \For{$t \gets 0$ \textbf{to} $T$}{
        \While{\textbf{not} $Q_t.empty()$}{
        $\vw \gets Q_t.pop()$\\
        $\gamma \gets \mW_t(G, \chi)[-, \vw]$\\
        \If{$\operatorname{rank}\left(\mathcal{B}_t\cup \{\gamma\}\right)\,\mathsf{>}\,|\mathcal{B}_t|$}{
            $\mathcal{B}_t \gets \mathcal{B}_t \cup \{\gamma\}$\\
            $S_t \gets S_t \cup \{\vw\}$\\
            \If{$t < T$}{
            \ForEach{$a$ \textbf{in} $\Sigma$}{
                $Q_\tp.push(\vw a)$\\
                }
                }
            }
        }
    }
    \Return $S_0, S_1, \dots, S_T$
    \end{algorithm}
    \end{minipage}
\end{figure}

We recall the Conditions given in \cref{subsec:sequential}:
The definition proceeds inductively: $S_0(G, \chi) = \{\emptyword\} = \Sigma^0$,
and for known $S_t(G, \chi)$, the set $S_{t+1}(G, \chi)\subseteq \Sigma^{t+1}(G, \chi)$ satisfies the following:
\begin{enumerate}[label=(\roman*)]
    \item for every $\va c\In S_{t+1}(G, \chi)$ there is $\va \In S_t(G, \chi)$ (\emph{prefix-closedness}),\label{cond_r:i}
    \item the columns of $\mW_{t+1}$ induced by $S_{t+1}(G, \chi)$ are \emph{linearly independent}, and \label{cond_r:ii}
    \item the set $S_{t+1}(G, \chi)$ is \emph{lexicographically minimal} among other sets satisfying \ref{cond_r:i} and \ref{cond_r:ii}. \label{cond_r:iii}
\end{enumerate}

\begin{lemma}[Correctness]\label[lemma]{lem:layeredcws}
    Let $G$ be a graph, $\chi$ a $\Sigma$-coloring on $G$
    then the Algorithm~\ref{alg:layeredcws} computes the canonical word subsets $S_t(G, \chi)$, 
    satisfying Conditions (i), (ii), and (iii), for all $t\In \Nat$.
\end{lemma}
\begin{proof}
    We proceed by induction on $t$. It is observed that in the for-loop of the algorithm ranging over $t$,
    we only add to the list $S_t$ and to the set $\mathcal{B}_t$, working only with the elements from the queue $Q_t$, 
    and adding new elements to the queue $Q_\tp$ based on the words we added to $S_t$.
    For the base case $t=0$, 
    the queue $Q_0$ is initialized with the empty word $\emptyword$,
    for which $\mW_0[-, \emptyword]$ is the all-one vector $\One$, and as the base $\mathcal{B}_0$ is empty and its rank $0$.
    Thus, we have $\mathcal{B}_0 = \{\One\}$ and $S_0 = \{\emptyword\}$.
    
    For the induction step, we assume that the algorithm computes $S_t$ and $\mathcal{B}_t$ correctly.
    In the beginning of the $(t+1)$-th iteration,
    the queue $Q_\tp$ contains the words of length $t+1$,
    of the form $\vw a$ for all $a\In\Sigma$ and $\vw\In S_t$, thus satisfying the condition (i).
    If we add $\gamma = \mW_\tp[-, \vw a]$ to $\mathcal{B}_\tp$, and $\vw$ to $S_\tp$,
    then we have $\operatorname{rank}(\mathcal{B}_\tp\cup \{\gamma\}) > |\mathcal{B}_\tp|$, thus words in $S_\tp$ are satisfying condition (ii).
    The last condition (iii) follows from the fact that we always add to $Q_\tp$ possible candidates
    by the foreach-loop over $\Sigma$ in a lexicographical order, and we keep the order while processing the queue.
\end{proof}

\begin{lemma}[Complexity]\label[lemma]{lem:layeredcwscomplexity}
    There is an implementation of Algorithm~\ref{alg:layeredcws}  that for a given $T\In \Nat$, 
    computes the canonical word subsets $S_1, S_2, \dots, S_T$,
    in time $\mathcal{O}(Tn^\omega)$, where $n$ is the number of vertices in $G$.
\end{lemma}
\begin{proof}
    \textit{1. Quartic bound.}
    We focus on the cost of an iteration of the for-loop over $t$, 
    the base case $t=0$ is trivially $\mathcal{O}(|\Sigma| + n)$.
    The size of the queue $Q_\tp$ is exactly $|S_t||\Sigma|\le n|\Sigma|$.
    For every word $\vw a$ in $Q_\tp$, we compute the corresponding vector of walk incidence matrix $\mW_t[-, \vw a]\In \Real^{V(G)}$,
    from the vector $\mW_t[-, \vw]$ for the word $\vw$ in $S_t$ as shown in the proof of \cref{lem:equivalencewalkrefinementinvariant},
    by multiplying by matrix $\mP_a$ if $t=1$, and by $\mP_a \mA$ if $t\ge 2$.
    The expensive part of the algorithm is the computation of the rank of $\mathcal{B}_\tp\cup \{\gamma\}$,
    this can be done in $\mathcal{O}(n^2)$ time,
    if we maintain the representation of the linearly independent vectors of $\mathcal{B}_t$ in a row echelon form.
    Thus, for a limit $T\In \Nat$, we have $T$ iterations of the for-loop over $t$, 
    which has the complexity of $\mathcal{O}(|S_t||\Sigma| + |S_t|n^2 + |S_\tp||\Sigma|) = \mathcal{O}(n^3|\Sigma|)$.
    And finally, the total complexity of the algorithm is $\mathcal{O}(Tn^3|\Sigma|)$.
    This yields $\mathcal{O}(Tn^4)$ using that $|\Sigma| \le n$.

    \textit{2. Cubic bound.}
    To improve upon the above quartic bound, we further exploit the structure of our transition matrices 
    $\mP_a\mA$ and $\mP_a \mI$. 
    In each iteration over $t$
    Partition matrices on the left side of $\mA$ effectively zero out rows of vertices that are not of color $a$.
    Therefore, instead of performing for a given base $\mB_t$ multiplication $\mP_a\mA \mB_t$ for all $a\In \Sigma$,
    we only multiply by $\mA \mB_t$ and then choose the row-submatrix $\mA_a$ of shape $V_a \times V$ where $V_a = \{u \mid u\in V,\, \chi(u) = a\}.$
    To that end, for each $\mA_a$ we select the first subset of independent columns for $S_t.$
    This is without loss of generality since the row partition ($V_{a_1} \sqcup \cdots \sqcup V_{a_|\Sigma|} = V$) gives orthogonal decomposition of the column space.
    Using the row echelon form technique as above, this accounts for $\mathcal{O}(n|V_a|^2)$ steps.

    In total, we sum $\sum_a n |V_a|^2$ over all colors in $\Sigma$, and since $n \sum_a |V_a|^2 \le n \cdot n^2,$ we get an algorithm of complexity $\mathcal{O}(Tn^3).$

    \textit{3. QR-bound. } 
    To obtain this bound, we focus on the linear dependence queries that cost $\mathcal{O}(|V_a|^2)$ each. 
    Instead of running a greedy loop over columns for each $\mA_a$, we use matrix decomposition.

    Recall a general \textbf{QR-decomposition} of square and/or tall matrix $\mX$ of shape $k\times l$ such that $k \ge l$.
    It computes a square matrix $\mQ$ of shape $k\times k$ and an upper-triangular matrix $\mR$ of shape $k\times l$ such that $\mX = \mQ\mR.$
    Diagonal $\delta$ of matrix $\mR$ is a vector of length $l$, where 
    $\delta[j]= \mR[j,j]$ is zero if and only if the $j$-th column of $\mX$ is linearly dependent on the previous columns.
    
    It remains to clarify usage of QR-decomposition in our case.
    Matrices $\mA_a$  of shape $V_a\times V$ such that $|V_a| \le n,$ are rather wide than tall.
    Therefore, we instead apply the following iterative approach:
    
    Initially, duplicate rows of $\mA_a$ to obtain matrix $\mA_a^\tup{0}$ of shape $V_a^\tup{0} \times V,$ where $V_a^\tup{0} = V_a \uplus V_a.$
    Next, split columns of $\mA_a^\tup{0}$ into $m_0 \coloneqq \lceil n/|V_a^\tup{0}| \rceil$ square blocks $\mA^\tup{0}_{a, 1}, \dots, \mA^\tup{0}_{a, m_0}$, padding the last one with zero columns if needed. 
    
    In the $s$-th step, compute QR-decomposition for each square matrix $\mA^\tup{s}_{a, i}$ independently, selecting its linearly independent columns.
    Crucially, due to the initial row duplication, $\rank(\mA^\tup{s}_{a, i}) \le \frac{1}{2} |V_a^\tup{s}|,$ hence obtaining at most $\frac{1}{2} |V_a^\tup{s}|$ of independent columns.
    From consecutive blocks $\mA^\tup{s}_{a,2i'}$ and  $\mA^\tup{s}_{a, 2i' + 1}$, 
    construct one square block $\mA^\tup{s+1}_{a, i'}$ containing at most 
    $2 \cdot \frac{1}{2} |V_a^\tup{s}|$ selected independent columns (padding rest with zeros if needed) of $\mA^\tup{s}_{a,2i'}$ and  $\mA^\tup{s}_{a, 2i' + 1}$.
    This way, we obtain $m_{s+1} \coloneqq \lceil \frac{m_s}{2} \rceil$ 
    blocks with rows remaining duplicated for the $(s+1)$-step.
    
    Since the number of blocks $m_s$ halves after every step,
    after a finite number of steps $S$,  we obtain a single square block.
    Its final QR-decomposition gives the lexicographically first column base marking the desired columns of $\mA$ to obtain $B_{t+1}$.
   During this process we perform at most $\sum_{s=0}^{S} m_s = \sum_{s=0}^{S} \lceil m_0 \cdot 2^{-s} \rceil \le 3 m_0$ decompositions.
   Each decomposition has cost $\mathcal{O}((2|V_a|)^{\omega + \epsilon})$, as shown by \citet{Dem+2007}, for any $\epsilon > 0$, where $\omega < 3$ is the exponent of matrix multiplication.
    
    Finally, accounting for all row partitions and considering for a suitable constant $C$, we have 
   \begin{align*}
    \sum_{a\in \Sigma} 3 \lceil n/|V_a^\tup{0}| \rceil (2|V_a^\tup 0|)^{\omega + \epsilon}&\le 
    n \sum_{a\in \Sigma} C |V_a|^{\omega + \epsilon - 1} \le
    n \cdot C n^{\omega + \epsilon - 1} \le
    C n^{\omega + \epsilon}.
   \end{align*}
   Therefore, to obtain $B_{t+1}$ from $B_t,$ we perform $\mathcal{O}(n^{\omega + \epsilon})$ operations and $\mathcal{O}(Tn^{\omega + \epsilon})$ in total.

\end{proof}

    Note that the last bound in the above proof of \cref{lem:layeredcwscomplexity}
    is both of theoretical and practical interest.
    The QR-decomposition is a standard building block of many numerical libraries (e.g., LAPACK and SciPy), providing highly optimized implementations available.
    Moreover, splitting wide submatrix into a batch of square blocks allows for highly parallel processing of the algorithm on modern hardware architectures.

\restate{thm:algorithm}{
Let $\chi$ be a $\Sigma$-coloring on a graph $G$ with $n$ vertices, and let $T\In \Nat$ be a limit. 
Then the canonical subsets $(S_t)_{t=0}^{T}$, (satisfying (i), (ii), and (iii)), are computable in time $\mathcal{O}(Tn^{\omega + \epsilon})$ for any $\epsilon > 0$ where $\omega$ is the exponent of matrix multiplication.
}\begin{proof}
    We use the Algorithm~\ref{alg:layeredcws}, which is correct by \cref{lem:layeredcws} and satisfying the time complexity by \cref{lem:layeredcwscomplexity}.
\end{proof}

In the following part,
we prove two statements from the end of \autoref{sec:efficientaggregation}.
For $t\ge 0$, we recall that the columns of matrix $\mB_t=\mB_t(G,\chi)$ span the same space as the columns of the walk incidence matrix $\mW_t(G, \chi)$.
The columns of $\mB_t$ are indexed by the words in $S_t(G, \chi)$.
Let us denote by $\vz_{t} = 00\cdots 0$ a constant word of length $t$.

\restate{obs:specialwalkpartition}{
    Let $\chitriv$ be the $\{0\}$-coloring on a graph $G$ with at least one edge.
    Then for every $t\ge 0$: (a) it holds that $|S_t| = 1$;
    (b) the only entries, $\mC_t^\mI[\vz_{t}, \vz_{{t+1}}] =\frac{\vone^\top \mA^{2t+1} \vone}{\vone^\top \mA^{2t} \vone}$, and 
    $\mC_t^\mA[\vz_{t}, \vz_{{t+1}}] =\frac{\vone^\top \mA^{2t+2} \vone}{\vone^\top \mA^{2t} \vone}$.
}
\begin{proof}
    The assumption that $G$ has at least one (undirected) edge ensures that the denominators are well defined.
    Indeed, $G$ then contains at least one walk of every length $t$, that is, $\vone^\top \mA^{2t}\vone \ge 1$.

    The first part (a) follows from the fact that the only word in $S_t(G, \chitriv)$ is $\vz_{t}$,
    and thus every queue $Q_\tp$ of the Algorithm~\ref{alg:layeredcws} contains exactly the one word $\vz_{t+1}$.
    Note that the only partition matrix $\mP_a = I$.
    
    For the second part (b), we follow the definition of the matrices $\mC^\mI_t$ and $\mC^\mA_t  \In \Real^{\{\vz_t\}\times \{\vz_\tp\}}$:
    \begin{align*}
        \mC^\mM_t[\vz_t, \vz_\tp] &= 
        \One \mC^\mM_t\OneT =
        \mB_t^\pin \mM \mB_\tp =\\
        &=((\mP_a\mA\mP_a)^t \One)^\pin \mM ((\mP_a\mA\mP_a)^{t+1} \One) =\\
        &= (\mA^t \One)^\pin \mM (\mA^{t+1} \One) =\\
        &= ((\mA^{t}\One)^\top \mA^t \One)^{-1} \cdot ( \mA^t \One)^\top \mM \mA^t \One =\\
        &= (\One^\top \mA^{2t} \One)^{-1} \cdot \One^\top \mA^t \mM \mA^\tp \One.
    \end{align*}
    By setting $\mM\coloneqq \mI$ we obtain the result for the $\mC^\mI_t$,
    and by $\mM\coloneqq \mA$ for the $\mC^\mA_t$.
\end{proof}

\restate{obs:specialmessagepassing}{
    Let $\chiid$ be the $V$-coloring on a graph $G=(V, E)$ with $n$ vertices.
    By $\vv_{t, u}$, we denote the (unique) word in $S_{t}$ with the last color $u \In V$.
    Then for every $t\ge 1$: (a) it holds that $|S_0|=1$, and $|S_t| = |V|$;
    (b) for entries $\mC_0^\mI[\lambda, u] = \frac{1}{n}$, and 
    $\mC_t^\mI[\vv_{t,u}, \vv_{t+1,v}] = \mI[u, v]$
    and $\mC_t^\mA[\vv_{t,u}, \vv_{t+1,v}] = \mA[u, v]$.
}
\begin{proof}
    We denote the vertices $V$ by $\{u_1, u_2, \dots, u_n\}$, which here coincides with the set of colors.
    To prove (a), we proceed by induction on $t$.
    For the base case $t=0$, we have trivially $|S_0|=1$, as $S_0 = \{\emptyword\}$.
    For the case $t=1$, we have $S_1 = V$, since $\mP_{u_i} = e_{u_i}e_{u_i}^\top$ for $u_i\In V$,
    and furthermore, $u_i$-th column of $\mB_1$ is $\mP_{u_i} \One = e_{u_i}$, and thus we have $\mB_1 = \mI$.
    For the induction step, we assume that $S_t = \{ \vv_{t, u} \mid  u \In V\}$.
    Since $\mW_\tp[-, \vw u] = \mP_u \mA \mW_t[-, \vw]$, 
    there is for each $u\In V$ a unique word $\vw u$ in $S_\tp$,
    and the base $\mathcal{B}_\tp$ is also canonical, that is, $\mB_\tp = \mI$.
    Thus, we have $|S_\tp| = |V|$.

    For the second part (b), we have $\mC_0^\mI = \One^\pin \mI \mB_1 = (\OneT \One)^{-1} \OneT \mI \mI = \frac{1}{n}\One^\top$.
    Next, for $t\ge 1$ and $\mM\In \{\mA, \mI\}$, we get $\mC^\mM_t = \mB_t^\pin \mM \mB_\tp = \mI^\pin \mM \mI = \mM$.
\end{proof}

\subsection{Constraints on Walk-Incidence Matrices}\label{appendix:constraints}

In the main text, we have noted that the matrices $\mW_t(G, \chi)$ are not completely arbitrary.
We now show in detail, how the structure of the matrices $\mW_t(G, \chi)$ is influenced by the weighted finite automata $\Ac(G, \chi)$.
Suppose we have two graphs $G$ and $H$, and that 
the automata $\Ac(G, \chi)$ and $\Ac(H, \chi)$ are equivalent, $\sem{\Ac(G, \chi)} = \sem{\Ac(H, \chi)}$.
Then there is a minimal weighted automaton $\Ac_1$ common for $\Ac(G, \chi)$ and $\Ac(H, \chi)$.
By \autoref{eq:fmatequivalence} and \autoref{eq:qmatequivalence}, there are matrices $\Fg\In\Real^{V(G)\times Q_1}$ 
and $\Fh\In\Real^{V(H)\times Q_1}$ mapping automata induced by $G$ and $H$ to the common minimal automata $\Ac_1$.
Here, we denote the matrix between the two automata $\Ac(G, \chi)$ and $\Ac(H, \chi)$ by the following:
\begin{align*}
    \mF = (\Fg) (\Fh)^\top \In \Real^{V(G)\times V(H)}.
\end{align*}
Note that
$\mF\One = \Fg (\Fh)^\top\One = \Fg \One = \One$, and similarly $\OneT \mF = \OneT$.

\begin{lemma}\label[lemma]{lem:fmatwalk}
    Using the notation above, we have for every $t\ge 0$:
    \begin{align*}
        \mW_t(G, \chi) = \mF\mW_t(H, \chi).
    \end{align*}
\end{lemma}
\begin{proof}
    As the $\Sigma$-coloring $\chi$ is fixed, we write $\mW_t^G = \mW_t(G, \chi)$ and $\mW_t^H = \mW_t(H, \chi)$.
    In addition, we denote the matrices of the first graph $\mA^G$, $\mP_a^G$, and similarly for the second graph $\mA^H$, $\mP_a^H$,
    given that $a\In\Sigma$.
    We proceed by induction on $t$, proving the following statement for each word $\vw\In\Sigma^t$:
    $\mW_t^G [-, \vw] = \mF \mW_t^H[-, \vw]$.

For $t=0$, from \cref{lem:equivalencewalkrefinementinvariant} it follows that 
\begin{align*}
    \mW_0^G[-, \lambda] = \One = \mF \One  = \mF \mW_0^H[-, \lambda].
\end{align*}
Moreover, using \autoref{eq:fmatequivalence}, for $t=1$, we get for any $a \In \Sigma$
\begin{align*}
    \mW_1^G[-, a] = \mP_a^G \One = \mP_a^G \mF \One = \mF \mP_a^H \One = \mF\mW_1^H[-, a].
\end{align*}
For the induction step, we assume that the lemma holds for all words of length $t$.
We take a word $\vw a$ of length $t+1$, which denotes the last color of $\vw$ by $b$.
It follows:
\begin{align*}
    \mW_\tp^G[-, \vw a] &= \mP_a^G \mA^G \mP_b^G \mW_t^G[-, \vw] =\\
    &= \mP_a^G \mA^G \mP_b^G F \mW_t^H[-, \vw] =\\
    &= \mF \mP_a^H \mA^H \mP_b^H \mW_t^H[-, \vw] =\\
    &= \mF\mW_\tp^H[-, \vw a],
\end{align*}
which finishes the proof.
\end{proof}

\paragraph{Characterization of Reduced Matrices.}
We recall that
for graph $G$ and a coloring $\chi$,
the \emph{graph invariant} of reduced matrices $\invEA(G, \chi)$ is defined as the set of pairs of matrices
\begin{align*}
    \invEA(G, \chi) = \left\{\left(\mC_t^\mA\!,\, \mC_t^\mI\right) \mid 0 \le t < n\right\}\!,
\end{align*}
to state the following lemma:

\begin{lemma}\label[lemma]{lem:giLeA}
    For any coloring $\chi$, we have
     $\invEA(-, \chi) \Ele \sem{\Ac(-, \chi)}$.
\end{lemma}
\begin{proof}
    Suppose we have two graphs $G$ and $H$ that are not distinguished by the semantics of their induced automata.
    Then there is a matrix $\mF\In\Real^{V(G)\times V(H)}$ adjoining these two automata.
    Given a $\Sigma$-coloring $\chi$, and two graphs $G$, $H$ we have from the \cref{lem:fmatwalk}
    that $\mB_t^G = \mF\mB_t^H$ for all $t \ge 0$.
    For $\mM\In\{\mA, \mI\}$, we have
    \begin{align*}
        \mC_t^\mM(G, \chi) &= (\mB_t^G)^\pin \mM^G \mB_\tp^G 
          = (\mB_t^G)^\pin \mM^G \mF\mB_\tp^H \\
        &= (\mB_t^G)^\pin \mF \mM^H \mB_\tp^H 
          = (\mB_t^G)^\pin \mF  \left(\mB_t^H(\mB_t^H)^\pin\right)  \mM^H \mB_\tp^H \\
        &= (\mB_t^G)^\pin \mF  \mB_t^H(\mB_t^H)^\pin  \mM^H \mB_\tp^H 
          = (\mB_t^G)^\pin \left(\mF  \mB_t^H\right)(\mB_t^H)^\pin  \mM^H \mB_\tp^H \\
        &= (\mB_t^G)^\pin \left(\mB_t^G\right)(\mB_t^H)^\pin  \mM^H \mB_\tp^H
            = (\mB_t^H)^\pin  \mM^H \mB_\tp^H. \\
    \end{align*}
    The final expression is equal to $\mC_t^\mM(H, \chi)$, which finishes the proof.
\end{proof}

In the previous lemma, we have shown that the reduced matrices are invariant under the equivalence of the automata.
To state the other direction, we first give a simpler lemma for $\chi=\chitriv$ to illustrate the structure of the more general lemma.
\begin{lemma}\label[lemma]{lem:giLeAtriv}
    Let $\chitriv$ be the trivial coloring, then
     $\invEA(-, \chitriv) \Ege \sem{\Ac(-, \chitriv)}$.
\end{lemma}
\begin{proof}
    As shown in Observation~\ref{obs:specialwalkpartition},
    the matrices $\mC_t^\mI(G, \chitriv)$ and $\mC_t^\mA(G, \chitriv)$ are of a single entry
    $(\OneT \mA^{2t}\One)^{-1} \OneT \mA^{2t+1} \One$ and $(\OneT \mA^{2t}\One)^{-1} \OneT \mA^{2t+2} \One$,
    respectively.
    We let the symbol $a=0$, so that $\chitriv$ is the $\{a\}$-coloring on $G$.
    Note that $|V(G)| = n = |\invEA|$.
    
    Next, the transition matrices of $\Ac(G, \chitriv)$ are 
    $\mM(a) = \mP_a = I$, and $\mM(a\vmin a) = \mP_a \mA \mP_a = \mA$,
    and generally for $k$-th repetition of $a$, $\mM(a\vmin \cdots \vmin a) = \mA^k$.
    Thus, the semantics of the automata $\Ac(G, \chitriv)$ is determined by
    the formal series given as $a^k \mapsto \OneT \mA^k \One$.

    For simplicity, we define the following three functions for all $m\In \Nat$:
    \begin{align*}
        f(m) &= \OneT A^m \One,\quad
        g_1(m) = \frac{\OneT \mA^{2m+1} \One}{\OneT \mA^{2m} \One}, \quad
        g_2(m) = \frac{\OneT \mA^{2m+2} \One}{\OneT \mA^{2m} \One}.
    \end{align*}
    In the language of such notation, it is sufficient to show that all values of $f$ are determined by values of function $g_1, g_2$ and single value $n$. 
    Moreover, it is sufficient to show values of $f(k)$ for $k< n$, 
    since for values $k\ge n$, we can apply Cayley-Hamilton theorem and express $A^k$
    using the powers of $A$ up to $n-1$.
    
    We shall proceed by the following induction on $k$.
    For the base case $k=0$, we have $f(0) = n$.
    For the induction step, assume that the value $f(k')$ is determined for each $k' < k$,
    we distinguished two cases: (a) $k = 2l+1$ is odd, and (b) $k = 2l+2$ is even.
    
    For the odd case (a), we have $f(k) = f(2l+1) = f(2l)g_1(l)$,
    and for the even case (b), we have $f(k) = f(2l+2) = f(2l)g_2(l)$.
    Since $f(2l)$ is known from the induction hypothesis, 
    we can compute $f(k)$ from $f(2l)$ and $g_1(l)$ or $g_2(l)$.
\end{proof}

Here, before we state a more general variant of \cref{lem:giLeAtriv} for any coloring $\chi$,
we introduce the following notation.
Let $\chi$ be a $\Sigg$-coloring on a graph $G$, and $\invEA(G, \chi)$ the invariant of reduced matrices $\mC^\mA_t(G, \chi)$, $\mC^\mI_t(G, \chi)\In \RealSSp$ 
for $0\le t< n$.
Let $a, b\In \Sigma$, and $t \In\Nat$ such that $0 \le t < n$,
we define $\mM_t(a)\In \RealSSp$ by setting
\begin{align}
    \mM_t(a)[\vw_1 a, \vw_2 a] \coloneqq \mC^\mI_t(G, \chi)[\vw_1 a, \vw_2 a],\label{eq:matMta}
\end{align}
for all $\vw_1 a\In S_t$, and $\vw_2 a\In S_\tp$ and letting all other entries be zero.
Similarly, we define $\mM_t(a\vmin b)\In \RealSSp$ by setting
\begin{align}
    \mM_t(a\vmin b)[\vw_1 a, \vw_2 ab] \coloneqq \mC^\mA_t(G, \chi)[\vw_1 a, \vw_2 ab],\label{eq:matMtab}
\end{align}
for all  $\vw_1 a, \vw_2 ab\In S_\tp$ and letting all other entries be zero.

\begin{lemma}\label[lemma]{lem:matMtaMtab}
    Let $G$ be a graph and $\chi$ a $\Sigma$-coloring on $G$.
    Then for all $a, b\In\Sigma$ and each $t \In\Nat$ it holds that 
    \begin{align*}
        \mM_t(a) = \mB_t^+ \mP_a \mB_\tp, \quad\text{ and }\quad \mM_t(a\vmin b) = \mB_t^+ \mP_b \mA \mP_a \mB_\tp,
    \end{align*}
    where $\mB_t = \mB_t(G, \chi)$.
\end{lemma}
\begin{proof}
We start with the proof of the first identity. Let us fix $t\In\Nat$, then 
\begin{align}
    \mC^I_t(G, \chi) &= \mB_t^\pin ({\textstyle \sum_a \mP_a }) \mB_\tp
    = \sum_{a\in\Sigma} \mB_t^\pin \mP_a \mB_\tp = \sum_{a\in\Sigma} \mM_t(a).\label{eq:sumMta}
\end{align}
On the other hand,
since $\mB_t[-, \vw_1 a]  = (\mP_a \mB_t)[-, \vw_1 a]$ for each $\vw_1 a\In S_t$, 
we get for each  $\vw_1 x\In S_t$, and each $\vw_2 y \In S_\tp$, that it holds
\begin{align*}
    ((\mB_t^\top \mB_t )^{-1}\mB_t^\top \mP_a \mB_\tp)[\vw_1 x, \vw_2 y] =
    ((\mB_t^\top \mB_t )^{-1}\mB_t^\top \mP_x^\top \mP_a  \mP_y \mB_t)[\vw_1 x, \vw_2 y],
\end{align*}
By \cref{prop:partition} it follows that the expression is zero if $x\neq a$ or $y\neq a$.
In conjunction with \cref{eq:sumMta}, it follows that left-hand side and right-hand side coincide entry-wise and thus
 $\mM_t(a) = \mB_t^\pin \mP_a \mB_\tp$.
The latter identity, we show similarly by fixing $t\In\Nat$,
\begin{align}
    \mC^A_t(G, \chi) &= \mB_t^\pin ({\textstyle \sum_{a,b} \mP_b \mA \mP_a }) \mB_\tp\notag\\
    &= \sum_{a,b\in\Sigma} \mB_t^\pin \mP_x \mP_b \mA \mP_a \mP_y \mB_\tp
    = \sum_{a,b\in\Sigma} \mM_t(a\vmin b).\label{eq:sumMtab}
\end{align}
Similarly, we have for each $\vw_1 x\In S_t$, and each $\vw_2 y \In S_\tp$,
\begin{align*}
    (\mB_t^\pin \mP_a \mA \mP_b \mB_\tp)[\vw_1 x, \vw_2 y] =
    (\mB_t^\pin \mP_x^\top \mP_a \mA \mP_b  \mP_y \mB_t)[\vw_1 x, \vw_2 y],
\end{align*}
which is zero if $x\neq a$ or $y\neq b$.
In conjunction with \cref{eq:sumMtab}, it follows that left-hand side and right-hand side coincide entry-wise and thus
 $\mM_t(a\vmin b) = \mB_t^\pin \mP_b \mA \mP_a \mB_\tp$.
\end{proof}

\begin{lemma}\label[lemma]{lem:giGeA}
    For any coloring $\chi$, we have
     $\invEA(-, \chi) \Ege \sem{\Ac(-, \chi)}$.
\end{lemma}
\begin{proof}
    We first show that reduced matrices $\invEA = \invEA(G, \chi)$,
    for some graph $G$,
    encode a specific layered computation of minimal automata with the semantics of $\Ac(G, \chi)$.
    A direct consequence of  \cref{lem:matMtaMtab} is that 
    by selecting the entries of reduced matrices in $\invEA$ as in \cref{eq:matMta} and \cref{eq:matMtab}, we can obtain $\mM_t(a)$ and $\mM_t(a\vmin b)$ for all $a, b \In \Sigma$ and where $0 \le t < n$.

    From the size of $\invEA$, we get $n = |\invEA| = |V(G)|$.
    For a word $\vw = w_1 w_2 \cdots w_\tp \In\Sigma^\tp$, we define the expression
    $\gamma_\vw(\invEA)\In \Real^{\{\lambda\}\times S_\tp}$ as follows:
    \begin{align*}
        \gamma_\vw(\invEA) =  n \cdot \mM_0(w_1 \vmin w_2) \mM_1(w_2 \vmin w_3)\cdots \mM_{t}(w_{t} \vmin w_\tp).
    \end{align*}
    Next, we take $\mB_t = \mB_t(G, \chi)$
    and $\mM\colon \Sigma \to \RealVV$ the transition matrices of $\Ac(G, \chi)$,
    it follows from \cref{lem:matMtaMtab} that
    \begin{align*}
        \gamma_\vw(\invEA) &= 
        n \mB_0^\pin \mP_{w_1} \mA \mP_{w_2} \mB_1 
        \mB_1^\pin \mP_{w_2} \mA \mP_{w_3} \mB_2 
        \mB_2^\pin \cdots \mP_{w_{t}} \mA \mP_{w_\tp} 
        \mB_{t}\\
        &= n \One^\pin \mP_{w_1} \mA \mP_{w_2} \cdots \mP_{w_{t}} \mA \mP_{w_\tp}
        \mB_{t} \\
        &= \One^T \mM(\ovecw) \mB_t,
    \end{align*}
    where $\ovecw = w_1 \vmin w_2 \vmin \cdots \vmin w_\tp \iSigg^\tp$.
    
    We prove for every $t$, such that $0 \le t < n$, that weight of any word $\ovecw\iSigg^t$,
    $\sem{\Ac(G, \chi)}(\ovecw) = \One^T \mM(\ovecw)\One$
    can be expressed as a linear combination of $\gamma$-vectors 
    that depend only on $\invEA(G, \chi)$.
    We shall use the following notation, for the word $\vw = w_1 w_2 \cdots w_m\In \Sigma^m$,
    we denote the word $w_1 \vmin w_2 \vmin \cdots \vmin w_m\iSigg^m$ by $\ovecw$,
    and the reverse words 
    $w_m w_{m_1} \cdots w_1\In\Sigma^m$ by $\vw^R$, and finally,
    $w_m \vmin w_{m-1} \vmin \cdots \vmin w_1\iSigg^m$ by $\ovecw^R$.

    For the base case $t=0$, we have $\ovecw = \lambda$,
    and thus $\One^T \mM(\lambda)\One = \One^T\One = n$,
    which is equal to $\gamma_\lambda(\invEA)\One = n\One^+\One = n$.

    In the induction step $t\ge 1$, 
    we distinguish two cases: 
     (a) $t = 2l$ is even, and 
    (b) $t = 2l+1$ is odd.
     For the even case (a), we write the given word as  $\ovecw = \ovecw_p  \vmin a \vmin \ovecw_s$,
    for the suitable $\vw_p$ and $a\vw_s\In\Sigma^{l}$.
    Since the columns of $\mW_l = \mW_l(G, \chi)$ are spanned by the columns of $\mB_{l}$,
    there is a vector $\vx\In \Real^{S_l}$ such that
    by \cref{lem:equivalencewalkrefinementinvariant} it holds that 
    \begin{align*}
        \mM(a\vmin \ovecw_s)\One = \mW_l[-, \vw_s^R a] = \mB_l \vx.
    \end{align*}
    Note that the vector $\vx$ is independent of the choice of the base of
    the transition matrices $\mM$, as 
    for any base-changing matrix $\mF\In \Real^{V(H)\times V(G)}$, 
    for some graph $H$ with $\invEA(H, \chi) = \invEA$,
    it holds $\mW_l(H, \chi) = \mF\mW_l = \mF\mB_l \vx = \mB_l(H, \chi)\vx$.
    Next, we obtain 
    \begin{align*}
        \One^T \mM(\ovecw)\One &= 
        \One^T \mM(\ovecw_p  \vmin a \cdot a \vmin \ovecw_s)\One\\
        &= \One^T \mM(\ovecw_p \vmin a)  \mM(a \vmin \ovecw_s)\One\\
        &=\One^T \mM(\ovecw_p \vmin a)  \mB_l \vx
        = \gamma_{{\ovecw_p \vmin a}}(\invEA)\vx.
    \end{align*}
    Similarly, for the odd case (b), we write $\ovecw = \ovecw_p  \vmin a\vmin  \ovecw_s$,
    for the suitable $\vw_p \In\Sigma^{l}$ and $a\vw_s\In\Sigma^{l+1}$.
    Analogously, we can find a vector $\vx\In \Real^{S_{l+1}}$ such that
    $\mM(a\vmin \ovecw_s)\One = \mB_{l+1} \vx$ to obtain
    \begin{align*}
        \One^T \mM(\ovecw)\One &=
        \One^T \mM(\ovecw_p  \vmin a \cdot a \cdot a \vmin \ovecw_s)\One\\
        &=\One^T \mM(\ovecw_p \vmin a) \mM(a) \mM(a \vmin \ovecw_s)\One\\
        &=\One^T \mM(\ovecw_p \vmin a)\mM(a)  \mB_{l+1} \vx\\
        &=\gamma_{{\ovecw_p \vmin a}}(\invEA)\mB_{l} \mB_{l}^\pin \mP_a \mB_{l+1}\vx
        = \gamma_{{\ovecw_p \vmin a}}(\invEA)\mM_l(a)\vx.
    \end{align*}
    Thus, $\invEA = \invEA(G, \chi)$ determines the semantics of automata $\Ac(G,\chi)$.
    Moreover, $\invEA(H,\chi)=\invEA(G,\chi)$ implies
    $S_t(H,\chi)=S_t(G,\chi)$ for all $0\le t<n$, and hence
    the semantics of $\Ac(G,\chi)$ agrees with that of $\Ac(H,\chi)$.
\end{proof}

\restate{thm:walkInvariantEquivalence}{
    For every coloring $\chi$, it holds that
    $\invEA(-, \chi) \Eeq \Walkk(-, \chi).$
}

\begin{proof}
    The proof follows from \cref{lem:giGeA} and \cref{lem:giLeA}.
\end{proof}


\section{Homomorphism Counts and Colored Walks}\label{appendix:homcounts}

\newcommand{\existWalks}{\ensuremath{W_{\!\exists}} }
\newcommand{\countWalks}{\ensuremath{W_{\!\#}} }
\newcommand{\Aspt}{\ensuremath{\mathcal{O}} }

This section contains the proofs for the statements of \cref{sec:homtowalk}.
We first show that homomorphism counts from caterpillars $\Cat_1$ 
into a given graph can be characterized by walk incidences with respect to the coloring $\chideg$.
Our initial argument employs elementary techniques that are later generalized using methods from quantum graphs~\citep{Lov+2009,Dvo+2010} and logic~ (cf. \citep{Imm+1990,Cai+1992,Gro+2017}).

\begin{figure}[h]
\centering
\vspace{1em}
\scalebox{0.8}{
\begin{tikzpicture}[every node/.style={}, scale=0.9, font=\small]
    \usetikzlibrary{positioning,calc,decorations,decorations.pathreplacing}
    \node[rectangle, inner xsep=1.7em, inner ysep=1.1em, fill=black!5] (A) at (0, 0) {$\hom(\CaT{1}{t},-)$};
    \node[rectangle, inner xsep=1.7em, inner ysep=1.1em, fill=black!5] (B) at (6, 0) {{}$\Walk{t}(-,  \chideg)$};
    \node[rectangle, inner xsep=1.7em, inner ysep=1.1em, fill=black!5] (C) at (10, 0) {$\hom(\CaT{h}{t},-)$};
    \node[rectangle, inner xsep=1.7em, inner ysep=1.1em, fill=black!5] (D) at (16, 0) {{}$\Walk{t}(-,  \chitree{h})$};
    \draw[draw=black!30, very thick] ([yshift=5pt]A.east) to node[fill=none, draw=none, midway, above](lAB) {\Large$\Eeq$} ([yshift=5pt]B.west);
    \draw[draw=black!30, very thick] ([yshift=-5pt]A.east) to node[fill=none, draw=none, midway, below](lAB2) {\Large$\Ele$} ([yshift=-5pt]B.west);
    \draw[draw=black!30, very thick] ([yshift=5pt]C.east) to node[fill=none, draw=none, midway, above](lCD) {\Large$\Eeq$} ([yshift=5pt]D.west);
    \draw[draw=black!30, very thick] ([yshift=-5pt]C.east) to node[fill=none, draw=none, midway, below](lCD2) {\Large$\Ele$} ([yshift=-5pt]D.west);
    \node[above of=lAB, yshift=-1em] {\cref{thm:hom1t}};
    \node[below of=lAB2, yshift=1em] {\cref{prop:homcountt}};
    \node[above of=lCD, yshift=-1em] {\cref{thm:homht}};
    \node[below of=lCD2, yshift=1em] {\cref{prop:homcounttt}};
\end{tikzpicture}
}
\caption{
        \autoref{sec:expressivity} establishes the equivalence between homomorphism counts over caterpillars and colored walk refinement with respect to (specific) colorings.
        First, we focus on folklore caterpillars (left), and then we generalize to $(h, t)$-caterpillars (right) in \cref{sec:homcounts_general}.
        For every $h, t \ge 0$, we relate the following two functions on graphs: homomorphism counts over  $(h, t)$-caterpillars $\hom(\CaT{h}{t}, -)$, and colored walk refinement $\Walk{t}(-, \chitree{h})$.
}
\end{figure}

We begin by proving two technical lemmas.
\begin{lemma}\label[lemma]{lem:tuplefaithful}
    For every integer $t\ge 0$ and natural numbers
    $D_0, D_1, D_2, \dots, D_t$
    there are integral exponents $s_1, s_2$ $\dots$, $s_t$
    such that $D_0 \le 2^{s_1}$,
    and such that
    all $t$-tuples of natural numbers $(d_1, d_2, \dots, d_t)$
    that satisfy $1 \le d_i < D_i$ for $i \In [t]$,
    are injectively represented by the product
    $p = d_1^{s_1} \cdot d_2^{s_2} \cdots \cdot d_t^{s_t}$.
\end{lemma}
\begin{proof}
    We can instead take logarithm of $p$, since $\log$ is an injective function,
    \begin{align*}\log_2 p = s_1 \log_2 d_1 + \cdots + s_t \log_2 d_t.\end{align*}
    The values of $\log_2 d_i$ in $\{\log_2 1, \log_2 2, \dots, \log_2 (D_i - 1)\}$ are called \emph{digits} for $i\In[t]$.
    We construct a positional numeral system with a heterogeneous basis (i.e., $s_1, s_2, \dots s_t$) special for each digit's position $i$.
    Since the intervals of digits at the $i$-th position are bounded by $\log_2 D_i$,
     we require to preserve the following inequality
    \begin{align*}s_{i+1} > \sum_{j=1}^i \log_2 (D_j- 1) \cdot s_j.\end{align*}
    To next meet the first condition $D_0\le 2^{s_1}$, we set $s_0 \coloneqq \lceil \log_2 D_0 \rceil$.
    Furthermore, we obtain $s_{i+1} \coloneqq 1 + \sum_{j=1}^i \lceil \log_2 (D_j-1) \rceil \cdot s_j$.
    Finally, we define $R_\tp = \log_2 p$ to
     inductively evaluate the following expression
    \begin{align*}
        \log_2 d_i = \frac{1}{s_i}\bigg(R_i - \sum_{j=0}^{i-1} s_j \log_2 d_j\bigg)
        \quad\text{ where }\quad
        R_l = \frac{R_{l+1}}{s_{l+1}} - \left\lfloor \frac{R_{l+1}}{s_{l+1}}\right\rfloor,
    \end{align*}
    for every $l\In[t]$, and every $i\In[t]$.
    Therefore, the values $s_i$ are sufficient.
\end{proof}

\begin{lemma}
    Let $\vx$ be an $m$-tuple in $\Nat_{\ge 1}$ with mutually distinct elements,
    that is $\vx[i] \neq\vx[j]$ whenever $i\neq j$, 
    and let $m$-tuples  $\va$, $\vb$ in $\Nat^m$ be such that 
    $\va \neq \vb$, that is $\va[i] \neq \vb[i]$ for existing $i\In[m]$.
    Then there is $k$ in $\Nat$ such that
    \begin{align}
        \sum_{i=1}^{m} \va[i] (\vx[i])^k \neq \sum_{i=1}^{m} \vb[i] (\vx[i])^k.
        \label{eq:distinct}
    \end{align}
    \label[lemma]{lem:distinct}
\end{lemma}
\begin{proof}
    Denote tuple $\va$ by $(a_1, a_2, \dots, a_m)$, $\vb$ by $(b_1, b_2, \dots, b_m)$,
    and the tuple $\vx$ by $(x_1, x_2, \dots, x_m)$.
    For a contradiction, suppose that for all $k$ in $\Nat$ the equality in \autoref{eq:distinct} holds.
    Let us choose the index $j$ such that $|a_j-b_j| > 0$ and, importantly, such that $x_j$ is maximal (such $j$ is unique since elements of $\vx$ are distinct).
    Let us define functions $f$ and $g$ as follows
    \begin{align}
        f(k) = \left|(a_j-b_j)x_j^k\right|, \quad g(k) = \Big|\sum_{i\neq j} (b_i - a_i) x_i^k\Big|.
        \label{eq:distinctfunc}
    \end{align}
    As $k\to \infty$, the function $f$ grows faster than the function $g$.
    Therefore, there exists $k$ such that $f(k) > g(k)$, that is
    $(a_j-b_j)x_j^k > \sum_{i\neq j} (b_i - a_i) x_i^k$,
    which gives us the contradiction.
\end{proof}

In \cref{prop:homcountt}, we prove the formula for counting homomorphisms from the class~${\Cat_1}$ that uses the leg sequence (see \autoref{sec:expressivity} and \autoref{fig:caterpillars}) of the caterpillar graph. Next, a \emph{star graph} is a rooted tree of height at most 1.

Given a $(1,t)$-caterpillar $F\In \CaT 1 t$,
we associate its leg sequence $(S_1, S_2, \dots, S_t)$ of star graphs $S_i$, 
with the tuple $\vs_F = (|S_1|-1, |S_2|-1, \dots, |S_t|-1)\In \Nat^{t}$,
so that $|S_i|-1$ is the number of leaves of the $i$-th leg of $F$.
In terms of folklore caterpillars, the $i$-th entry of $\vs_F$
is exactly the number of one-edge legs attached to the $i$-th vertex of the spine path.
Importantly, every $(1,t)$-caterpillar graph $F$ is fully described and determined by ${\vs}_F$.
Furthermore, we recall that $\Walk{t}(G, \chideg)$ is the multiset of all colored walks of length $t$ in $G$.

\begin{proposition}\label[proposition]{prop:homcountt}
    Let $t \In \Nat$, then for every graph~$G$ it holds
    \begin{align}
        \hom(\CaT 1 t,G)[F] = 
        \sum_{\substack{
            \vw\In \Nat^t
        }} 
        \Walk{t}(\vw)
        \cdot 
        \prod_{i=1}^{t} \vw[i]^{\vs_F[i]},
        \label{eq:homcountt}
    \end{align}
    for each $F\In {\CaT 1 t}$ with $\vs_F\In \Nat^{t}$,
    where $\Walk{t}(\vw) = \Walk{t}(G, \chideg)(\vw)$ for $\vw \In \Nat^t$.
\end{proposition}
\begin{proof}
Given a colored walk $\vw\In \Nat^t$, consider its occurrence $\vu=(u_1, u_2, \dots, u_t)$ in~$G$.
We denote the vertices of the spine of $F$ that correspond to $\vs_F$ by $l_1, l_2, \dots, l_t$.
Let $k_\vu$ be the number of graph homomorphisms $\varphi\colon F\to G$,
mapping the spine of $F$ to the occurrence $\vu$, that is, such that $\varphi(l_i) = u_i$ for each $i\In[t]$.

Therefore, any two distinct homomorphisms contributing to $k_\vu$
differ precisely by their mapping of vertices outside the spine,
namely, the leaves of the star graphs $S_i$ for $i\In [t]$.

Specifically, each leaf of $S_i$ can be mapped to any vertex in its neighborhood $N(u_i)$, 
hence, we have exactly $\vw[i]=\deg_G(u_i)$ choices.
Because $S_i$ contains $\vs_F[i]$ such leaves, 
there are $\vw[i]^{\vs_F[i]}$ choices to map the legs of $\mP_i$ independently of the other $\mP_j$ ($j\neq i$).
Therefore, we have $k_\vu = \prod_{i=1}^{t} \vw[i]^{\vs_F[i]}$.

Finally, we sum over all independent $k_\vu$ as follows:
\begin{align*}
    \sum_{\substack{
        \vu\In V(G)^t\\
    }} k_\vu
    = 
    \sum_{\substack{
    \vu \text{ occurrence } \\
    \text{ of } \vw \text{ in } G
    }}
    \prod_{i=1}^{t} \vw[i]^{\vs_F[i]}
    = 
    \sum_{\substack{
        \vw\In \Nat^t
    }} 
    \Walk{t}(\vw)
    \cdot 
    \prod_{i=1}^{t} \vw[i]^{\vs_F[i]},
\end{align*}
where in the second step, we distinguish  occurrences $\vu$ by their color $\vw$,
next, these occurrences are quantified by $\Walk{t}(\vw)$.
Thus, we obtained an expression equal to the left-hand side of \autoref{eq:homcountt}.
\end{proof}


\begin{theorem}\label{thm:hom1t}
    For every $t\In\Nat$, it holds that $\Walk{t}(-, \chideg)\Eeq \hom({\CaT 1 t}, -) $.
\end{theorem}
\begin{proof}
We first show $\Walk{t}(-, \chideg)\Ege \hom({\CaT 1 t}, -)$.
Suppose that $\Walk{t}(G, \chideg) = \Walk{t}(H,\chideg)$ for two given graphs $G$ and $H$.
Then by  \autoref{eq:homcountt} of \cref{prop:homcountt} we have that $\hom(\CaT 1 t, G)[F] = \hom(\CaT 1 t, H)[F]$ for every entry indexed by $F\In \CaT 1 t$.

For the other direction, suppose that $\Walk{t}(G, \chideg) \neq \Walk{t}(H,\chideg)$ for two given $G$ and $H$.
For clarity, we use the shortcut $m_{\vw}^X = \Walk{t}(X,\chideg)(\vw)$ for graph $X\In\{G, H\}$.
Using this notation, there exists $\vw'\In \Nat^t$ such that 
$m_{\vw'}^G \neq m_{\vw'}^H$.

Let $n$ bound the number of vertices of $G$ and $H$, then the maximum degree of both graphs is at most $n-1$.
That means there are at most $(n-1)^{t}$ plain walks of length $t$ in each graph, implying both
\begin{align*}
    0 \le m_\vw^G \le (n-1)^{t},
    \quad \text{and}\quad
    0 \le m_\vw^H \le (n-1)^{t}.
\end{align*}

Here, we use \cref{lem:tuplefaithful}, applied on $t$-tuples $\vw\In [n-1]^t$.
We choose $D_0 = n^{t}$ and bound entries by $D_1 = D_2 = \dots = D_t = n$,
to get injective representations of tuples in $[n]^t \times [n] \times \cdots \times [n]$.
As a result, we obtain coefficients $(s_1, s_2, \dots, s_t)$
such that the function
\begin{align*}
    \left(
        m_\vw^X, \vw[1], \vw[2], \dots, \vw[t]
    \right)
    \,\mapsto\,
    m_\vw^X
    \cdot \vw[1]^{s_1}\cdot \vw[2]^{s_2} \cdots \vw[t]^{s_t},
\end{align*}
is injective provided that $m_\vw^X > 0$ on both $X \In \{G, H\}.$

In case either $m_\vw^G = 0$ or $m_\vw^H = 0$, choose $k=1$ as a multiplier of coefficients $(s_1, \dots, s_t)$ to satisfy \cref{eq:distinctcat} directly.
Otherwise, to find suitable $k$, we apply \cref{lem:distinct} by setting 
vectors $\vx, \va, \vb\In [n-1]^t$ as follows:
\begin{align*}
    \vx[i] &= \vw_i[1]^{s_1}\cdot \vw_i[2]^{s_2} \cdots \vw_i[t]^{s_t},  \quad\quad
    \va[i] = m^G_{\vw_i},\quad\quad
    \vb[i] = m^H_{\vw_i},
\end{align*}
where $i$ is the index enumerating each $\vw_i\In [n-1]^t$.

Since we assumed $m_{\vw'}^G \neq m_{\vw'}^H$, it holds $\va \neq \vb$
and therefore by \cref{lem:distinct} we can find a finite $k$ such that
\begin{align}\label{eq:distinctcat}
    \sum_{\substack{
        \vw \in [n-1]^t
    }} 
    m^G_\vw \cdot 
    \left(\prod_{i=1}^{t} \vw[i]^{s_i}\right)^k 
    \neq
    \sum_{\substack{
        \vw \in [n-1]^t
    }} m^H_\vw \cdot 
    \left(\prod_{i=1}^{t} \vw[i]^{s_i}\right)^k.
\end{align}
Finally, there is a caterpillar $F'$ in ${\CaT 1 t}$ determined by $\vs_{F'} = (s_1k, s_2k, \dots, s_tk)$,
and since both sides of \cref{eq:distinctcat} can be rewritten by applying exponent $k$ as $\hom(\Cat_{1,t},G)[F'] \neq \hom(\Cat_{1,t},H)[F']$ by \cref{prop:homcountt},
we obtain $\hom(\Cat_{1,t},G) \neq \hom(\Cat_{1,t}, H)$.
\end{proof}

\subsection{Quantum Graph Homomorphisms}\label{sec:homcounts_general}
In this subsection, we extend the proof of \autoref{thm:hom1t} to the case of generalized caterpillars $\CaT h t$.
Similarly to the previous case in \cref{prop:homcountt},
where we counted possible mappings of each leg of the folklore caterpillar,
we can adapt a more general approach by replacing star graphs with $1$-labeled graphs.
Finally, we replace the counts of all possible mappings of star graphs
by the homomorphism counts of linear combinations of $1$-labeled graphs, called quantum $1$-labeled graphs.

\paragraph{Labeled graphs.} We follow the algebraic approach to quantum graphs by \citet{Lov+2009}, instantiating for the $1$-labeled case. 
A~\emph{$1$-labeled graph}, or simply \emph{labeled graph},
$G^\bullet$ is a graph $G$ with one distinguished vertex
$u\In V(G)$, called a \emph{label}, denoted by $\lab(\Gbl)$.

Let $\Fcbl$ be a class of labeled graphs,
$\Fbl\In\Fcbl$ labeled,
and $\Gbl$ another labeled graph
then we define vector $\hom(\Fcbl,\Gbl)\In \Nat^\Fcbl$ entry-wise: each entry $\hom(\Fcbl,\Gbl)[\Fbl]$ 
is the number of graph homomorphisms $\varphi\colon V(G)\to V(H)$ 
that, moreover, preserve the label, i.e.,
\begin{align*}
\varphi(\lab(\Fbl)) = \lab(\Gbl).
\end{align*}
In cases where we need to explicitly indicate the labeled vertex of $\Gbl$, we write
$G^u$ for $u = \lab(\Gbl) \In V(G)$.

Next, a \emph{product $\Gbl_1\cdot \Gbl_2$} of two labeled graphs $\Gbl_1$, $\Gbl_2$
is the graph created by identification of $\lab(\Gbl_1)$ and $\lab(\Gbl_2)$
in the disjoint union of $\Gbl_1$ and $\Gbl_2$.

A \emph{quantum graph} is a formal linear combination of finitely many graphs.
A \emph{$1$-labeled quantum graph}, or simply \emph{labeled quantum graph}, is a formal linear combination (with real coefficients) of finitely many $1$-labeled graphs.
The homomorphism counting extends linearly to quantum graphs:
\begin{align*}
\hom(\Fcbl,\,\Gbl)\Big[\sum_{i=1}^{d}\alpha_i \Fbl_i\Big]
\coloneqq
\sum_{i=1}^{d}\alpha_i\,\hom(\Fcbl,\Gbl)[\Fbl_i],
\end{align*}
for the coefficients $\alpha_i\In \Real$, and the  $1$-labeled graphs $\Fbl_i$ for $i\In[d]$.

Moreover, quantum graphs $\Gbl_1, \Gbl_2$ can be naturally combined by sum, product, and exponentiation operations. 
A \emph{sum} $\Gbl_1 + \Gbl_2$ is the sum of their linear combinations.
A \emph{product} $\Gbl_1 \cdot \Gbl_2$ is the product of their linear combinations, where
we use the definition of the product of two labeled graphs.
Finally, an \emph{exponentiation} $(\Gbl_1)^k$ for an integer $k\geq 1$ is the $k$-fold product of $\Gbl$ with itself.
We will use a standard identity for $1$-labeled, possibly quantum, graphs:
\begin{align}
    \hom(\Fcbl,\Gbl)[(\Fbl)^k] = \left(\hom(\Fcbl,\Gbl)[\Fbl]\right)^k.\label{eq:homcounttexp}
\end{align}

The following result is due to \citet{Cai+1992}. 
Let $\mathcal{C}_{2,h}$ denote the class of formulas of two-variable first-order logic with counting quantifiers, where the quantifier depth is bounded by $h$.
\begin{theorem}[{\citet[Theorem 5.2]{Cai+1992}}]\label{thm:cai}
    Let $G^u$, $H^v$ be a pair of labeled graphs and $\chitree{h}$ be a coloring for $h\In\Nat$.
    Then the following are equivalent:
    \begin{enumerate}[label=(\roman*)]
        \item $\chitree{h}(G, u) = \chitree{h}(H, v)$,
        \item $(G, u) \equiv_{\mathcal{C}_{2, h}} (H, v)$.
    \end{enumerate}
\end{theorem}
The above result was followed by the work of~\citet{Dvo+2010} 
stating that the homomorphism counts of $1$-labeled graphs are 
also equivalent to the first-order logic with counting quantifiers.
\begin{theorem}[{\citet[Theorem 7]{Dvo+2010}}]\label{thm:dvorak}
    Let $G^u$, $H^v$ be a pair of labeled graphs,
    and let $\mathcal{T}_h^\bullet$ be the class of all $1$-labeled trees of depth $h$.
    Then the following are equivalent
    \begin{enumerate}[label=(\roman*)]
        \item $(G, u) \equiv_{\mathcal{C}_{2, h}} (H, v)$,
        \item $\hom(\mathcal{T}_h^\bullet, G) = \hom(\mathcal{T}_h^\bullet, H)$.
    \end{enumerate}
\end{theorem}

\begin{proposition}\label[proposition]{prop:homcounttt}
Let $h, t \In \mathbb{N}$, and let $G$ be a graph and $\chitree{h}\colon V(G)\to \Sigma(G)$ a coloring.
For each color $c \In \Sigma(G)$, choose a vertex $u \In V(G)$ such that $\chitree{h}(G, u) = c$, 
and denote by $G(c)$ the labeled graph $G^u$. Then, it holds:
\begin{align}
    \hom(\CaT h t, G)[F] = 
    \sum_{\substack{
        \vw\In \Sigma^t
    }} 
    \Walk{t}(G, \chitree{h})(\vw)
    \cdot
    \prod_{i=1}^{t} \hom(\mathcal{T}_h^\bullet, G(\vw[i]))[\Tbl_i],
    \label{eq:homcath2}
\end{align}
for every $F\In \CaT h t$, where $(\Tbl_1, \Tbl_2, \dots \Tbl_t)$
is a leg sequence of $F$ corresponding to the spine  $(\lab(\Tbl_1), \lab(\Tbl_2), \dots \lab(\Tbl_t))$.
\end{proposition}
\begin{proof}
Given a colored walk $\vw\In \Nat^t$, consider its occurrence $\vu=(u_1, u_2, \dots, u_t)$ in~$G$.
Let $k_{\vu}$ be the number of homomorphisms
$\varphi\colon V(F)\to V(G)$ that map the spine of $F$ to the occurrence $\vu$, 
that is, $\lab(T_i) = u_i$ for $i\In[t]$.

Specifically, for each $u_i$, independently, there are exactly 
$\hom(\mathcal{T}_h^\bullet, G^{u_i})[T_i^\bullet]$ ways to map the attached
leg $\Tbl_i$ into $G$:
\begin{align}
    \hom(\CaT h t, G)[F] = 
    \sum_{\vu \in V(G)^t}
    k_{\vu} =
    \sum_{\substack{
        \vu \text{ occurrence }\\
        \text{ of } \vw \text{ in } G
    }} 
    \prod_{i=1}^{t} \hom(\mathcal{T}_h^\bullet, G^{\vu[i]})[\Tbl_i].
    \label{eq:equationofahomsum}
\end{align}
Furthermore, by the theorems \cref{thm:cai} and \cref{thm:dvorak},
the number  of tree homomorphisms $\hom(\mathcal{T}_h^\bullet, G^{u_i})[T_i^\bullet]$
only depends on the color of $u_i$, which is 
 $\chitree{h}(G, u_i) = \vw[i]$.
 Indeed we get, $\hom(\mathcal{T}_h^\bullet, G^{u_i})[T_i^\bullet] = \hom(\mathcal{T}_h^\bullet, G(\vw[i]))[\Tbl_i]$.
 Finally, reorganizing the sum in \cref{eq:equationofahomsum} to range over all possible colored walks $\vw \In   \Sigma^t$
  using the known multiplicities $\Walk{t}(G, \chitree{h})(\vw)$,
 we obtain exactly the expression on the right-hand side of \cref{eq:homcath2}.
\end{proof}

We also make use of the following result of \cite{Dvo+2010}, referred to as Lemma~6.
\begin{proposition}[{\citet[Lemma 6]{Dvo+2010}}]\label[proposition]{prop:lemma6}
For every formula $\psi(x)\In \mathcal{C}_{2,h}$, 
there exists a quantum graph $\Tbl$ 
with its base in $\mathcal{T}_h^\bullet$
such that
\begin{align*}
    \hom(\Tbl, G^u) = \begin{cases}
        1, & \text{if } (G, u)\models \psi(x),\\
        0, & \text{otherwise.}
    \end{cases}
\end{align*}
\end{proposition}

We are prepared to restate and prove the main result of this section.

\restate{thm:homht}{
    For every $h, t\ge 0$, it holds that $\Walk{t}(-, \chitree{h}) \Eeq \hom({\CaT h t}, -).$
}

\begin{proof}
Following the structure of the proof of \cref{thm:hom1t},
we first show that $\Walk{t}(-, \chitree{h}) \Ege \hom({\CaT h t}, -)$.
Suppose that $\Walk{t}(G, \chitree{h}) = \Walk{t}(H,\chitree{h})$ for two given graphs $G$ and $H$.
By \cref{prop:homcounttt} we have that $\hom(\CaT h t, G)[F] = \hom(\CaT h t, H)[F]$ for every entry indexed by $F\In \CaT h t$.

For the other direction, 
suppose that $\Walk{t}(G, \chitree{h}) \neq \Walk{t}(H,\chitree{h})$ for two given $G$ and $H$.
We want to prove the existence of a caterpillar in $\CaT h t$ for which the homomorphism counts differ.
We denote the common finite set of colors given by $\chitree{h}$ in both graphs by $\Sigma = \Sigma(G) \cup \Sigma(H)$.
Given a color $c\In \Sigma$, we select a vertex $u \In V(G) \sqcup V(H)$ such that $\chitree{h}(G', u) = c$, where $G'$ is the disjoint union of $G$ and $H$.
We denote $G'^\bullet$ such that $\lab(G'^\bullet) = u$ by $G'(c)$.

Additionally, for each color $c \In \Sigma$, 
there exists a formula $\psi_c \In \mathcal{C}_{2,h}$ 
such that, for every graph $G$ and vertex $u \In V(G)$, we have $(G, u)\models \psi_c(x)$ if and only if $\chitree{h}(G, u) = c$. 
This follows from Claim~2 in the proof of Theorem~5.5.3 given in~\citep{Gro+2017}; see also~\citep{Cai+1992,Imm+1990}.

By \cref{prop:lemma6}, it follows that for each $\psi_c(x)$
there is a quantum graph~$\Tbl_c$ of depth at most $h$
such that $\hom(\Tbl_c, G^u) = 1$ if $\psi_c(x)$ holds and $0$ otherwise.

We fix a linear order on $\Sigma$ by $\{c_1, c_2, \dots, c_p\}= \Sigma$, where $p = |\Sigma|$.
For each color $c_i$ there is a quantum graph $\Tbl_{c_i}$ corresponding to $\psi_{c_i}$.
We denote the quantum graph obtained by the scalar multiplication by $i$ as  $\Lbl_{c_i} = (i+1)\cdot \Tbl_{c_i}$, so that 
\begin{align*}
    \hom(\mathcal{T}_h^\bullet, G^u)[\Lbl_{c_i}] = \begin{cases}
        i+1, & \text{if } \chitree{h}(G, u) = c_i,\\
        0, & \text{otherwise.}
    \end{cases}
\end{align*}

Let $n$ bound the number of vertices of $G$ and $H$,
let $m^X_\vw = \Walk{t}(X, \chitree{h})(\vw)$ for $X \In \{G, H\}$ and $\vw\In \Sigma^t$.
We apply \cref{lem:tuplefaithful} with bounds $D_0 = n^{t}$, and $D_1 = \dots = D_t = p+2$ to get 
injective representations of tuples in $[n]^t \times [p+2] \times \cdots \times [p+2]$.
As a result, we obtain coefficients $(s_1, s_2, \dots, s_t)$
such that the function
\begin{align*}
    \left(m_\vw^X, \vw[1], \vw[2], \dots, \vw[t]\right) \mapsto
    m_\vw^X\cdot 
    \prod_{i=0}^t
    \big(\hom(\mathcal{T}_h^\bullet, G'(\vw[i]))[\Lbl_{\vw[i]}]\big)^{s_i}
\end{align*}
is injective provided that $m_\vw^X > 0$ on both $X \In \{G, H\}$.

Next, for each $\vw = w_1 w_2 \dots w_t \In \Sigma^t$, we construct a quantum caterpillar graph $F_{\vw}$ with a sequence of $t$ legs given by $\big((\Lbl_{w_1})^{s_1}, (\Lbl_{w_2})^{s_2}, \dots, (\Lbl_{w_t})^{s_t}\big)$. 

In case either $m^G_\vw = 0$ or $m^H_\vw = 0$, choose $k=1$ as a multiplier of coefficients $(s_1, \dots, s_t)$ to satisfy \cref{eq:distinctcat2} directly with $F^1_\vw$ given by \cref{eq:distinctcat3}.
Otherwise, to find suitable $k$, we apply \cref{lem:distinct} by setting $\vx, \va, \vb\In [p+2]^t$:
\begin{align*}
\vx[i] = \prod_{j=0}^t (\hom(\mathcal{T}_h^\bullet, G'(\vw_i[j]))[\Lbl_{\vw_i[j]}])^{s_j},
\quad
\va[i] = m^G_{\vw_i},
\quad
\vb[i] = m^H_{\vw_i},
\end{align*}
where $i$ is the index enumerating each $\vw_i$ in $[p+2]^t$.
As a result we get sufficiently large exponent $k$ distinguishing the expressions involving $\va$ and $\vb$. 
Specifically, we consider the quantum caterpillar given by the spine
\begin{align}
F_{\vw}^k = \big((\Lbl_{w_1})^{k\cdot s_1},\, (\Lbl_{w_2})^{k\cdot s_2},\,\dots,\,(\Lbl_{w_t})^{k\cdot s_t}\big),\label{eq:distinctcat3}
\end{align}
for which we obtain the following 
\begin{align}
\hom(\CaT h t, G)[F_{\vw}^k] &= (\hom(\CaT h t, G)[F_{\vw}])^k \\
&\neq (\hom(\CaT h t, H)[F_{\vw}])^k = \hom(\CaT h t, H)[F_{\vw}^k],\label{eq:distinctcat2}
\end{align}
where both equalities follow from \cref{prop:homcounttt} and the identity in \cref{eq:homcounttexp}.

Finally, because every quantum graph is a linear combination of non-quantum ones, it follows immediately that there exists at least one non-quantum caterpillar graph $F' \In \CaT h t$ in the base of $F^k_\vw$ for which
\begin{align*}
\hom(\CaT h t, G)[F'] \neq \hom(\CaT h t, H)[F'].
\end{align*}
\end{proof}

\begin{corollary}\label[corollary]{cor:homht}
    For every $h\ge 0$ it holds that $\Walkk(-, \chitree{h})\Eeq \hom(\Cat_h, -)$.
\end{corollary}
\begin{proof}
    It follows from \cref{thm:homht} and the fact that $\Cat_h$ is a superclass of $\CaT h t$ for $t\In\Nat$.
\end{proof}

\begin{figure}[h]
    \centering
    \includegraphics[width=0.5\textwidth]{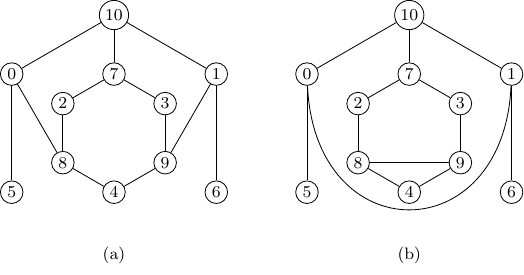}
    \caption{
        An example of a pair of graphs $G$ and $H$ that strictly separates the graph functions $\hom(\Cat_1, -)$ and $\hom(\Cat_0, -)$. 
        Note that the graph $H$ can be derived from $G$ by removing the edges $\{0,8\}$ and $\{1,9\}$, and then adding the edges $\{0,1\}$ and $\{8,9\}$.
        }
        \label{fig:separation}
    \end{figure}
    \subsection{Strict separation}

For a strict separation, we define closure under disjoint unions of generalized caterpillar graphs in order to apply existing theory on minor-closed classes.
Formally, let $G = (V,E)$ be a graph and let $e=uv\In E$ be an edge.
The \emph{contraction} of $e$ in $G$ is a graph $G/e$ with removed $e$ and identified vertices $u$ and $v$.
Graph $G_m$ is a \emph{minor} of $G$ if it can be obtained by a sequence of edge deletions, vertex deletions and edge contractions from $G$.

A class of graphs $\mathcal{C}$ is \emph{minor-closed} if for every $G \In \mathcal{C}$ 
and every minor $G'$ of $G$, we have $G' \In \mathcal{C}$. 
We denote by $\CaTu h t$ the closure of $\CaT h t $ under finite disjoint unions, 
that is, $\CaTu h t \supseteq \CaT h t$ and if $(V_1,E_1),(V_2,E_2) \In \CaTu h t$ then 
$(V_1 \sqcup V_2,\, E_1 \sqcup E_2)\In \CaTu h t.$
Analogously, we denote by $\Catu_h$ the closure of $\Cat_h $ under finite disjoint unions.
\begin{proposition}
For all integers $h, t\ge 0$, the class $\CaTu h t$ is minor-closed.
\end{proposition}
\begin{proof}
Take a graph $G \In \CaTu h t$ and consider either deletion or contraction of its edge $uv$.
Edge $uv$ lies in one of the connected components $G'$, which lies in $\CaT h t$.
Therefore, there is a sequence of legs $((L_1, s_1),\dots (L_t, s_t))$.
We now consider following cases:
\begin{enumerate}
\item Contraction of edge in $(L_i, s_i)$: we obtain by contraction $(L_i / uv, s_i)$ which is in $\mathcal{T}^\bullet_h$
\item Deletion of edge in $(L_i, s_i)$: we obtain two graphs, one of them contains $s_i$, $(L'_i, s_i)\In \mathcal{T}^\bullet_h$,
and for the second one containing $u$ (without loss of generality) $(L'_i, u)\In \mathcal{T}^\bullet_h$.
\item Contraction of edge $s_i s_{i+1}$: we obtain only shorter spine, thus $G/s_i s_{i+1}$ in $\CaTu h {t-1} \subseteq \CaTu h t$.
\item Deletion of edge $s_i s_{i+1}$: we obtain two components, each is of them is again of a shorter spine.
\item Vertex deletions are analogous.
\end{enumerate}
\end{proof}

The following lemma is a consequence of \citet[Lemma 5.14]{Rob+2022}.
\begin{lemma}\label[lemma]{lem:strictsep01}
It holds $\hom(\Catu_0, -) \Elt \hom(\Catu_1, -)$.
\end{lemma}
\begin{proof}
Follows from a stronger statement about the homomorphism distinguishability closedness of some 
classes closed on minors and disjoint unions.
Specifically for $\Catu_0$, this was shown by \citet[Lemma 5.14 as remarked in Section 5.1]{Rob+2022}.
For illustration, we give a concrete separating example in \autoref{fig:separation}.
\end{proof}

The following lemma is a consequence of \citet[Theorem 4.13.]{Sch+2025}.
\begin{lemma}\label[lemma]{lem:strictsep12}
It holds $\hom(\Catu_1, -) \Elt \hom(\Catu_2, -)$.
\end{lemma}
\begin{proof}
The class of unions of caterpillars $\Catu_1$
corresponds to the graph class of pathwidth at most $1$ \citep[Section 6]{Pro+1999}.
We apply a stronger statement about the homomorphism distinguishability closedness of
$\Catu_1$, this was shown recently by \citep[Theorem 4.13.]{Sch+2025}.
\end{proof}

\begin{lemma}\label[lemma]{lem:strictseparation}
It holds $\hom(\Cat_t, -) \Elt \hom(\Cat_{t+1}, -) \Elt \hom(\mathcal{T}, -)$ for every $t\In \Nat$.
\end{lemma}
\begin{proof}
\emph{Case $t \in \{0,1\}:$ }
For a disjoint union of two graphs $F_1,F_2$ it holds $\hom(F_1 \sqcup F_2, G) = \hom(F_1, G)\cdot \hom(F_2, G)$, e.g. \citep[pg. 74]{Lov+2012}.
For any $h\In \Nat$, if $\hom(\Catu_h, G)[F_1\sqcup F_2] \neq \hom(\Catu_h, H)[F_1\sqcup F_2]$,
then $\hom(\Catu_h, G)[F_1] \neq \hom(\Catu_h, H)[F_1]$ or $\hom(\Catu_h, G)[F_2] \neq \hom(\Catu_h, H)[F_2]$.
Therefore, we get $\hom(\Catu_h, -) \Eeq \hom(\Cat_h, -)$.
Finally, we use \cref{lem:strictsep01} and \cref{lem:strictsep12}.

\emph{General case $t \ge 2:$}
Similarly, the remaining equivalences are due to \cite{Neu+2026}[Theorem 1.5], 
stating a general result for all forests closed on (topological) minors and disjoint unions.
\end{proof}

\subsection{Expressivity hierarchy}
\begin{corollary}\label[corollary]{cor:expressivityhierarchy}
    There is the following expressivity hierarchy such that for any integer $h\ge 3$:    
\[
\begin{tikzpicture}[every node/.style={anchor=center}, font=\normalsize, scale=0.75]
    \matrix (m) [matrix of math nodes, nodes in empty cells, row sep=0.45cm] {
    \hom(\mathcal{P},-) 
    &\Elt & \hom(\Cat_1,-) 
    &\Elt & \hom(\Cat_2,-) 
    &\Elt{\text{\scriptsize $\cdots$}}\Elt& \hom(\Cat_h,-) 
    &\Elt{\text{\scriptsize $\cdots$}}\Elt& \hom(\mathcal{T},-)\\
    \invEA(-,\chitriv) 
    &\Elt & \invEA(-,\chideg) 
    &\Elt & \invEA(-,\chitree{2}) 
    &\Elt{\text{\scriptsize $\cdots$}}\Elt& \invEA(-,\chitree{h}) 
    &\Elt{\text{\scriptsize $\cdots$}}\Elt& \Treee(-),\\
    };
    \foreach \col in {1,3,5,7,9}{
        \node[rotate=90] at ($(m-1-\col)!0.5!(m-2-\col)$) {$\Eeq$};
    }
    \label{diag:hierarchy_restated}
\end{tikzpicture}
\]
\end{corollary}
\begin{proof}
The second equivalence in the hierarchy is given by \cref{thm:hom1t}, 
while the intermediate equivalences follow from \cref{thm:homht} and \cref{thm:walkInvariantEquivalence}. 
The strict separations follow from \cref{lem:strictseparation}. 
The last equivalence follows from \cref{thm:dvorak}.
\end{proof}


\begin{figure}[h]
        \centering
        \includegraphics[width=0.4\textwidth]{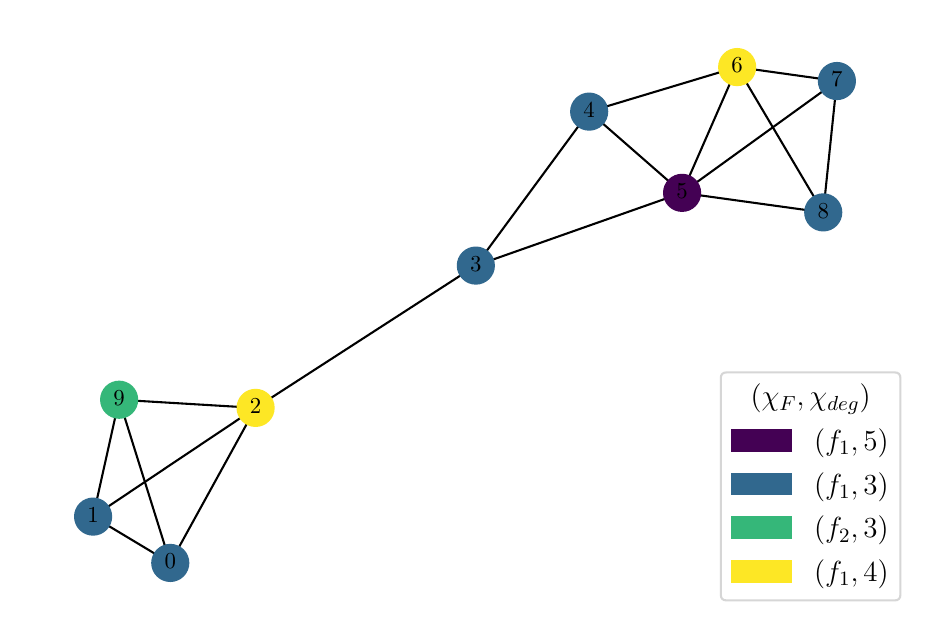}
        \caption{
            An example graph representation of a protein structure, colored by $(\chi_F, \chideg)$.
            The shown datapoint is taken from the Proteins dataset~\citep{Mor+2020}.
        }
        \label{fig:comp_proteins_example}
\end{figure}

\section{Graph Neural Networks}\label{appendix:gnns}
In this section, we further discuss variants of Caterpillar GNN,
a practical architecture built upon the theoretical foundations of graph-level aggregation.
Graphs in real-world datasets such as Proteins~\citep{Mor+2020}, ZINC~\citep{Irw+2012}, or ESOL~\citep{Wu+2018}
typically come attributed with \emph{vertex features}. An example is given in \autoref{fig:comp_proteins_example}.
Here, we assume that these features are seen as categorical, taken from a finite set $\Sigma$.
Rather than encoding continuous-valued properties, these features represent discrete properties 
such as an atom type or molecular class.

The vertex features of a given graph $G$, we represent as a coloring function $\chi_F(G, -)\colon V(G)\to \Sigma$. To seamlessly integrate vertex features with our scalable aggregation scheme,
 we introduce a parametrized combined coloring that incorporates both the vertex features and the refinement coloring:
\begin{align}
    \tilde{\chi}^\tup{h}(G, u) = (\chi_F(G, u), \chitree{h}(G, u)),\label{eq:tildechi}
\end{align}
for every vertex $u \In V(G)$ and $h\In \Nat$.
The primary motivation for combining with $\chi_F$
 is to prevent the computational graph from being too downscaled to accommodate distinct vertex features.
 Technically, in the following architecture, we employ vertex-feature matrices $\mY=\mY(G, \chi_F)$ indexed by (specifically selected) colored walks $\va c\In\Sigma^{t+1}$, and feature channels $i$. Crucially, each entry $Y[\va c, i]$ only needs to represent the $i$-th channel of the vertex feature $c = \chi_F(G, u)$ for a suitable $u\In V(G)$. This is due to the structure of the prefix-successor relation we explained in~\autoref{sec:efficientaggregation}.

\subsection{Caterpillar GCN}\label{appendix:caterpillargcn}

Let $L$ denote the number of layers in the network.
For each layer $\ell$, where $0\le \ell \le L$ and $t_\ell = L - \ell$,
we define the following parameters:
$c_\ell \In \Nat$, the number of feature channels;
$\mW^\tup{\ell} \In \Real^{c_\ell\times c_{\ell+1}}$, a learnable weight matrix;
$\mY^\tup{\ell}$, vertex feature matrix with channel embeddings for $c_\ell$;
$\sigma$, an activation function, e.g. ReLU. 
Recall from Section~\ref{sec:efficientaggregation} that 
$S_t = S_t(G, \tilde \chi)\subseteq \Sigma^t$ denotes canonical subsets 
of colored walks of size at most $|V(G)|$.

We derive a \emph{Caterpillar GCN} 
as a sequence of layers transforming feature matrixes.
Specifically, for each layer $\ell$, we define the features
 $\vh^\tup{\ell} \In \Real^{S_{t_\ell} \times c_\ell}$ as follows:
\begin{align*}
    \vh^\tup{0} &= \mY^\tup{0},\quad
    \vh^\tup{\ell+1} =
    \sigma\big( 
        \mC^{\large \tilde \mA}_{t_\ell}\, 
        \vh^\tup{\ell}\,
        \mW^\tup{\ell}
        \big)
        \,\Box\, \mY^\tup{\ell}
        ,\quad
        \vh^\tup{L} = 
        \mC^{\mI}_{0}\, 
        \vh^\tup{L}\,
        \mW^\tup{L},
    \end{align*}
    where $\Box$ represents a standard addition or a concatenation. 
    Finally, the resulting graph-level feature is 
    \begin{align}
        \vh(L, \theta; G,\chi)\coloneqq \vh^\tup{L}\In \Real^{\{\emptyword\}\times c_L},
    \end{align} 
    where $\theta$ is the set of all learnable parameters.
    For the $\ell$-th layer, we used the $t_\ell$-th \reduced matrix 
    (\autoref{eq:catmat}). We specifically set $\mM\coloneqq\tilde{\mA}$ for the reduced variant $\mC^\mM_t$,
    where $\tilde{\mA}$ is the augmented normalized adjacency matrix, see \citet{Kip+2017}, defined as follows:
    $\tilde{\mA} = \hat{\mD}^{-1/2} \hat{\mA} \hat{\mD}^{-1/2}$, where $\hat{\mA} = \mA + 2\mI$, and $\hat{\mD}+2\mI$.
    For an example of such a computational graph, see \autoref{fig:comp_proteins}.
    
    Indexing of the \reduced variants of graph matrices ensures the correct structure of aggregation 
    from longer to shorter words, and, ultimately, $\mC^{\mI}_{0}$ maps to the space of dimension $S_0 = \{\emptyword\}$, as shown in~\autoref{fig:efficient_neighborhood}.
    Note that if we set $\tilde\chi = (\chi_F, \chiid)$
    and operation $\Box$ to ignore $\mY$, our architecture becomes nothing else than a network of GCNConv \citep{Kip+2017} with global readout.

\begin{figure}[t]
            \centering
            \begin{minipage}{0.48\textwidth}
                \centering
                \includegraphics[width=0.9\textwidth]{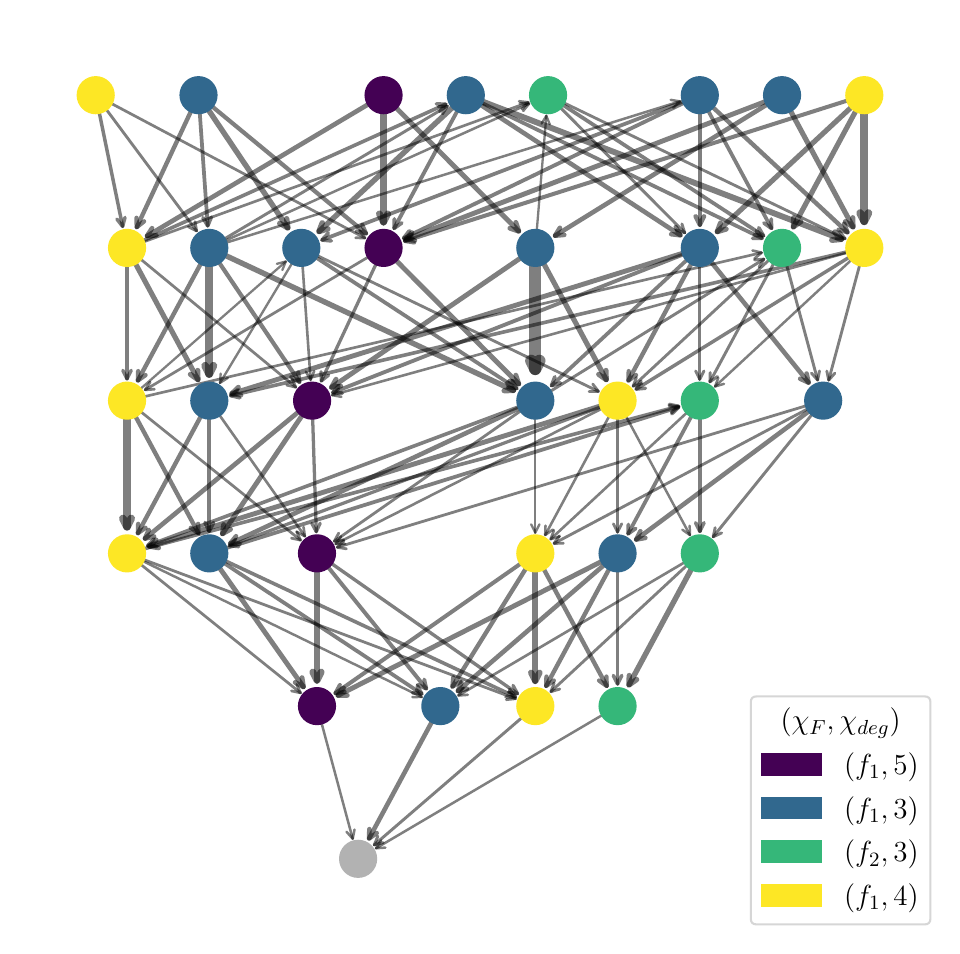}
            \end{minipage}
            \hfill
        \begin{minipage}{0.48\textwidth}
            \centering
            \includegraphics[width=0.9\textwidth]{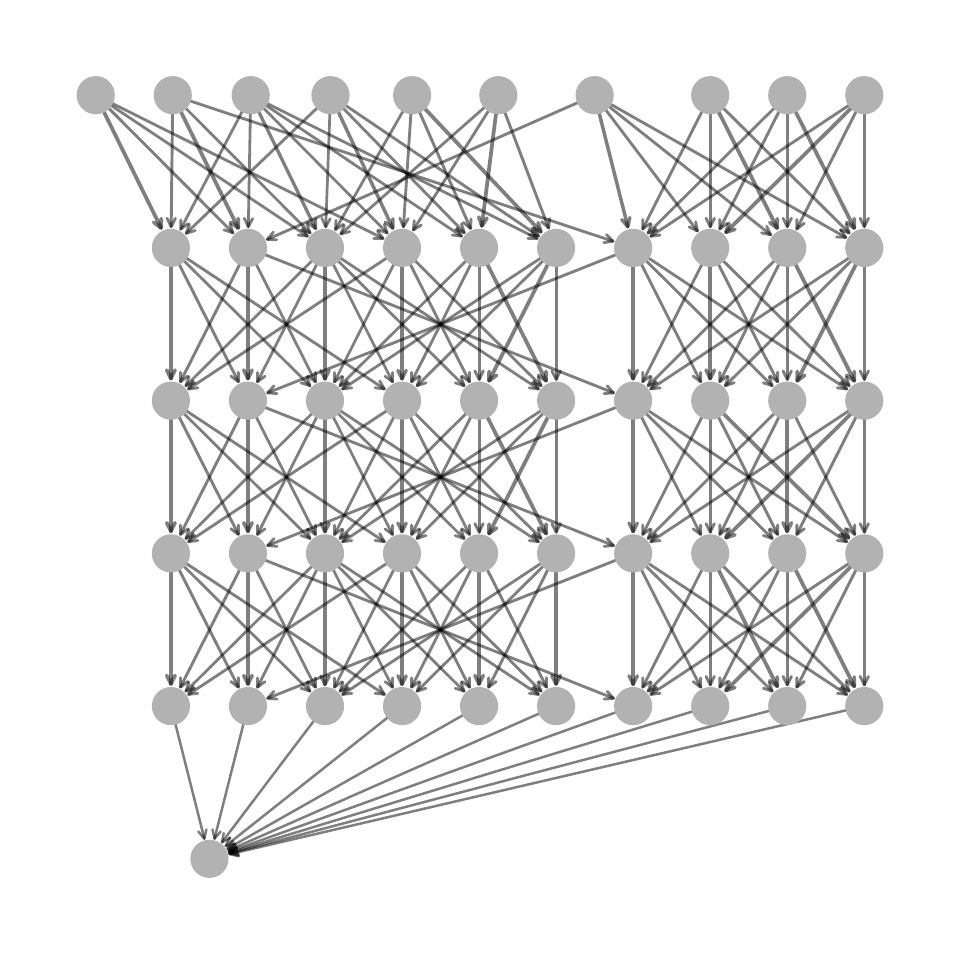}
        \end{minipage}
            
            \caption{
            Left: Computational graph \textbf{(ours)} of Caterpillar GCN with $h=1$ that uses  $\mC^{\large \tilde \mA}_{t_\ell}$ matrix at $\ell$-th layer.
            This diagram is analogous to \autoref{fig:efficient_neighborhood}.
            Right: 
            Computational graph (message-passing) of GCN that uses copies of matrix $\tilde\mA$ for its layers. 
            Weights are represented by line width, and signs by arrow direction.
            The input graph is shown in \autoref{fig:comp_proteins_example}.
            }
            \label{fig:comp_proteins}
\end{figure}

\subsection{Density of Computational Graphs}\label{appendix:density_computational_graphs}

As defined in the previous subsection,
while computational graphs of GCNs are given by matrix $\tilde{\mA}$ 
the computational graphs of Caterpillar GCNs are given by the \reduced variants $\mC^{\large \tilde \mA}_t$ (see \autoref{fig:efficient_neighborhood} and \autoref{fig:comp_proteins}).

Even though the number of nodes of the computational graph may \emph{decrease} significantly with height, as shown in \autoref{fig:height}, and further detailed in \cref{tab:dataset_statistics_acc}, 
this does not necessarily imply a proportional decrease of the number of edges.
To study this phenomenon, we introduce two metrics to quantify the sparsity of computational graphs.

    For any matrix $\mM\In \mathbb{R}^{R \times C}$, we denote the number of its \emph{non-zero entries} by 
    \begin{align*}
        \nnz(\mM) \coloneqq \left|\left\{(r,c) \mid \mM[r,c] \neq 0,\,\,  r\In R,\,\, c\In C \right\}\right|.
    \end{align*}
    Then the \emph{density} of $\mM$ is defined as the ratio of its non-zero entries to the total number of entries, i.e.,
    $\nnz(\mM)\cdot |R\times C|^{-1}$.
    Complementarily, the \emph{sparsity} of $\mM$ is defined as $1 - \nnz(\mM)\cdot |R\times C|^{-1}$.

    \paragraph{Absolute density.}
    We extend density to the computational graph of Caterpillar GCN by taking the average of layer densities.
    Let $G$ be an input graph with a $\Sigma$-coloring $\chi$, and let $L$ denote the number of layers. 
\begin{align}
    \operatorname{absolute\_density}(L, G, \chi) \coloneqq 
    \frac{1}{L}
    \sum_{t=1}^{L} 
    \frac{
        \nnz\big(\mC^{\large \tilde \mA}_t\big)
    }{
        |S_{t} \times S_{t+1}|
    }.
\end{align}

\paragraph{Relative density.} 
By relative density $\delta_{h, L}$, we measure the ratio of the number of edges in the computational graph of Caterpillar GCN to that of (standard) GCN.
\begin{align}
    \operatorname{relative\_density}(L, G,\chi) &\coloneqq
    \frac{1}{L}
    \sum_{t=1}^{L} 
    \frac{
        \nnz\big(\mC^{\large \tilde \mA}_t\big)
    }{
        \nnz\left(\mA\right)
    },\quad\text{ and } \notag\\ 
    \delta_{h, L} &\coloneqq \operatorname{relative\_density}\big(L, G,(\chi_F, \chitree{h})\big).\label{eq:reldensity}
\end{align}
Note that for the computational graph of GCN, the relative density equals one.
By Proposition~\ref{obs:specialmessagepassing}, we can state this formally as
 $\operatorname{relative\_density}(L, G, \chiid) = 1.0.$
Moreover, 
since $\nnz(\mC^{\large \tilde \mA}_t) = \nnz(\mC^{\large \hat \mA}_t) = \nnz(\mC^{\large \mA}_t) + \nnz(\mC^{\large \mI}_t)$, these metrics extend naturally to computational graphs of other variants of GNNs.

    \paragraph{Densities and dimensions of $\mC^{\large \tilde \mA}_t$ on real-world datasets.}
    To better understand the sparsity (density) of computational graphs induced by our model, 
    we empirically evaluate how the density evolves with increasing height parameter $h$ for fixed $L = 5$.
    Specifically, for a fixed sequence length $L = 5$ and varying height parameter~$h$, we measure the densities (defined above)
    of the resulting computational graphs across TUDataset, MoleculeNet, and ZINC, as summarised in \autoref{tab:sparsity_computation_graphs}.

In our experiments, we observe that the relative density $\delta_{h,5}$ is, in most cases, strictly below~$1.0$, and can overall be bounded by a small constant, for instance $\delta_{h,5} \lessapprox 2.0$. This implies that the matrices $\mC^{\large \tilde{\mA}}_t$ \emph{typically contain fewer non-zero entries} than the original adjacency matrix~$\tilde{\mA}$. On unattributed graphs such as IMDB-BINARY and COLLAB, the ratio $\delta_{h,5}$ tends to be even lower (e.g., approximately $0.1$ and $0.3$, respectively).

    On the other hand, the $\operatorname{absolute\_density}$ tends to increase as the height parameter~$h$ decreases.
    This indicates that the computational graphs become significantly \emph{denser}, which is beneficial for execution on hardware accelerators, provided that the relative density $\delta_{h,L}$ does not increase by more than a small constant factor.
    On typically sparse datasets such as \textsc{PROTEINS} and \textsc{ENZYMES}, we observed a substantial rise in absolute density (from approximately $20\%$ for \textsf{MP} to around $50\%$ for \textsf{C}$_0$), suggesting that such computational graphs may be more desirably processed using dense tensor operations rather than sparse ones.

\begin{table}[t]
    \centering
    \includegraphics[width=0.99\textwidth]{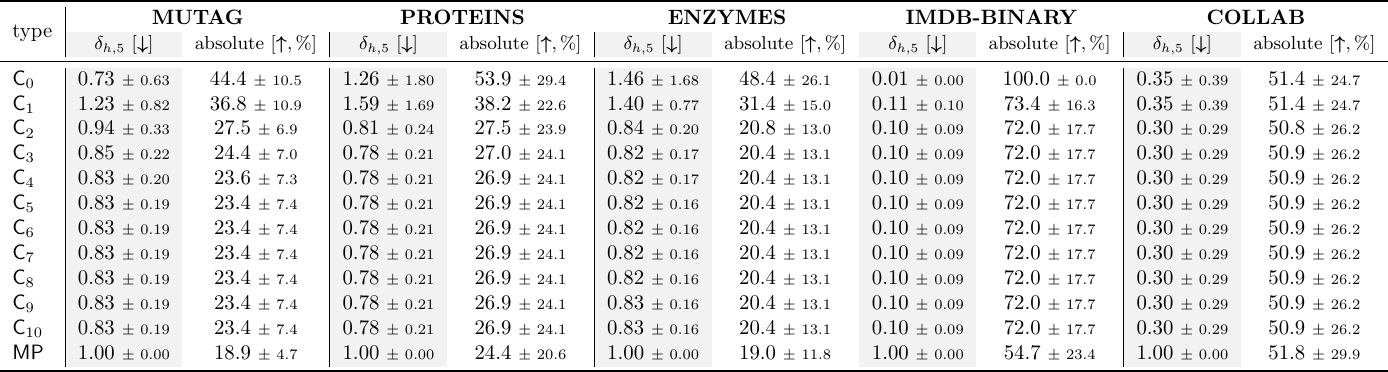}
    \includegraphics[width=0.99\textwidth]{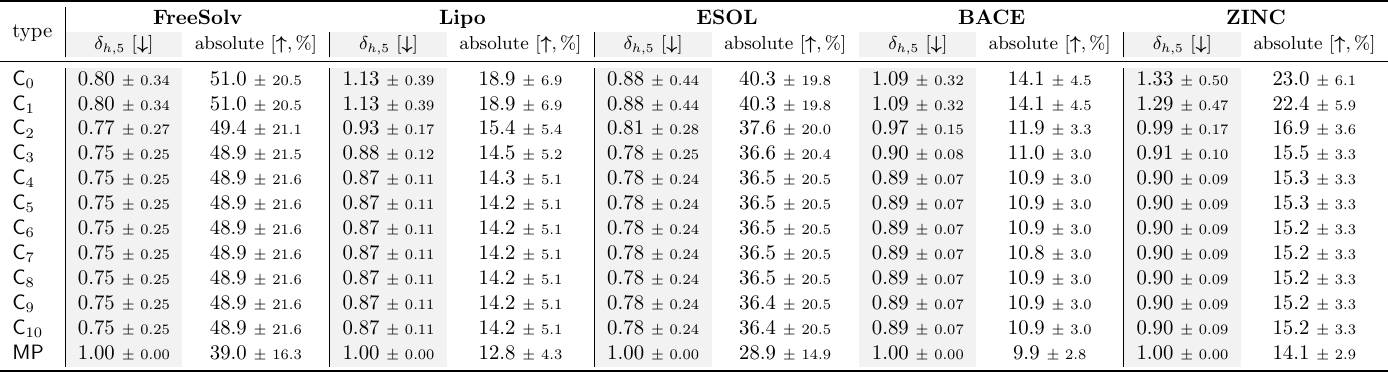}
    \vspace{0.5em}
    \caption{
        \emph{Density of computational graphs across datasets.}
        By $\textsf{C}_h$ is denoted the (caterpillar height) parameter $h\In\Nat$ of graph-level aggregation (ours), while $\textsf{MP}$ denotes the full message-passing GCN.
        The subcolumns correspond to $\delta_{h,L}$ (lower is better) and $\operatorname{relative\_density}$ (higher is better once $\delta_{h,L}$ is bounded) for $L=5$ layers.
        We report absolute density in $\%$.
    }
    \label{tab:sparsity_computation_graphs}
\end{table}

\subsection{Impacts of Relative Density on Asymptotic Complexity}
In this subsection, we discuss the asymptotic time complexity of Caterpillar GNN (including precomputation) 
in comparison to some other methods learning on sequential patterns in graphs.
We first focus on relative density $\delta_{h,L}$, and later, we also discuss the impact of the number of nodes in the computational graph.

When the number of channels $d$ in the GNN is taken as a constant, we give a short overview \cref{tab:asymptotic_comparison}.
Assuming that an input graph $G$ on $n$ vertices is connected and $n\ge 3$,
it holds for the maximal degree $\Delta$ that $2\le \Delta \le n-1.$ 
A multiplicative factor in our comparison $\delta_{h, L}$ is the average relative density, as defined in \cref{appendix:density_computational_graphs}.
For Caterpillar GNNs, it is fixed that $L = T.$ 
The best known theoretical bound to our knowledge is the trivial one: $\delta_{h, L} \le n^2 / |E|$.
Importantly, on TUDataset, MoleculeNet and ZINC, we see that $\delta_{h, 5}$ is typically a fraction smaller than $1.0$ and overall
$\delta_{h, 5} \lessapprox 2.0$,  see \cref{tab:sparsity_computation_graphs}.

\begin{table}[t]\label{tab:asymptotic_comparison}
    \centering
\begin{tabular}{l l l l c}
    \small
    Method & \small{Pre-comp.} & \small{Time per layer} & \small{Worst case} & \small{Equivariant} \\
    \midrule
    \small MPGNN \citep{Gil+2017} & --- & $|E|$ & $n^2$ & \checkmark  \\
    \small CRaWL \citep{Ton+2023} & --- & $n\cdot f_G(\Delta^T)\cdot T$ & $nf_G((n-1)^T)\,T$ & X\\
    \small NeuralWalker \citep{Che+2024} & --- & $n\cdot f_G(\Delta^T)\cdot T$ & $nf_G((n-1)^T)\,T$ & X\\
    \small Path NN \citep{Gas+2023} & $n\cdot\Delta^T$ & $ n \cdot \Delta^T \cdot T $  & $n^{T+1}\, T$ & \checkmark \\
    \small Path-based GNN \citep{Gra+2024} & $n\cdot\Delta^T$ & $n \cdot \Delta^T \cdot T$ & $n^{T+1}\, T$ & \checkmark \\
    \small Caterpillar GNN (ours) & $T\cdot n^{2.372}$ & $|E| \cdot \delta_{h,L} $ & $n^2$ & \checkmark \\
\end{tabular}
\caption{
Asymptotic comparison of methods processing sequential patterns of length $T.$
The function~$f_G$ gives the number of sampled walks on $G=(V,E)$ from the given upper bound.
Symbol $\Delta$ stands for the maximal degree and $\delta_{h,L}$ is the relative density on $G$ (see \autoref{eq:reldensity}).
}
\end{table}

\paragraph{Complexity under large width.}
In the regime where the number of channels $d$ (i.e., the width) of our GNN is large, 
the multiplication of learnable parameters $\mW^{(\ell)} \In \mathbb{R}^{d \times d}$ with features 
$\vh^{(\ell)} \In \mathbb{R}^{|S_{t_\ell}| \times d}$ (see \cref{appendix:caterpillargcn}) becomes a dominant contributor to both time and space complexity. 
Concretely, the cost of one layer is
\begin{equation}
    \mathcal{O}\bigl(\nnz(\mC^{A}_{t_\ell})\cdot d + |S_{t_\ell+1}|\cdot d^2\bigr),
\end{equation}
hence for a network with $L$ layers ($t_\ell = L-t$) and sufficiently large $d$, the term
$d^2 \cdot \sum_{t=0}^{L} |S_t|$
dominates the overall complexity. 
Hence, the \emph{number of nodes in the computational graph}, $\sum_{t=0}^{L} |S_t|$, appears as a multiplicative factor in the time and memory complexity and is asymptotically significant with respect to the width $d$.

Empirically, on real-world datasets TUDataset, MoleculeNet, and ZINC, 
we observe that the number of nodes in the computational graph tends to \emph{decrease} as the height parameter $h$ decreases, 
as illustrated in \autoref{fig:height} and detailed further in \cref{tab:dataset_statistics_acc,tab:dataset_statistics_mae}.

\subsection{Towards Modern Architectures and Large Graph-Level Datasets}\label{appendix:modern_architectures}

To broaden our empirical evaluation, this subsection extends our analysis of computational graphs to more practical settings (\cref{tab:dataset_statistic_group}).
In the previous analysis, we fixed as many hyperparameters as possible, such as the number of layers.
Here, in contrast, for structural comparison against a strong message-passing baseline, 
we adopt the experimental settings of the recent GCN-based architecture (GCN$^{+}$) proposed by \citet{Luo+2025}. 
Their work demonstrates that, with careful design and tuning, message-passing backbones can be highly competitive on graph-level prediction tasks.

Following their setup, we preserve the configuration parameters such as the number of layers, and the dimensionality of the preprocessed node features, and we reuse their protocol to ensure a controlled structural comparison where only the aggregation mechanism differs.
To further isolate the impact of lower expressivity, we intentionally disable positional encodings and edge features for all architectures, consistently with the GCN$^{+}$ ablation setting used in Table 6 of Section 5.2 in \citep{Luo+2025}. 
Hence, we obtain a clear assessment of how graph-level aggregation contributes to graph-level performance uninfluenced by other structural augmentations.

We analyze the computational graphs of the MalNet-Tiny dataset \citep{Fre+2021}, which contain graphs with up to $5,000$ nodes ($1,410$ on average). 
In addition, we also include an analysis of datasets with discrete node features: Peptides-func, Peptides-struct from Open Graph Benchmark \cite{Hu+2021}, and datasets ogbg-code2, ogbg-molhiv, ogbg-molpcba from Long-Range Graph Benchmark \citep{Dwi+2022}, see \cref{tab:dataset_statistic_group}.
These datasets allow us to examine how the structure of proposed inductive biases behaves under substantially larger graph structures and more diverse data modalities.
We report the preprocessing times of a \emph{serial} precomputation on the platform of an Intel Xeon Platinum 8168 @ 2.70 GHz processor (24 cores, 48 threads, 33 MB L3 cache, 64B cache line).

\paragraph{Recipe for equivariant density reduction.}
Our proposed reduced variants of graph matrices are fully (permutation) equivariant.
In practice, we can therefore further reduce the value of relative density $\delta_{h,L}$ (\autoref{eq:reldensity}) by, e.g., thresholding small-magnitude entries of our matrices. 
This optional step decreases computational cost while preserving the symmetries inherent to our construction.
On the other hand, during our measurements in \cref{tab:dataset_statistic_group}, similarly to our previous analysis, we observe that the relative density $\delta_{h,L}$ of the computational graphs is a very small constant.
The only exception we found was in the case of $\delta_{h, 8}$ for the MalNet-Tiny and the height $h=0$, for which we paid particular attention to eliminate possible hardware errors.

\begin{table}[t]
    \centering
    \includegraphics[width=0.99\textwidth]{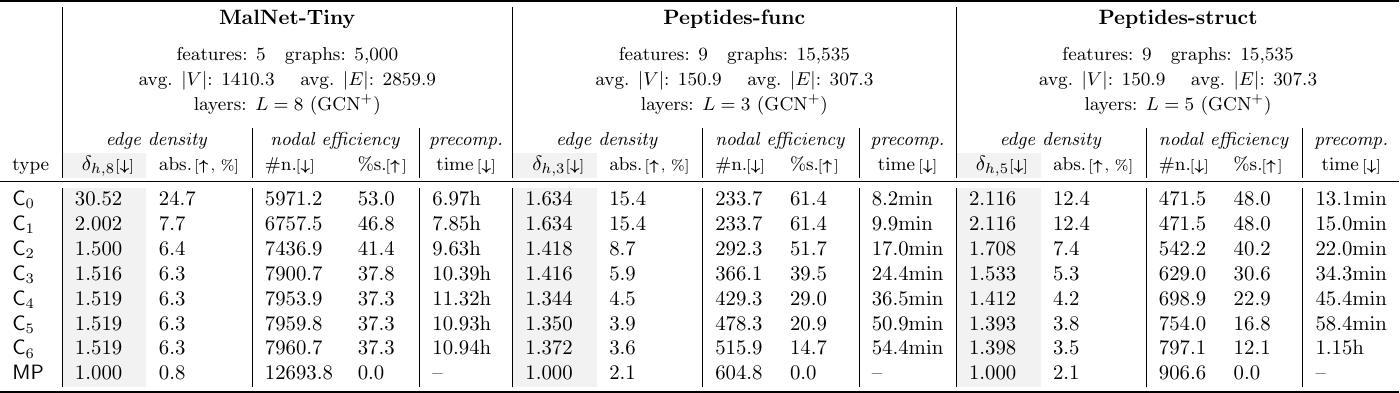}
    \includegraphics[width=0.99\textwidth]{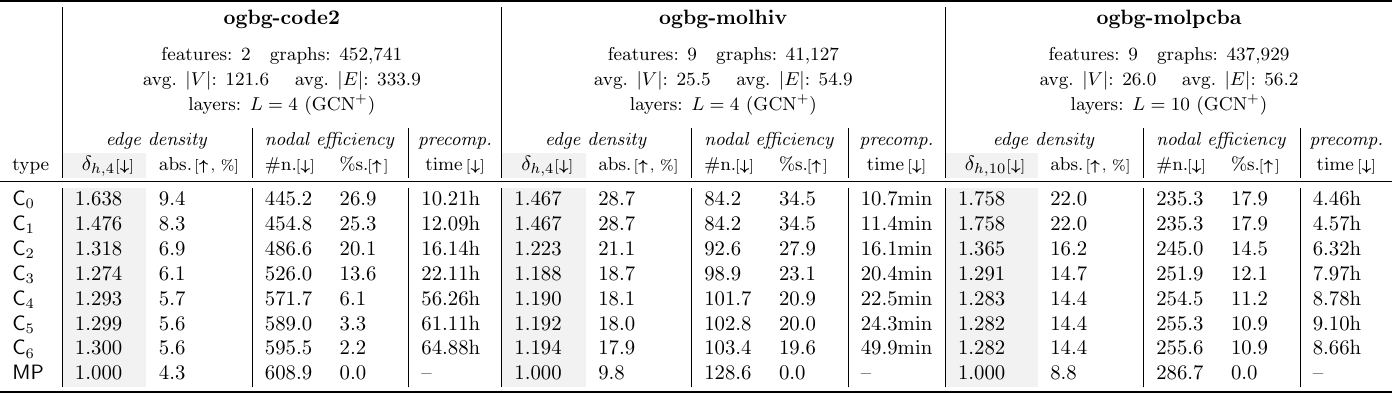}
    \vspace{0.5em}
    \caption{
        \emph{Density and nodal efficiency of computational graphs across large datasets.}
        By $\textsf{C}_h$ is denoted the (caterpillar height) parameter $h\In\Nat$ of graph-level aggregation (ours), while $\textsf{MP}$ denotes the full message-passing GCN$^{+}$  \citep{Luo+2025}.
        The subcolumns correspond to $\delta_{h,L}$ (lower is better) and $\operatorname{relative\_density}$ (higher is better once $\delta_{h,L}$ is bounded) for $L=5$ layers.
        We report absolute density in $\%$.
        Columns ``$\#$n.'' denote the number of \emph{nodes} of the computation graph,
        and columns ``$\%$s.'' percentage of nodes of the computation graph \emph{saved} 
        compared with message-passing~($\textsf{MP}$).
    }
    \label{tab:dataset_statistic_group}
\end{table}

\section{Details of Experiments}\label{appendix:implementation}

\subsection{Additional Conclusions for Paradoxically Reversed Descent}\label{appendix:descent}
We observed a double descent, which we attribute to training oscillation between two regimes: digit-by-digit processing, and a higher-level aggregation, producing the high-variance performance dip.
As we scale our models (cf. \cref{subsec:scaling}, \cref{sec:expressivity}), incidence topology scales analogously from $\tau(\chitriv)$ up to $\tau(\chiid)=\tau(E)$.
Unlike rewiring strategies, e.g., \cite{Top+2022}, changing edges $E'$ to operate in $\tau(E')$, our approach restructures the computational graph into a less expressive one (\cref{sec:expressivity}).
Finally, given the systematically decreasing performance of models $\textsf{C}_{10}$ up to $\textsf{MP}$, 
a bottleneck of \emph{information alignment} of $\tau(E)$ is revealed by the topology $\tau(\chideg)$ that qualitatively differs from, e.g., oversquashing \citep{Alo+2021}.

\newcommand{\Bc}{\mathcal{B}}
\subsection{Incidence Topology}
Although our discussion of topology mainly serves heuristic and interpretative purposes, we provide formal definitions of the relevant notions in \cref{subsec:topology}.

\emph{Topology.} 
Let $V$ be a set and $\tau\subseteq 2^V$ be a family of subsets of $V$. Then $\tau$ is a \emph{topology} on $V$ if the following three conditions hold:
\begin{enumerate}
    \item[(T1)] Set $V \In \tau$, and empty set $\emptyset\In \tau$.
    \item[(T2)] The family $\tau$ is closed on all unions of its sets.
    \item[(T3)] The family $\tau$ is closed on finite intersections of its sets. 
\end{enumerate}
A subfamily $\Bc\subseteq \tau$ is a \emph{subbase} for $\tau$ if $\tau$ is the intersection of all topologies on $V$ containing $\Bc$.
We say that $\Bc$ \emph{generates} $\tau$.

For a graph $G = (V,E)$ on $V$, the neighborhood topology $\tau(E)$ is generated by the following family of sets for every $v\In V$ and $r \In \mathbb{N}$ such that $r> 0$:
\[
B_r(v) = \{ u \mid u \In V,\, (\mI + \mA)^r[u, v] \neq 0 \},
\]
where $\mA$ is the adjacency and $\mI$ is the (self-loop) identity matrix of shape $V\times V$ for $G$.

Let $\chi$ be a $\Sigma$-coloring on $G$ and let $T$ be the maximum length of colored walks. 
Then we denote the following family of sets for every $\va \In \Sigma^{\le T}$
\begin{align}
B(\va, T) \coloneqq \{u \mid u \In V,\, (\mI + \mA)\mW[u, \va] \neq 0\}.\label{eq:geninctop}
\end{align}
where we recall the walk-incidence matrix $\mW$ of shape $V\times \Sigma^*$ defined in \cref{subsec:sequential}.

We call the topology generated by the subbase $\Bc = \{B(\va, T) \mid \va \In \Sigma^{\le T}\}$ the \emph{incidence topology} $\tau(\chi)$ on $V$.
We state the following proposition to highlight the similarity with \cref{obs:specialwalkpartition} and \cref{obs:specialmessagepassing}.
\begin{proposition}
    Let $G =(V,E)$ be a graph without isolated vertices, and $T > 0.$
    Then the topology $\tau(\chitriv)$ is trivial, and the topology $\tau(\chiid) = \tau(E)$.
\end{proposition}
\begin{proof}
    Let us denote the colored walk of a single color of length $t$ by $\vz_t$
    (as in \cref{obs:specialwalkpartition}).
    Since every node $u\In V$ has at least one neighbor, it is adjacent to colored walk $\vz_{2t}$ and incident to colored walk $\vz_{2t+1}$. That is, $(\mA\mW)[u, \vz_{2t}]\neq 0$ and $(\mI\mW)[u, \vz_{2t+1}]\neq 0$ and thus $\tau(\chitriv)$ contains only $\emptyset$, and $V$.
    Therefore, the subbase is of $B(\vz_t, T) = V$, and thus $\tau(\chitriv) = \{\emptyset, V\}$.
    
    For the second part, we recall that $\chiid$ is a $V$-coloring. For $v\In V$ we take any colored walk $\va v$,
    and we get $(\mI + \mA)W[u, \va v] = (\mI + \mA)[u, v]$. Finally, every ball $B_r(v)$ in $\tau(E)$ is (already) generated by the union of some $1$-balls of the form $B_1(v) = B(\va v, T) = \{ u \mid u \In V,\, (\mI + \mA)[u, v]\neq 0\}$.
\end{proof}

\begin{figure}[t]
    \centering
    \includegraphics[width=0.6\textwidth]{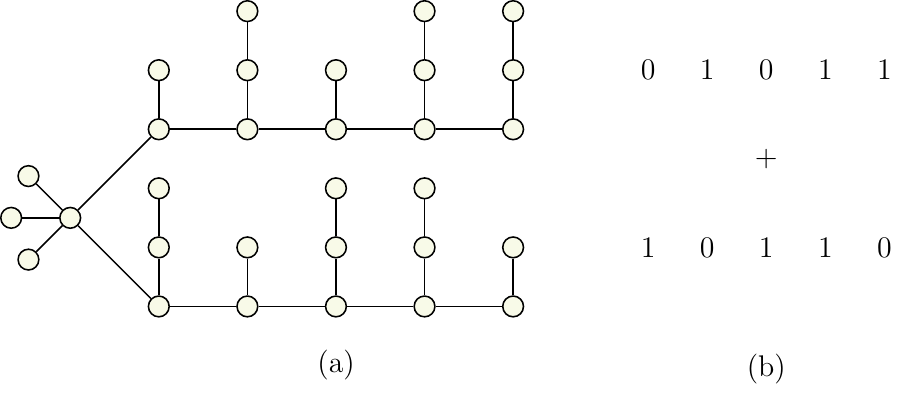}
    \caption{
        \textsc{nStepAddition}: An example of a graph $G$ (a). The graph $G$ encodes two $5$-bit numbers (b),
        namely $11$~(top) and $22$~(bottom).
        If $N=33$ for a classification dataset then $G$ is a positive example, as $11+22=33$.
    }
    \label{fig:experiment_sum}
\end{figure}

\subsection{Generation of \textsc{nStepAddition} Dataset}\label{appendix:synthetic}
We construct the dataset for a binary classification task, where each graph encodes two integers represented in binary.
For an example of a graph $G$ encoding two integers $11$ and $22$ in the $5$-bit binary case, see \autoref{fig:experiment_sum}.
To construct a graph of our dataset of a target sum $N = 2^{B-1}$, 
we randomly generate two integers $x_1$ and $x_2$ in the range $[0, 2^{B}-1]$, such that $x_1 + x_2 = N$.
With $0.5$ probability, we add an offset to one of the numbers in the range of $[0, \frac{2}{3}N]$ to deviate from the target sum $N$ by not more than $66.7\%$.
Accordingly, we assign the class of the graph to be $1$ if $x_1 + x_2 = N$, and $0$ otherwise.
We do not use any node features, so the classification relies solely on graph structure.
Samples: 6,000 graphs with integers generated using $B=15$ bits.
The dataset we provide is balanced with respect to the class labels.

\subsection{Implementation Details}
We use the PyTorch Geometric library \citep{Pas+2019} to implement our models.
As it might be considered common,
we precomputed the normalization of the (sparse) adjacency matrix for the \texttt{GCNConv} layer,
to use it later in the forward pass with the option \texttt{norm=False}, and add the normalized
weights to the message-passing instead.
For the cases where Caterpillar GCN is not identical to message passing,
we added an efficient (sparse) version of the adjacency matrix $C^{\tilde{A}}_{t_\ell}$, 
specialized for every layer $\ell$, in a way similar to the normalization of the adjacency matrix.
We preprocessed datasets in a unified fashion for message-passing and Caterpillar GCN ($\textsf{C}_{0}$ up to $\textsf{C}_{10}$).

\begin{table}[t]
    \centering
    \includegraphics[width=0.7\textwidth]{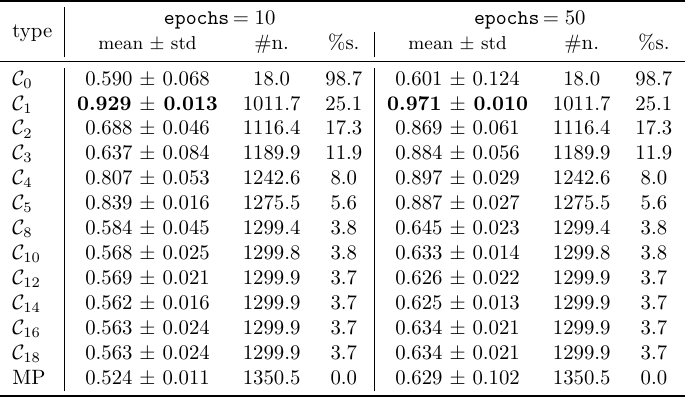}
    \vspace{0.5em}
    \caption{
        Results for the \textsc{nStepAddition} dataset.
        By $\textsf{C}_h$ is denoted the (caterpillar height) parameter $h\In\Nat$ of graph-level aggregation (ours), while MP denotes the full message-passing GCN.
        We report mean validation accuracy with standard deviation across 10 different splits.
        The 18-layer model was trained for 10 and 50 epochs.
        The best results are highlighted in bold.
        The columns ``$\#$n.'' denote the number of \emph{nodes} of the computation graph,
        and columns ``$\%$s.'' the percentage of nodes of the computation graph \emph{saved}
        compared with message-passing (MP).
    }
    \label{tab:expressivity_hurts}
\end{table}

We trained our models on the \textsc{nStepAddition} dataset described above using deeper architectures (18 layers),
with small hidden dimensions (width=8), batch size of 64, and moderate regularization settings (final dropout 0.3, weight decay $10^{-6}$). Training was performed for 10 \autoref{fig:expressivity_hurts} and 50 epochs using 10-fold cross-validation. 
Complete results are reported in \autoref{tab:expressivity_hurts}.
We deliberately adopted a minimal configuration in order to cleanly isolate the topological effect of computational graph scaling behind our approach.
We remark that the definition of incidence topology in general supports richer benchmarks.
The ongoing work may extend these experiments to include multiple occurrences (multiple graph branches to sum over multiple numbers) and vertex labels (for larger number systems than binary), e.g. in line with the experimental contributions of~\cite{Alo+2021}.

    \begin{figure}[h]
    \centering
    \includegraphics[width=0.9\linewidth]{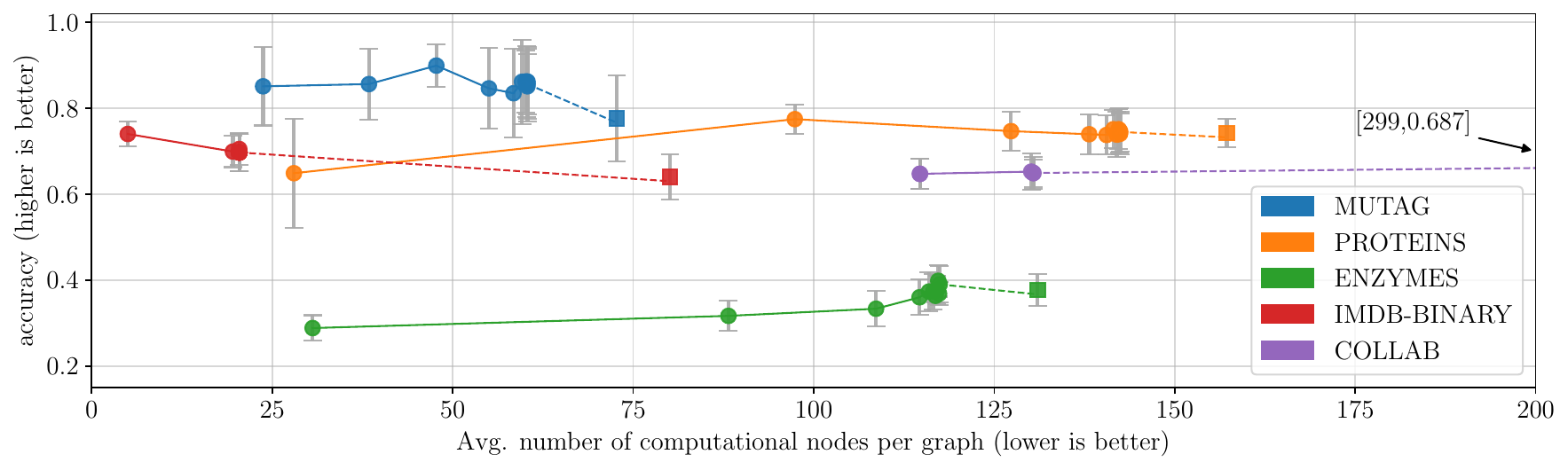}
    \includegraphics[width=0.9\linewidth]{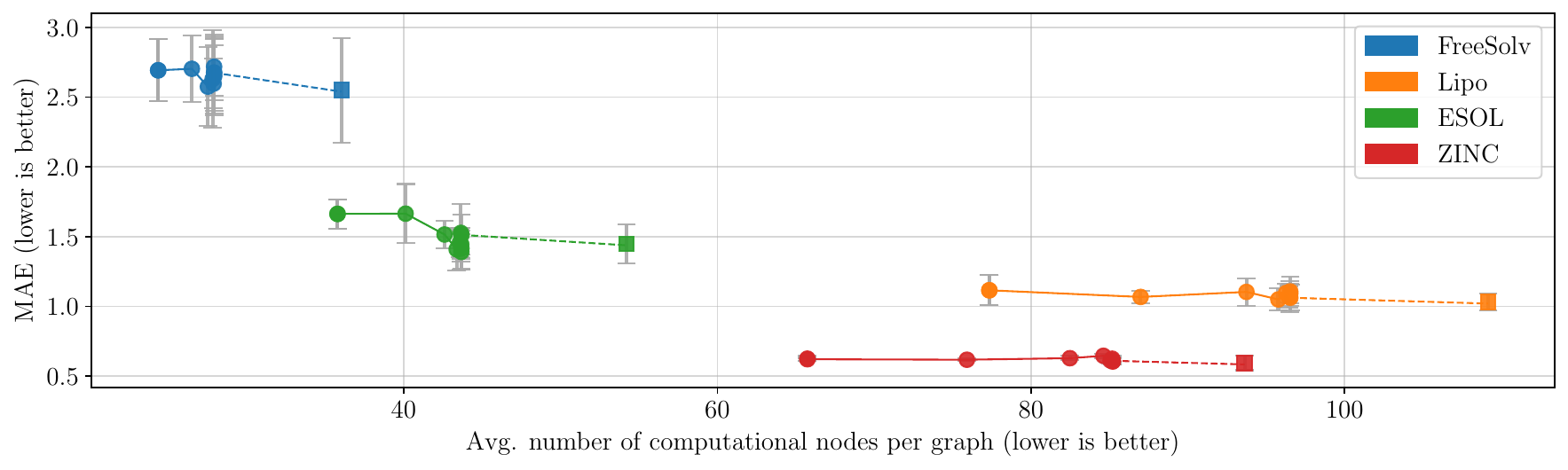}
    \caption{Computational nodes vs. accuracy.
            Solid segments connect models of Caterpillar GNN, with height $h = 0$ ($\textsf{C}_0$, circle) up to $h=10$ ($\textsf{C}_{10}$, circle), the final dashed segment connects to MPGNN (square).
        }
    \label{fig:height}
\end{figure}

\subsection{Experimental Setting on Real-World Datasets}
We performed empirical experiments across multiple standard graph datasets, categorized by their evaluation metrics. 
Accuracy-based evaluation was used for bioinformatics and social network datasets, including the TUDataset~\citep{Mor+2020} benchmarks such as ENZYMES, MUTAG, PROTEINS, COLLAB, and IMDB-BINARY (see \autoref{tab:dataset_statistics_acc}).
\add{We report running times for both preprocessing and training (see \autoref{tab:running_times}).}
Training code, along with exact hyperparameter configurations for reproducibility, is available in the supplementary material.

For each dataset, we trained graph convolutional networks (GCNs) with variants of our proposed model Caterpillar GCN.
To compare the number of saved nodes consistently, for all experiments, we \emph{fixed the number of layers to five} ($L=5$),
gradient clipping to $1.0$, and based on the validation set performance, we employed early stopping with a fixed \texttt{patience} parameter of 20 to prevent overfitting.
We conducted extensive k-fold cross-validation (typically 10 folds), ensuring robust performance evaluation. For the ZINC dataset, we always used 5 random seed initializations for the public splits.
Models were trained with standard settings, using the Adam optimizer \citep{Kin+2014}, a moderate learning rate, and weight decay to balance training stability and convergence speed.
Complete hyperparameter details are provided in the supplementary material.
Experiments were repeated using fixed random seeds to ensure reproducibility. All models were trained using an Intel Xeon E5-2690 @ 2.90 GHz processor (16 cores, 32 threads, 20 MB L3 cache, 64 B cache line) equipped with 64 GB of RAM.

\begin{table}[p]
    \centering
    \includegraphics[width=0.99\textwidth]{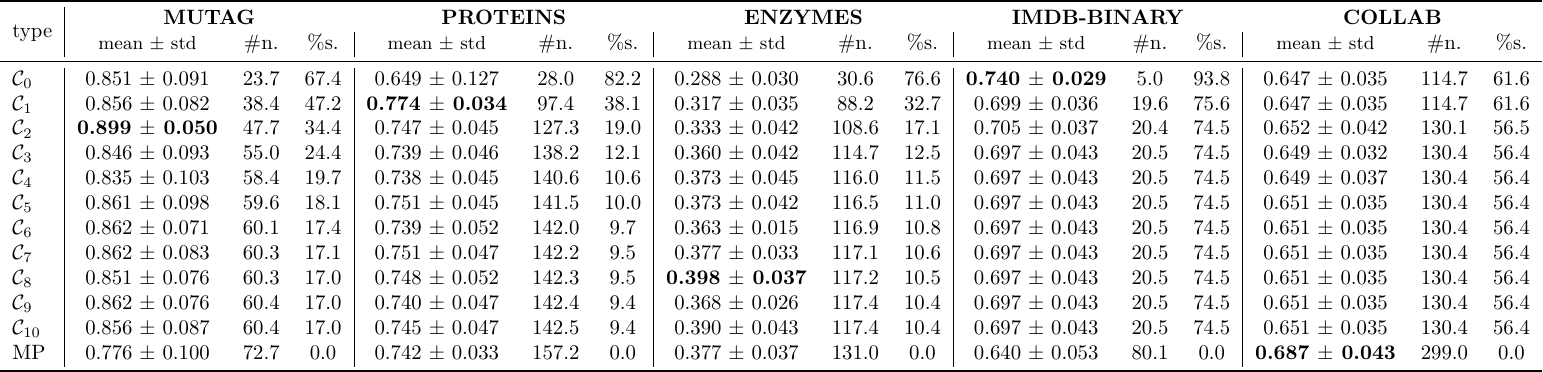}
    \vspace{0.5em}
    \caption{
        Results for \emph{graph-level classification} datasets.
        By $\textsf{C}_h$ is denoted the (caterpillar height) parameter $h\In\Nat$ of graph-level aggregation (ours), while $\textsf{MP}$ denotes the full message-passing GCN.
        We report mean validation accuracy with standard deviation across 10 different splits.
        Columns ``$\#$n.'' denote the number of \emph{nodes} of the computation graph,
        and columns ``$\%$s.'' the percentage of nodes of the computation graph \emph{saved} 
        compared with message-passing ($\textsf{MP}$).
    }
    \label{tab:dataset_statistics_acc}
\end{table}

\begin{table}[p]
    \centering
    \includegraphics[width=0.99\textwidth]{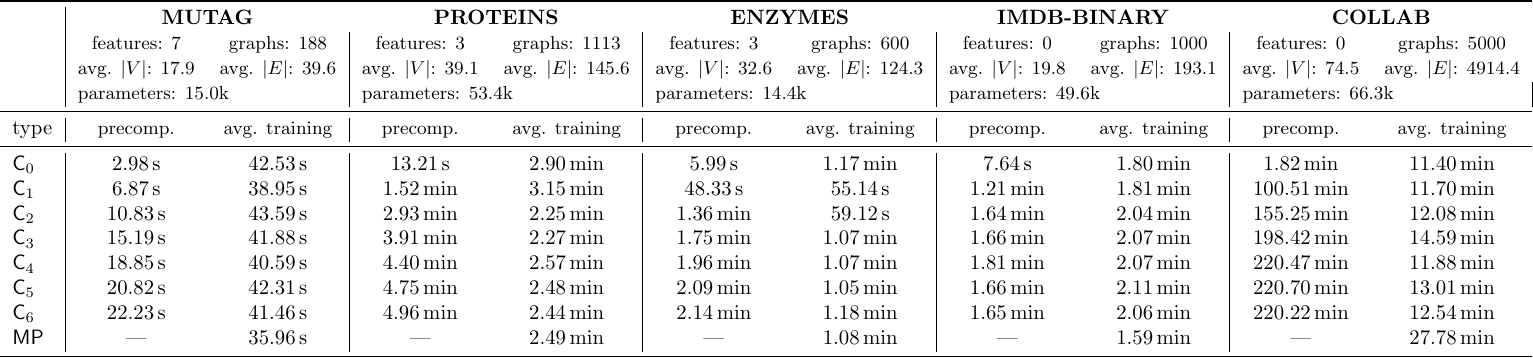}
    \vspace{0.5em}
    \caption{
        \emph{Running times across datasets.}
        By $\textsf{C}_h$ is denoted the (caterpillar height) parameter $h\In\Nat$ of graph-level aggregation (ours), while $\textsf{MP}$ denotes the full message-passing GCN.
        Columns ``precomp" report the preprocessing time for the computation of graph-level aggregation matrices,
        and columns ``avg. training"  report the average training time across multiple runs.
    }
    \label{tab:running_times}
\end{table}

\begin{table}[p]
    \centering
    \includegraphics[width=0.99\textwidth]{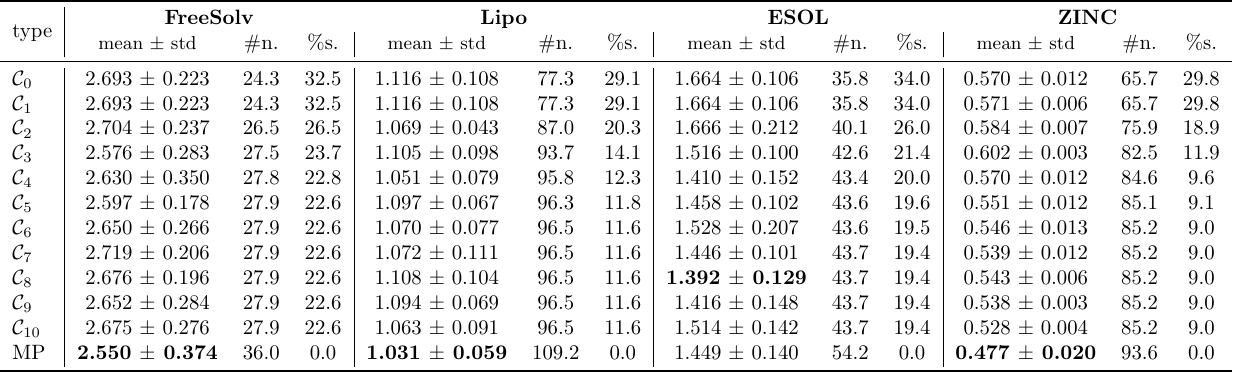}
    \vspace{0.5em}
    \caption{
        Results for \emph{graph-level regression} datasets.
        By $\textsf{C}_h$ is denoted the (caterpillar height) parameter $h\In\Nat$ of graph-level aggregation (ours), while $\textsf{MP}$ denotes the full message-passing GCN.
        We report validation mean absolute error (MAE) with standard deviation across 10 different splits and random seeds, with five repetitions using different random seeds for ZINC.
        Columns ``$\#$n.'' denote the number of \emph{nodes} of the computation graph,
        and columns ``$\%$s.'' the percentage of nodes of the computation graph \emph{saved} 
        compared with message-passing ($\textsf{MP}$).
    }    
    \label{tab:dataset_statistics_mae}
\end{table}


\newpage

\end{document}